\documentclass[twoside]{article}

\usepackage[accepted]{aistats2025}
\usepackage[round]{natbib}

\usepackage[utf8]{inputenc} 
\usepackage[T1]{fontenc}    

\usepackage{hyperref}       
\usepackage[table]{xcolor}         
\definecolor{ballblue}{rgb}{0.13, 0.67, 0.8}
\definecolor{babypink}{rgb}{0.96, 0.76, 0.76}
\definecolor{antiquefuchsia}{rgb}{0.57, 0.36, 0.51}
\definecolor{blue(pigment)}{rgb}{0.2, 0.2, 0.6}
\definecolor{blush}{rgb}{0.87, 0.36, 0.51}
\hypersetup{
    colorlinks = true,
    citecolor = {blue},
    linkcolor = {blue(pigment)}
}
\usepackage{url}  

\usepackage{ISG_style}

\usepackage{url}            
\usepackage{booktabs}       
\usepackage{amsfonts}       
\usepackage{nicefrac}       
\usepackage{microtype}      
\usepackage{tabularx}
  
\usepackage{amssymb}
\usepackage{amsfonts}
\usepackage{mathrsfs}
\newcommand{\holder}{H\"{o}lder }

\usepackage{dsfont}
\usepackage{graphicx}
\usepackage{upgreek}
\usepackage{algorithm,algorithmic}
\usepackage{subcaption}
\usepackage{comment}
\usepackage{bigints}
\usepackage{float}

\usepackage{hyperref}
\usepackage{enumitem}
\usepackage[compact]{titlesec} 

\setlist{topsep=-.05in, itemsep=-.05in}

\DeclareMathOperator*{\argmax}{arg\,max}
\DeclareMathOperator*{\argmin}{arg\,min}
\newcommand*\diff{\mathop{}\!\mathrm{d}}

\newcommand{\norm}[1]{\left\lVert#1\right\rVert}

\newcommand{\der}{{\rm d}}
\usepackage{bbm} 
\usepackage{diagbox}

\usepackage[toc,page,header]{appendix}
\usepackage{minitoc}

\renewcommand{\F}{\mathbb{F}}
\newcommand{\G}{\mathbb{G}}
\renewcommand{\Q}{\mathbb{Q}}

\makeatletter
\expandafter\def\expandafter\normalsize\expandafter{%
    \normalsize%
    \setlength\abovedisplayskip{5pt}%
    \setlength\belowdisplayskip{3pt}%
    \setlength\abovedisplayshortskip{-8pt}%
    \setlength\belowdisplayshortskip{2pt}%
}
\makeatletter
\renewcommand{\paragraph}{%
  \@startsection{paragraph}{4}%
  {\z@}{-.0ex \@plus 0ex \@minus .2ex}{-1em}%
  {\normalfont\normalsize\bfseries}%
}
\makeatother

\begin{document}
\allowdisplaybreaks 
%

%





\twocolumn[

\runningtitle{Risk-sensitive Bandits: {Arm Mixture Optimality and Regret-efficient Algorithms}}
\aistatstitle{Risk-sensitive Bandits: \\ {Arm Mixture Optimality and Regret-efficient Algorithms}}

\aistatsauthor{Meltem Tatl{\i}$^*$ \And Arpan Mukherjee$^*$ \And Prashanth L.A. \And Karthikeyan Shanmugam \And Ali Tajer}

\aistatsaddress{ RPI \And  RPI \And IIT Madras \And Google Deepmind India \And RPI }

]

\begin{abstract}
This paper introduces a general framework for risk-sensitive bandits that integrates the notions of risk-sensitive objectives by adopting a rich class of {\em distortion riskmetrics}. The introduced framework subsumes the various existing risk-sensitive models. An important and hitherto unknown observation is that for a wide range of riskmetrics, the optimal bandit policy involves selecting a \emph{mixture} of arms. This is in sharp contrast to the convention in the multi-arm bandit algorithms that there is generally a \emph{solitary} arm that maximizes the utility, whether purely reward-centric or risk-sensitive. This creates a major departure from the principles for designing bandit algorithms since there are uncountable mixture possibilities. The contributions of the paper are as follows: (i) it formalizes a general framework for risk-sensitive bandits, (ii) identifies standard risk-sensitive bandit models for which solitary arm selections is not optimal, (iii) and designs regret-efficient algorithms whose sampling strategies can accurately track optimal arm mixtures (when mixture is optimal) or the solitary arms (when solitary is optimal). The algorithms are shown to achieve a regret that scales according to $O((\log T/T )^{\nu})$, where $T$ is the horizon, and $\nu>0$ is a riskmetric-specific constant.

\end{abstract}

\section{MOTIVATION \& OVERVIEW}

\label{sec:introduction}
The canonical objective of stochastic multi-armed bandits is designing a sequence of experiments to identify the arm with the largest expected reward. This is a \emph{utilitarian} view that lacks a mechanism for risk management. For instance, designing a financial portfolio based on maximizing the expected financial return is susceptible to incurring heavy financial losses due to tail events. Recovering from such losses translates into increased utility regret due to the losses and the upcoming missed opportunities.  

Leveraging the recent advances in risk analysis, we provide a versatile risk-sensitive bandit framework that replaces the utilitarian reward functions with proper risk-sensitive counterparts. Specifically, we adopt the notion of \emph{distortion riskmetrics} (DRs)~\citep{Wang2020}, which encompass a broad range of widely-used \emph{risk} and \emph{deviation} measures. For a given function $h:[0,1]\to [0,1]$ that satisfies $h(0)=0$, and a given probability measure with the cumulative distribution function (CDF) $\Q$, the distortion riskmetric $U_h$ is defined as the \emph{signed} Choquet integral 
\begin{align}
\small
\label{eq:wang_def}
\nonumber
    U_h(\Q) = \int_{-\infty}^{0} & \Big(h\big(1- \Q(x)\big)  - h(1)\Big) \, \der x \\ & \qquad \qquad + \int_{0}^{\infty} h\big(1-\Q(x)\big) \, \der x .
\end{align}
The function $h$, referred to as the {\em distortion function}, enables high versatility in specifying various risk measures and unifies various notions of risk, variability, and preference. Most such notions can be classified based on the monotonicity of the distortion function $h$ as follows.
\begin{itemize}
    \item {\bf Monotone DRs:} There exists extensive literature on a subclass of DRs in which  $h$ is monotonically increasing and $h(1)=1$. These include $L$-functionals in statistics \citep{huber2009robust}, Yaari’s dual utilities in decision theory \citep{yaari1987dual}, distorted premium principles in insurance \citep{denneberg1994nonadditive}, and distortion risk measures in finance \citep{kusuoka2001law}. Some widely used examples of monotone distortion riskmetrics are value-at-risk (VaR), expected shortfall (ES) \citep{artzner99}, and conditional VaR (CVaR) \citep{rockafellar2000optimization}.
\newpage    \item {\bf Non-monotone DRs:} Unlike their monotone counterparts and despite their significance, the non-monotone DRs are far less investigated. These include different measures of variability, such as deviation measures \citep{rockafellar2006generalized},  mean-median deviation, inter-quantile range, Wang's right-tail deviation, inter-expected shortfall, Gini deviation \citep{Wang2020}, cumulative Tsallis past entropy \citep{zuo2024worstCases};
    preference measures such as 
 Gini shortfall \citep{furman2017gini}; and rank-based decision-making in decision theory \citep{quiggin1982theory}.
\end{itemize}

\paragraph{Risk-sensitive bandits.} Canonical bandit algorithms aim to maximize the expected cumulative reward. 
Risk-sensitive bandit algorithms, in contrast, seek to strike a balance between risk and reward, often formalized through proper risk measures. Notable recent advances on risk-sensitive bandits include adopting VaR and CVaR as risk measures~\citep{
aditya2016weighted, prashanth2022wasserstein, chang2022, cassel2018general,baudry2021optimal, tamkin2019cvar,tan2022cvar,liang2023distribution}. 

Besides balancing the risk-reward dichotomy, there also exist other studies on risk-sensitive best arm identification~\citep{kagrecha2019distribution, prashanth2020concentration} and contextual bandits~\citep{Huang2021}. 

The majority of the existing studies on risk-sensitive bandits, despite their distinct emphases and approaches, fall in the category of \emph{monotone} DRs, with the only exception of mean-variance risk studied in \citep{chang2022, cassel2018general, sani2013risk, vakili2016risk}. The non-monotone DRs, while actively studied in the adjacent fields, e.g., reinforcement learning, remain unexplored for bandits. Some recent studies of non-monotone DRs in reinforcement learning include Gini deviation~\citep{Luo2023}, and inter-ES~\citep{Han2022}.

\paragraph{Contributions.} The contributions are three-fold. We (i)~design a risk-sensitive bandit framework that unifies all the monotone and non-monotone DRs; (ii)~introduce novel \emph{fixed-horizon} risk-sensitive algorithms; and (iii)~establish that our algorithms are regret-efficient. Due to the high versatility of DRs in accommodating a significant range of risk measures, our framework subsumes most existing studies in risk-sensitive bandits through proper choices of the distortion function $h$. Some special cases are presented in Table~\ref{table:table_regrets}. Our framework provides regret analysis for non-monotone DRs, which the existing literature cannot address, and recovers the regret guarantees for some of the existing monotone DRs. Finally, we note that by setting $h(u)=u$, the DR $U_h(\Q)$  simplifies to the expected value of $\Q$, representing the standard risk-neutral objective. 

\newpage
\paragraph{Key observations.} A hitherto unknown observation is that for certain DRs, the optimal policy is a \emph{mixture} policy, i.e., there is no single optimal arm and the optimal strategy should be mixing arm selections according to carefully designed mixing coefficients. This contrasts the existing utility-centric or risk-sensitive bandits in which the optimal policy involves selecting a \emph{solitary} arm \citep{cassel2018general}. 

More specifically, we show that contrary to the convention of considering the optimal solution involving a \emph{solitary} arm, the optimal solution of a non-monotone DR might involve a \emph{mixture} of arms.
In the latter case, we show that the risk of choosing any single arm is uniformly dominated by that of sampling based on a proper mixture of the arms. This guides a significantly different way of designing bandit algorithms, an integral part of which will be estimating mixing coefficients and designing arm selection strategies that accurately track the optimal mixture over time.

\paragraph{Technical challenge of learning optimal mixtures.} Learning the optimal mixtures is a problem fundamentally distinct from identifying a solitary optimal arm. Searching for a solitary arm is guided by sequentially identifying and selecting an arm that optimizes a desired metric (e.g., upper confidence bound). 
In contrast, a mixture policy introduces two new challenges to the arm-selection policy. The one pertains to estimating $K$ continuous-valued mixing coefficients and determining the optimal mixture from these estimates.
The second dimension is the need for an efficient algorithm to track the optimal mixture. This necessitates balancing arm selections over time so that their mixture, in aggregate, conforms to the optimal one.

Addressing the first challenge (estimating mixture coefficients) requires estimating arms' CDFs over time, since the mixture coefficients depend on knowing the DRs associated with different arms, which are functions of the arms' CDFs. For the second challenge (tracking the optimal mixture), we design an arm-selection policy using the mixing estimates that deviates from the counterpart strategies with a solitary optimal arm. 

\paragraph{Organization.} The DR-centric risk-sensitive bandit framework is presented in Section~\ref{sec:setting}. We present two algorithms in Section~\ref{sec:algorithm} based on the explore-then-commit (ETC) and upper-confidence bound (UCB) principles. The key processes in these algorithms are routines for learning the optimal mixtures. This new addition renders the designs and, especially, analyses of these algorithms significantly distinct from the standard ETC- and UCB-type analyses. 
The regret analyses are presented in Section~\ref{sec:analysis} and the empirical evaluations are discussed in Section~\ref{sec: experiments}. The proofs are relegated to the appendices.

\setlength{\textfloatsep}{5pt}
\begin{table*}[t!]
\small
\caption{Regret bounds of ETC-type ($\mathfrak{R}_{\bnu}^{\rm E}(T)$) and UCB-type ($\mathfrak{R}^{\rm U}_{\bnu}(T)$) algorithms, where $\varpi(T)\triangleq  \sqrt{\frac{\log T}{T}}$.}
\label{table:table_regrets}
\begin{minipage}[t]{\textwidth}
\centering
{
\begin{tabular}{|l|l|l|l|l|}
\hline 
\small
\textbf{Distortion Riskmetrics 
\footnote{Summary of results for the $K$-arm Bernoulli bandit model. The more general results are presented in Section~\ref{sec:analysis}.}\footnote{In non-shaded rows, solitary arms are optimal. In the shaded rows, mixtures of arms are optimal.}\footnote{$\kappa \in(0,\frac{1}{2})$ and $\gamma\in(0,1)$ are constants associated with the distortion function $h$ and are specified in Section~\ref{sec:analysis}.}
}&  \(h(u)\)& parameter & \textbf{$\mathfrak{R}_{\bnu}^{\rm U}(T)$} & \textbf{$\mathfrak{R}_{\bnu}^{\rm E}(T)$} \\
\hline\hline 
Risk-neutral Mean Value     &  \(u \)& &  {$O(\varpi^{1/2}(T))$} &  $O(\varpi^2(T))$ \\ 
Dual Power                 &  \(1-(1-u)^s \)& \(s \in [2,+\infty)\)  & {$O(\varpi^{1/2}(T))$} & $O(\varpi^2(T))$ \\
Quadratic                   &   \((1+s)u-su^2\)  & \(s\in[0,1]\) &{$O(\varpi^{1/2}(T))$}& $O(\varpi^2(T))$ \\
CVaR$_\alpha$ \footnote{Arm means $<1-\alpha$ for $\mathfrak{R}_{\bnu}^{\rm U}(T)$.}
&  \(\min\left\{\frac{u}{1-\alpha}, 1\right\} \)& &   {$O(\varpi^{1/2}(T))$} & $O(\varpi^2(T))$ \\ 
PHT               &  \(u^s\)  & \(s=1/2 \)& {$O(\varpi^{1/5}(T))$}     & $O(\varpi(T))$ \\
\rowcolor{gray!30}
Mean-Median Deviation             & \(\min\left\{u, 1-u\right\} \)  & & {$O(\varpi^{2 \kappa}(T))$}  & {$O(\varpi^{1/\beta}(T^{\gamma}))$} \\
\rowcolor{gray!30}
Inter-ES Range 
& \(\min\left\{\frac{u}{1-\alpha}, 1\right\} \) + \(\min\left\{\frac{\alpha-u}{1-\alpha}, 0\right\} \) & \(\alpha = 1/2\) & {$O(\varpi^{ 2 \kappa}(T))$}  & {$O(\varpi^{1/\beta}(T^{\gamma}))$} \\
\rowcolor{gray!30}
Wang's Right-Tail Deviation             &  \(\sqrt{u}-u \) & & {$O(\varpi^{2  \kappa}(T))$}   & {$O(\varpi^{1/2\beta}(T^{\gamma}))$} \\ 
\rowcolor{gray!30}
Gini Deviation        & \(u(1-u)\) & & {$O(\varpi^{4 \kappa}(T))$}  & {$O(\varpi^{2/\beta}(T^{\gamma}))$} \\
\hline 
\end{tabular}
}
\end{minipage}
\vspace{-.2 in}
\end{table*}

\section{DR-CENTRIC FRAMEWORK}
\label{sec:setting}
\vspace{-.05 in}
\paragraph{Bandit model.} Consider a $K$-armed unstructured stochastic bandit.  Each arm $i\in [K]\triangleq\{1,\cdots,K\}$ is endowed with a probability space $(\Omega,\mcF,\F_i)$, where $\mcF$ is the $\sigma$-algebra on $\Omega\subseteq\R_+$ and $\F_i$ is an {\em unknown} probability measure\footnote{We focus on positive-valued random variables. Extension to include negative values
is straightforward.}. Accordingly, define 
$\F \triangleq  \{\F_i:i\in[K]\}$.
At time $t\in\N$, a policy $\pi$ selects an arm $A_t\in[K]$ and the arm generates a stochastic sample $X_t$ distributed according to $\F_{A_t}$. Denote the sequence of actions, observations, and the $\sigma$-algebra that policy $\pi$ generates up to time $t\in\N$ by $\mcX_t\triangleq\left ( X_1,\cdots, X_t\right )$, $ \mcA^\pi_t \triangleq\left( A_1,\cdots,A_t\right )$, and $\mcH^\pi_t\triangleq\sigma\left (A_1,X_1,\cdots,A_{t-1},X_{t-1}\right )$.

Corresponding to any bandit instance $\bnu$ and policy $\pi$, $\P_{\bnu}^{\pi}$ denotes the push-forward measure on $\mcH^\pi_t$, and $\E_{\bnu}^{\pi}$ denotes the associated expectation. The sequence of independent and identically distributed (i.i.d.) rewards generated by arm $i\in[K]$ up to time $t\in\N$ is denoted by  $\mcX_t(i)\triangleq \{X_t : A_t= i\}$. We define $\tau_t^\pi(i)\triangleq |\mcX_t(i)|$ as the number of times that policy $\pi$ selects arm $i\in[K]$ up to time $t$.

\textbf{Distortion riskmetric.} We have a significant departure 
from the objective of evaluating the cumulative reward collected over time, and instead focus on evaluating the risk associated with the decisions made over time. We consider a generic distortion function $h:[0,1]\to [0,1]$, which is not necessarily monotonic. The distortion riskmetric associated with $h$ is denoted by $U_h$ and is defined in~\eqref{eq:wang_def}. Corresponding to any given policy $\pi$ at time $t$, the overall risk associated with the sequence of arm selections $(A_1,\dots, A_t)$ is given by
\begin{align}    
\label{eq:U}
\textstyle U_h\left(\sum_{s=1}^t\sum_{i\in[K]}\frac{\mathds{1}\{A_s=i\}}{t}\; \F_i\right) = U_h\left(\sum_{i\in[K]}\frac{\tau_t^\pi(i)}{t} \; \F_i\right).
\end{align}
We highlight that the DR $U_h$ depends on the full descriptions of the arms' statistical models (CDFs). 

\vspace{-.05 in}

\paragraph{Oracle policy.} An oracle policy can accurately identify the optimal sequence of arm selections $\{A_t:t\in\N\}$. Given the structure in~\eqref{eq:U}, designing an optimal oracle policy is equivalent to determining the optimal mixing of the CDFs. For a bandit instance $\bnu\triangleq(\F_1,\cdots,\F_K)$, the vector of optimal mixture coefficients is denoted by
\begin{align}
\label{eq:optimal alpha}
\balpha^\star_{\bnu}\;\in\;\argmax\limits_{\balpha\in\Delta^{K-1}}\;U_h\Bigg ( \sum\limits_{i\in[K]}\alpha(i) \; \F_i\Bigg )\ ,
\end{align}
where  $\Delta^{K-1}$ is a $K$-dimensional simplex. When clear from the context, we use the shorthand  $\balpha^\star$ for $\balpha^\star_{\bnu}$.
We note that the oracle considered in~\eqref{eq:optimal alpha} is omniscient, i.e., it is aware of all the distributions $\F =  \{\F_i:i\in[K]\}$ and the mixing policy $\balpha^\star_{\bnu}$ that optimizes the DR. 
Throughout the rest of the paper, for any given vector $\balpha$ and set of CDFs $\F =  \{\F_i:i\in[K]\}$, we define
\begin{align}
    \label{eq:V}
    V(\balpha, \F) \triangleq  U_h\Bigg ( \sum\limits_{i\in[K]}\alpha(i) \; \F_i\Bigg )\ .
\end{align}
 Next, we introduce a motivating example for designing mixture policies. Consider the widely-used Gini deviation. We show that the optimal policy is a mixture. The distortion function for Gini deviation is $h(u)=u(1-u)$ for $u\in[0,1]$ \citep{Wang2020}.
\vspace{-.05 in}
\begin{lemma}[Gini Deviation]
\label{example utility}
Consider a two-arm Bernoulli bandit model. For a given $p\in[0,1]$, the arms' distributions are ${\rm Bern}(p)$ and ${\rm Bern}(1-p)$. For distortion function $h(u)=u(1-u)$, we have  
\begin{align*}
    \sup\limits_{\balpha\in\Delta^{K-1}} V(\balpha, \F)\ = U_h (0.5(\F_1 + \F_2)) > \max_{i\in\{1,2\}} U_h(F_i)\ .
\end{align*} \vspace{-.05 in}
\end{lemma}
\vspace{-.1 in}
Hence, {\bf (i)~the optimal value of this DR is achieved when samples come from the mixture distribution $\frac{1}{2}(\F_1 + \F_2)$, and (ii) the DR associated with this mixture is strictly larger than those associated with the individual arms.}

\paragraph{Mixture-centric Objective.} Motivated by the observation that, generally, for a given distortion function $h$, the optimal DR might be achieved by a \emph{mixture} of arm distributions, we provide a general mixture-centric framework in which bandit policies can take mixture forms. This also subsumes solitary policies as a special case when the mixing mass is placed on one arm.  

By the oracle policy's definition $\balpha^\star_{\bnu}$ in~\eqref{eq:optimal alpha}, for a bandit instance $\bnu\triangleq(\F_1,\cdots,\F_K)$, 
our objective is to design a bandit algorithm with minimal regret with respect to the DR achieved by the oracle policy. Hence, for a given policy $\pi$ and horizon $T$, we define the  regret as the gap between the DR achieved by the oracle policy and the \emph{average} DR achieved by $\pi$, i.e., 
\begin{align}
\label{eq:regret}
\mathfrak{R}_{\bnu}^\pi(T)\;\triangleq\;  V(\balpha^\star, \F) 
    -      \E_{\bnu}^\pi\Big[V\Big(\frac{1}{T}\btau_T^\pi, \F \Big)\Big]\ .
\end{align}

\paragraph{Assumptions.} 
For each arm $i\in[K]$, let $\mcP^1(\Omega)$ denote the set of all probability measures with finite first moment on $\Omega$. 
We assume the maximum value the DRs can take is bounded, i.e., \(U_h \leq B\), where \(B\) is a positive constant. Additionally, assume that the distributions $\F_i$ are $1-$sub-Gaussian and belong to the metric space $(\mcP^{1}(\Omega),\norm{\cdot}_{\rm W})$, where $\norm{\G-\mathbb{S}}_{\rm W}$ denotes the $1-$Wasserstein distance between distributions $\G$ and $\mathbb{S}$. We define $W$ as the maximum ratio between Wasserstein and total variation distances between any two mixture of arms, i.e.,
\begin{align*}
    W \triangleq \max \limits_{\balpha \neq \bbeta \in \Delta^K} \frac{1}{\lVert \balpha - \bbeta \rVert_1}\Big\|\sum_i \alpha_i \F_i-\sum_j \beta_j \F_j\Big\|_{\rm W}\ .
\end{align*}  In Theorem~\ref{theorem:W} in Appendix~\ref{Appendix:W_finitess}, we show that $W$ is bounded.
We denote the convex hull of the set of distributions $\{\F_i: i\in[K] \}$ by $\Xi$, and assume $U_h$ satisfies the following notions of continuity. 
\begin{definition}[H\"{o}lder continuity]
\label{assumption:Holder} The DR $U_h$ is H\"{o}lder continuous with exponent $q\in(0,1]$, if for all distributions $\G_1,\G_2\in(\Xi,W_1)$, there exists a finite $\mcL_{\rm H}\in\R_+$ such that
    \begin{align}
    \label{eq:Holder}
        U_h(\G_1) - U_h(\G_2) \; \leq \; \mcL_{\rm H}\norm{\G_1-\G_2}_{\rm W}^q.
    \end{align}
\end{definition}
\begin{definition} [Mixture H\"{o}lder continuity]
The DR $U_h$ is H\"{o}lder continuous with exponent $r\in\R_+$ if for a set of distributions $\{\F_i:i\in[K]\}$ with the convex hull $\Xi$ and a set of mixing coefficients $\{\alpha^\star(i):i\in[K]\}$ for $\F^\star\triangleq \sum_{i\in[K]}\alpha^\star(i)\F_i$ and
any distribution $\G\in(\Xi,W_1)$, there exists a finite $\mcL_{\rm MH}\in\R_+$ such that
\label{assumption:Holder_opt}
    \begin{align}
    \label{eq:Holder_opt}
        U_h(\F^\star) - U_h(\G) \; \leq \; \mcL_{\rm MH}\norm{\F^\star-\G}_{\rm W}^r.
    \end{align}
\end{definition}
\vspace{-0.1 in}
Finally, we define \( \mcL\triangleq  \max\{\mcL_{\rm H}, \mcL_{\rm MH}\}\) as a unified H\"{o}lder constant.
Many of the widely used DRs are H\"{o}lder continuous. 
In Table~\ref{table:table_risks} in Appendix~\ref{Appendix:Additinal Tables}, we present a list of riskmetrics, and specify their associated H\"{o}lder exponents $q$ and $r$.

\section{DR-CENTRIC ALGORITHMS}
\label{sec:algorithm}

\paragraph{Overview.} We present a mixture-centric algorithm designed to select and sample the arms based on a mixture policy. A mixture policy consists of the following two sub-routines that will be integral parts of its algorithm design.

\begin{itemize}
    \item {\bf Estimate mixing coefficients:} The first routine forms estimates of the mixing coefficients. These estimates are updated over time based on the collected empirical DR $U_h$ and are expected to get refined and eventually track the optimal mixture $\balpha^\star$ over time. 
    \item {\bf Track mixtures:} The second routine translates the mixing coefficient estimates into arm selection decisions. These decisions have to balance two concurrent objectives. First, there is a need to improve the estimates of the mixing coefficients. Second, the arm selection rules need to ensure that, in aggregate, their selections track the mixing estimates. This necessitates an explicit mixture tracking rule that is regret-efficient. 
    \end{itemize}
We provide two ETC- and UCB-type algorithms, each with its performance or viability (model information) advantages. Furthermore, we provide a variation of the UCB-type algorithm to address some computational challenges of the base design. To this end, we provide a few definitions and notations that will be used to describe and analyze the algorithms. We use the shorthand $\pi\in\{{\rm E,U,C}\}$ to refer to the ETC-type, UCB-type, and the computationally efficient UCB-type policies, respectively. 

An important routine in these algorithms is estimating the arms' CDFs accurately.
All three algorithms form estimates of arms' CDFs. We denote the empirical estimate of $\F_i$ at time $t$ using policy $\pi$ by
\begin{align}
\label{eq:empirical_CDF}
    \F_{i,t}^\pi(x)\triangleq \frac{1}{\tau_t^\pi(i)}\sum\limits_{s\in[t] : A_s = i} \mathds{1}\left \{X_s \leq x \right\}\ . 
\end{align}
 The following lemma provides a concentration bound on the empirical CDFs in the $1-$Wasserstein metric. 
\begin{lemma}
\label{lemma: concentration}
    For any policy $\pi$ and arm $y\in\R_+$ we have
    \begin{align}
\label{eq:meta_concentration}
\nonumber 
    \P_{\bnu}^\pi & \Big( \norm{\F_{i,t}^\pi  -\F_i}_{\rm W}  >   y \Big) \\ & \leq 
    2\exp\Bigg( -\; \frac{\tau^{\pi}_t(i)}{256 {\rm e}} \Big(y- \frac{512}{\sqrt{\tau^{\pi}_t(i)}}\Big)^2\Bigg)\ .
\end{align}
\end{lemma}
\begin{proof}
    See~\citep[Lemma 8]{prashanth2022wasserstein}.
\end{proof}

\subsection{Discrete Mixture Coefficients}
The optimal mixing coefficients $\balpha^\star$ can take arbitrary irrational values. Sampling frequencies, on the other hand, will always be rational values. Hence, estimating $\balpha^\star$ beyond a certain accuracy level is not beneficial as it cannot be implemented. We define $\varepsilon$ to specify the desired level of fidelity for each mixing coefficient. Accordingly, we specify $\Delta^{K-1}_{\varepsilon}$ by uniformly discretizing each dimension of $\Delta^{K-1}$ with the $\varepsilon$ discrete level. The discretization schemes are slightly different for the ETC- and UCB-type algorithms and the specifics are provided in Sections \ref{Section:RS_ETC_M} and \ref{RS_UCB_M}. 
We define $\ba^\star$ 
as the counterpart of the $\balpha^\star$ in the discrete simplex, i.e., 
\begin{align}
\label{eq:discrete optimal mixture}
    \ba^\star \triangleq \argmax\limits_{\ba\in\Delta^{K-1}_{\varepsilon}}\; V(\ba, \F)\ .
\end{align}
Finally, we define the \emph{minimum sub-optimality gap} with respect to the discretization level \(\varepsilon\) as
\begin{align}
\label{eq:discrete_gap}
    \Delta_{\min}(\varepsilon)\;\triangleq\; \min\limits_{\ba\in\Delta^{K-1}_{\varepsilon} : \ba\neq \ba^\star} \{V(\ba^\star,\F) - V(\ba,\F)\} \ .
\end{align}

\subsection{Risk-sensitive ETC for Mixtures}
\label{Section:RS_ETC_M}
We propose the {\bf R}isk-{\bf S}ensitive {\bf ETC} for {\bf M}ixtures (RS-ETC-M) algorithm, following the ETC principles, albeit with important deviations needed to accommodate an implementation of mixture policies. 

The RS-ETC-M algorithm consists of an initial \emph{exploration} phase during which the arms are sampled uniformly for a fixed interval to form high-fidelity estimates of the arms' CDFs (in contrast to canonical ETC that estimates arms' mean values). The duration of this phase depends on the minimum sub-optimality gap. Subsequently, in the next phase, the ETC algorithm \emph{commits} to a fixed policy, which is a mixture of arms with pre-fixed mixing coefficients. The mixing coefficients are selected to maximize the estimate of the DR. We show that with a high probability, they are equal to $\ba^\star$. The main processes of this algorithm are explained next and its pseudocode is presented in Algorithm \ref{algorithm:RS-ETC-M-Alg}.

\textbf{Discretization.} We specify $\Delta^{K-1}_{\varepsilon}$ by uniformly discretizing the each coordinate of $\Delta^{K-1}$ into intervals of length $\varepsilon$, i.e.\footnote{When $\frac{1}{\varepsilon}$ is not an integer, we make up for the deficit/excess of the weights in the last coordinate.},
\begin{align}
\label{eq:etc_discrete_set}
     \Delta^{K-1}_{\varepsilon} \triangleq  \{\varepsilon\bn: \bn\in\{\N\cup\{0\}\}^K\ , \bone^\top \cdot \varepsilon \bn = 1  \}\ .
 \end{align}

 \textbf{Explore (estimate mixing coefficients).} The purpose of this phase is to form high-confidence estimates of the empirical arm CDFs. We specify a time instant $N(\varepsilon)$ that determines the duration of the exploration phase, and that is the instance by which we have confident-enough CDF estimates. The arms are selected uniformly, each $\lceil \frac{1}{K}N(\varepsilon)\rceil$ times. For a DR with a distortion function that has \holder continuity exponent $q$ and constant $\mcL$, the time instant $N(\varepsilon)$ becomes a function of $q$, $\mcL$, the horizon $T$, and the minimum sub-optimality gap $\Delta_{\min}(\varepsilon)$ as specified below.
\begin{align} 
\small
\label{eq:number of samples main paper}
 & N(\varepsilon) \triangleq 256K  \e \left(\frac{2K \mcL}{\Delta_{\min}(\varepsilon)}\right)^{\frac{2}{q}}\nonumber\\ & \qquad \times \bigg[\frac{32}{\sqrt{\rm e}} + \log^{\frac{1}{2}} \Big( 2K T^2   \big(\varepsilon^{-(K-1)} + 1\big) \Big) \bigg]^2\ .
\end{align}
This implies that $N(\varepsilon)$ scales as $O(\log T)$. We also define $M(\varepsilon) \triangleq \frac{N(\varepsilon)}{\log T}$, which based on~\eqref{eq:number of samples main paper} scales as $O(1)$. 
Next, using the collected samples during exploration, the RS-ETC-M algorithm constructs the empirical estimates of arms' CDFs, as specified in~\eqref{eq:empirical_CDF}.

\paragraph{Choosing the mixing coefficient: } At time instant $N(\varepsilon)$, the RS-ETC-M algorithm identifies the discrete mixture coefficients that maximize the DR using the empirical CDFs. These coefficients are denoted by $\ba^{\rm E}_{N(\varepsilon)}$, where 
\begin{align}
\small
    \label{eq:ETC_alpha}
    \ba^{\rm E}_{t}\;\in\; \argmax\limits_{\ba\in\Delta^{K-1}_{\varepsilon}} U_h\Bigg( \sum\limits_{i\in[K]} a(i)\; \F_{i,t}^{\rm E}\Bigg)\ . 
\end{align}
\paragraph{Commit (track mixtures).} For the remaining sampling instants $t\in[N(\varepsilon),T]$, the RS-ETC-M algorithm commits to selecting the arms such that their selection frequencies are as close as possible to $\ba^{\rm E}_{N(\varepsilon)}$. Until time instant $N(\varepsilon)$ the algorithm samples all arms uniformly. Uniform sampling results in some arms being sampled more than what the policy $\ba^{\rm E}_{N(\varepsilon)}$ dictates and hence, these arms will not be sampled again. Having over-sampled arms implies that some arms are under-sampled. In that case, these arms would be sampled such that the number of times they are chosen converges to the mixing coefficient $\ba^{\rm E}_{N(\varepsilon)}$. 

To formalize how to track the mixture, let $S$ be the set of first $K-1$ arms (or any desired set of arms). For each arm $i \in S$, the algorithm calculates the required number of samples, which is $T \times a_{N(\varepsilon)}(i)$ for a horizon~$T$. The sampling procedure proceeds as follows: 
\begin{enumerate}
\label{eq: ETC_samp}
    \item if $T a_{N(\varepsilon)}(i) > \lceil\frac{ N(\varepsilon)}{K} \rceil$ (insufficient exploration), then, the arm $i$ is sampled according to $T a_{N(\varepsilon)}(i)$ before the algorithm moves on to the next arm.
    \item if $T a_{N(\varepsilon)}(i) \leq  \lceil \frac{N(\varepsilon)}{K} \rceil$ (sufficient exploration), then, the arm $i$ is skipped.  
\end{enumerate}
The remaining sampling budget is allocated to arm $K$.

\begin{algorithm}
        \footnotesize
		\caption{RS-ETC-M}
		\label{algorithm:RS-ETC-M-Alg}
 		\begin{algorithmic}[1]
            \STATE \textbf{Input:} Minimum gap $\Delta_{\min}(\varepsilon)$, horizon $T$
            \STATE Sample each arm $\lceil N(\varepsilon)/K \rceil$ times and obtain observation sequences $\mcX_{\lceil N(\varepsilon)/K \rceil}(1),\cdots,\mcX_{\lceil N(\varepsilon)/K \rceil}(K)$
            \STATE \textbf{Initialize:} $\tau_K^{\rm E}(i) = \lceil N(\varepsilon)/K \rceil ;\forall\;i\in[K]$, empirical arm CDFs $\F^{\rm E}_{1, \lceil N(\varepsilon)/K \rceil},\cdots,\F^{\rm E}_{K, \lceil N(\varepsilon)/K \rceil}$
			\FOR{$t= K\lceil N(\varepsilon)/K \rceil + 1,\cdots,T$}
			    \STATE Select an arm $A_{t}$ via~\eqref{eq: ETC_samp} and obtain reward $X_t$\\
                \STATE Update the empirical CDF $\F^{\rm E}_{A_t,t}$ according to~\eqref{eq:empirical_CDF}
			\ENDFOR
 		\end{algorithmic}
	\end{algorithm}

\textbf{Discussion.} An important advantage of the RS-ETC-M algorithm is its computational simplicity. The algorithm involves uniform arms exploration for a finite interval followed by committing to a mixture estimate based on the data in the exploration phase.
In Section~\ref{RS_UCB_M}, we will discuss that the relative performances of RS-ETC-M versus the UCB-type counterpart depends on the choice of DR, and neither has a uniform regret advantage over the other. 

Despite its computational simplicity and its better regret guarantee for some DRs,
the RS-ETC-M algorithm has the crucial bottleneck of relying on the instance-specific gap information (through $N(\varepsilon)$), which may not always be available -- an issue that is addressed by the UCB-type algorithms.

\subsection{Risk-sensitive UCB for Mixtures}
\label{RS_UCB_M}
In this section, we present the {\bf R}isk-{\bf S}ensitive {\bf UCB} for {\bf M}ixtures (RS-UCB-M) algorithm, which does not require the information on the sub-optimality gap $N(\varepsilon)$. The salient features of this algorithm are (i) a distribution estimation routine for forming high-confidence estimates for arms' CDFs and subsequently mixture coefficients; and (ii) a sampling rule based on an {\em under-sampling} criteria to ensure that arm selections track the mixture coefficients.  The pseudocode of RS-UCB-M is presented in Algorithm~\ref{algorithm:B-UCB-M}.

\paragraph{Discretization.} Similar to RS-ETC-M, we uniformly discretize each coordinate of $\Delta^{K-1}$ into intervals of length $\varepsilon$ with the distinction that for RS-UCB-M we use the the intervals mid-points, i.e.,
\begin{align*}
    \Delta^{K-1}_{\varepsilon} \triangleq  \{\varepsilon\bn: \bn\in\Big\{\N\cup\{0\}\}^K\ , \bone^\top \cdot \varepsilon \big(\bn+\frac{1}{2}\big) = 1  \Big\}\ .
\end{align*}

Based on the concentration of CDF estimates \eqref{eq:meta_concentration}, at time $t$ and given the empirical CDFs $\{\F_{i,t}^{\rm U}:i\in[K]\}$, for arm $i\in[K]$ we define the {\em distribution confidence space} as the collection of all the distributions that are within a bounded $1-$Wasserstein distance of $\F_{i,t}^{\rm U}$. Specifically, for each $i\in[K]$, we define
\begin{align}
\label{eq:UCB_confidence_sets}
\mcC_t(i)\triangleq\bigg\{\eta \in  \Omega & :\norm{\F_{i,t}^{\rm U}-\eta}_{\rm W} 
\leq 16\ \frac{\sqrt{2 {\rm e} \log T } + 32}{\sqrt{\tau^{\rm U}_t(i)}} \bigg\}\  .
\end{align}
Next, we also need to estimate the optimal mixing coefficients. For this purpose, 
we apply the UCB principle. This, in turn, requires that we compute the following  {\em optimistic} estimates for the mixing coefficients:
\begin{align}
\label{eq:UCB_alpha}
    \ba^{\rm U}_t\in \argmax\limits_{\ba\in\Delta_{\varepsilon}^{K-1}}\;\max\limits_{\eta_i\in\mcC_t(i) , \forall i\in[K]}\; U_h\Big ( \sum\limits_{i\in[K]} a(i) \; \eta_i\Big )\ .
\end{align}

\paragraph{Track mixtures.} Once we have estimates of the optimal mixing coefficients, i.e., $\ba_t^{\rm U}$, we design an arm selection rule that translates the mixing coefficients and CDF estimates to arm selection choices. Designing such a rule requires addressing a few technical challenges. First, the optimal mixing coefficients might not be unique. 
Let us denote them by $\{\psi_\ell:\ell\in[L]\}$. It is critical to ensure that we \emph{consistently} track only one of these optimal choices over time. The reason is that if we track multiple mixtures, in aggregate, we will be tracking a mixture of $\{\psi_\ell:\ell\in[L]\}$, which is not necessarily optimal. Secondly, in the initial sampling rounds, the estimates $\balpha_t^{\rm U}$ are relatively inaccurate, and tracking them leads to highly sub-optimal decisions. Finally, we need a rule for translating the estimated mixtures to arm selections. Next, we discuss how we address these issues.

\paragraph{Tracking a single mixture.} To ensure tracking only one optimal mixture, at time $t$, the RS-UCB-M algorithm checks if the coefficient from the previous step, i.e., $\ba^{\rm U}_{t-1}$ also maximizes~\eqref{eq:UCB_alpha}. If it does, then $\ba^{\rm U}_{t-1}$ is chosen as a candidate optimistic estimate $\ba^{\rm U}_t$. Otherwise, any random candidate solving~\eqref{eq:UCB_alpha} is chosen. 

{\bf Initial rounds:} To circumvent estimation inaccuracies in the initial rounds, we introduce a \emph{short} forced exploration phase for each arm.  Specifically, for $\rho \in (0,1)$, for the first $K\lceil\rho T \epsilon/4 \rceil$ rounds, we explore the arms in a round-robin fashion to initiate the algorithm with sufficiently accurate estimates of the arm CDFs. 

{\bf Decision rules:}
We specify a rule that converts mixtures to arm selections. When all the arms are sufficiently explored, motivated by the effective approaches to best arm identification~\citep{pmlr-v49-garivier16a,jourdan2022,pmlr-v117-agrawal20a,mukherjee2023best} we 
sample the most {\em under-sampled} arm. 
An arm is considered under-sampled if it has been sampled less frequently than the rate indicated by the estimate $\ba^{\rm U}_t$, and has the largest gap between its current fraction and its estimated fraction $a^{\rm U}_t(i)$. At time  $t\geq K\lceil \rho T \varepsilon/4\rceil$, the 
arm selection rule is specified by
\begin{align}
\label{eq:UCB_sampling_rule}
A_{t+1}\;\triangleq\;\argmax\limits_{i\in[K]}\; \{ta^{\rm U}_t(i) - \tau_t^{\rm U}(i)\}\ .
\end{align}

\begin{algorithm}
        \footnotesize
		\caption{RS-UCB-M}
		\label{algorithm:B-UCB-M}
		
 		\begin{algorithmic}[1]
            \STATE \textbf{Input:} Exploration rate $\rho$, horizon $T$
            \STATE Sample each arm {$N(\rho, \varepsilon) \triangleq \lceil \rho T \varepsilon/ 4 \rceil$} times and obtain observation sequences $\mcX_{K N(\rho, \varepsilon)}(1),\cdots,\mcX_{K N(\rho, \varepsilon)}(K)$
            \STATE \textbf{Initialize:} {$\tau_{KN(\rho, \varepsilon)}^{\rm U}(i) = N(\rho, \varepsilon)$} $\;\forall i\in[K]$, emprical arm CDFs $\F^{\rm U}_{1, KN(\rho, \varepsilon)},\cdots,\F^{\rm U}_{K, KN(\rho, \varepsilon)}$, confidence sets $\mcC_{KN(\rho, \varepsilon)}(1),\cdots\mcC_{KN(\rho, \varepsilon)}(K)$ according to~\eqref{eq:UCB_confidence_sets}
			\FOR{$t=KN(\rho, \varepsilon)+1,\cdots,T$}
			    \STATE Select an arm {$A_t$} {via}~(\ref{eq:UCB_sampling_rule}) and obtain reward $X_t$\\
                \STATE Update the empirical CDF $\F^{\rm U}_{A_t,t}$ according to~\eqref{eq:empirical_CDF}
                \STATE Update the confidence set $\mcC_t(A_t)$ according to~\eqref{eq:UCB_confidence_sets} 
                \STATE Compute the optimistic estimate $\ba_t^{\rm U}$ according to~\eqref{eq:UCB_alpha}
			\ENDFOR
 		\end{algorithmic}
	\end{algorithm}

\subsection{Computationally Efficient CE-UCB-M }
In the RS-UCB-M algorithm, determining the mixing coefficients $ \ba^{\rm U}_t$ via~\eqref{eq:UCB_alpha} involves extremization over a class of distribution functions, and it is computationally expensive. To circumvent this, we present the {\bf C}omputationally-{\bf E}fficient risk-sensitive {\bf UCB} for {\bf M}ixtures (CE-UCB-M) algorithm as a computationally tractable modification of RS-UCB-M. In CE-UCB-M, instead of solving ~\eqref{eq:UCB_alpha}, for any given set of CDF estimates $\{\F_{i,t}^{\rm U}:i\in[K]\}$ and mixing vector $\ba$, we define
\begin{align*}
\small
\nonumber
     {\rm UCB}_t(\ba) & \;\triangleq\;  U_h\Big( \sum\limits_{i\in[K]} a(i)\; \F_{i, t}^{\rm C}\Big) \\ & \;\;  + \mcL \sum\limits_{i\in[K]} \bigg( a(i) \cdot 16\; \frac{\sqrt{2 {\rm e} \log T } + 32}{\sqrt{\tau^{\rm U}_t(i)}} \bigg)^q \ ,
\end{align*}
where $q$ and $\mcL$ are the \holder parameters of the underlying distortion function $h$. We specify estimates of the  mixing coefficients as
\begin{align}
\label{eq:UCB_alpha2}
    \ba_t^{\rm C}\;\in\; \argmax\limits_{\ba\in\Delta_{\varepsilon}^{K-1}}\; {\rm UCB}_t(\ba)\ .
\end{align}   
The remainder of the algorithm (the tracking block) follows the same steps as in RS-CS-UCB-M. The CE-UCB-M procedure is summarized in Algorithm~\ref{algorithm:UCB-M} in Appendix~\ref{Appendix:CE-UCB-M-alg}.

\textbf{Discussion.} Unlike RS-ETC-M,  RS-UCB-M and the CE-UCB-M are independent of instance-dependent parameters. The explicit exploration phase involves a hyperparameter $\rho$, which must be bounded away from $0$. The performance of RS-ETC-M and RS-UCB-M depends on the choice of DR, with neither uniformly dominating the other. In Section~\ref{sec:analysis} we show that RS-ETC-M has better regret guarantees for DRs favoring a solitary arm policy.
In contrast, for the DRs that have a mixture optimal policy, the performance advantage 
depends on the DR. For instance, for Wang's Right-tail deviation, RS-ETC-M is better, whereas for  Gini deviation RS-UCB-M outperforms RS-ETC-M.

\vspace{-.05 in }
\section{REGRET ANALYSIS}
\label{sec:analysis}
\vspace{-.05 in }

In this section, we characterize regret guarantees for the three algorithms presented in the previous section. 
We provide a decomposition for the regret to two terms, where one term accounts for the discretization inaccuracies and one accounts for the scaling behavior in terms of $T$, which we refer to as the {\em discrete regret}. Based on this decomposition, we present a regret bound in terms of $\varepsilon$ for any desired discretization level. Subsequently, we also provide an $\varepsilon$-independent regret in which $\varepsilon$ is chosen carefully to achieve the best performance subject to algorithmic constraints.

 An important observation is the following contrast between the regrets characterized for the mixture algorithms and their canonical ETC and UCB  counterparts: the canonical ETC and UCB algorithms generally exhibit the same regret, even though ETC requires access to instance-dependent parameters. In the mixtures setting, however, access to instance-dependent parameters yields better regret guarantees for the ETC-type algorithm (i.e., RS-ETC-M).

For a given discretization level $\varepsilon$ and a bandit instance $\bnu\triangleq (\F_1,\cdots,\F_K)$,  for policy $\pi$, we decompose the regret defined in~\eqref{eq:regret} into a discretization error component and a discrete regret component as follows:
\begin{align}
\label{eq:regret_decomposition}
    \mathfrak{R}^\pi_{\bnu}(T)\; & =\; \Delta(\varepsilon) +  \bar{\mathfrak{R}}_{\bnu}^\pi(T)\ ,
\end{align}
where based on the definition of $\ba^\star$ in~\eqref{eq:discrete optimal mixture} we have
\begin{align}
\label{eq:regret_decomposition1} \Delta(\varepsilon) & \triangleq V(\balpha^\star,\F) - V(\ba^\star,\F)  \ , \\ \label{eq:regret_decomposition2}
    \bar{\mathfrak{R}}_{\bnu}^\pi(T) & 
    \triangleq V(\ba^\star,\F) - \E_{\bnu}^\pi\Big[V\Big(\frac{1}{T}\btau_T^\pi, \F \Big)\Big] \ .
\end{align}
In the decomposition in \eqref{eq:regret_decomposition}, $\Delta(\varepsilon)$ accounts for the discretization error and $\bar{\mathfrak{R}}_{\bnu}^\pi(T)$ represents the \emph{discrete} regret. We begin by presenting the regret guarantees for the RS-ETC-M algorithm. 
\vspace{-0.08in}
\begin{theorem}[RS-ETC-M -- $\varepsilon$-dependent]
\label{theorem: ETC upper bound}
For any $\varepsilon\in\R_+$ and distortion function $h$ with H\"older exponent $q$, for all $T>N(\varepsilon)$, 
RS-ETC-M's regret
is upper bounded as
    \begin{align*}
\mathfrak{R}_{\bnu}^{{{\rm E}}}(T)\ &
   \leq   (\mcL K + W^{-q}) \bigg(3 WM(\varepsilon)\;  \frac{\log T}{T} \bigg)^q + \Delta(\varepsilon)\ . 
    \end{align*}
\end{theorem}
\vspace{-0.1in}

\paragraph{Choosing $\varepsilon$.} Theorem~\ref{theorem: ETC upper bound} is valid for any $\varepsilon\in\R_+$, allowing freedom to appropriately choose the desired accuracy for the mixing coefficients. Observe that $\Delta(\varepsilon)$ in proportional to $\varepsilon$, and $N(\varepsilon)$ is inversely proportional to $\varepsilon$. Hence, while arbitrarily diminishing $\varepsilon$ decreases the discretization error, it may violate the condition that $T>N(\varepsilon)$. We chose $\varepsilon$ to be small enough (small discretization error), while conforming to the condition $T>N(\varepsilon)$ in Theorem~\ref{theorem: ETC upper bound} (feasibility). The minimum feasible $\varepsilon$ depends on $q$, $r$, and its connection to the sub-optimality gap as follows. Let us define
\begin{align}
    \beta\triangleq \lim_{\varepsilon \to 0}\frac{\log \Delta_{\min}(\varepsilon)}{\log \varepsilon}\ ,
\end{align}
 which quantifies how fast the minimum sub-optimality gap $\Delta_{\min}(\varepsilon$) diminishes as $\varepsilon$ tends to 0. 
Appendix \ref{Appendix:Additinal Tables} provides characterizes $\beta$ values for some DRs. Let $\varepsilon = \Theta((K^{2+2/q}T^{-\gamma}\log T)^{q/2\beta})$ where we set $\gamma\triangleq 2\beta /(2\beta + r)$, which leads to the following regret bound.
\vspace{-0.05in}
\begin{theorem}[RS-ETC-M -- $\varepsilon$-independent]
\label{theorem:RS-ETC-M}
Under the conditions of Theorem \ref{theorem: ETC upper bound}, when $\beta$ exists, the minimum feasible regret RS-ETC-M satisfies is
\begin{align}
\mathfrak{R}_{\bnu}^{{\textnormal{\rm E}}}(T)\ & \leq O\Bigg(\Bigg[K^{c_{\rm E}}\cdot  \frac{\log T}{T^{\gamma}}\Bigg]^{\frac{qr}{2\beta}}\Bigg)\ ,
\end{align}
where $c_{\rm E}\triangleq {2(1+\frac{\beta+1}{q})}$ and $\gamma = \frac{2\beta}{2\beta+r}$.
\end{theorem}

Next, we present the regret guarantees for the RS-UCB-M and CE-UCB-M algorithms. An important observation is that these algorithms yield weaker \emph{discrete} regret guarantees compared to RS-ETC-M. The regret degradations are the expense of not knowing the instance-dependent gaps. For stating the theorem, we define an instance-dependent finite time instant $T(\varepsilon)$. 
\begin{align}
  T_0(\varepsilon)\; \triangleq \;\inf \{ t\in\N : \forall s \geq t\ , \; s\in\mcQ \}\ ,
\end{align}
where we have defined
\begin{align}
    \mcQ\triangleq \left\{s: \frac{ \sqrt{2\e\log s} + 32}{\sqrt{ \frac{\rho}{4} s \varepsilon }}
     \leq \frac{1}{16} \left( \frac{\Delta_{\min}(\varepsilon)}{2K\mcL}\right)^{\frac{1}{q}} \right\}\ .
\end{align}
Accordingly, we define
\begin{align}
    \textstyle T(\varepsilon)\;\triangleq\; \frac{2}{\varepsilon}\Big(T_0(\varepsilon) -1\Big) \ .
\end{align}
The next theorem presents a $\varepsilon$-dependent regret for the RS-UCB-M and CE-UCB-M algorithms. 
\begin{theorem}[RS/CE-UCB-M -- $\varepsilon$-dependent]
\label{theorem:UCB upper bound}
For any $\varepsilon\in\R_+$, and distortion function $h$ with \holder exponent $q$, for all $T>\max\{\e^K,T(\varepsilon)\}$,  for $\pi\in\{{\rm U,C}\}$ we have \begin{align}
    \label{eq:UCB regret}
    \nonumber
         \mathfrak{R}_{\bnu}^{\pi }(T) & \leq [B + \mcL(W^q+1)] \Bigg[\frac{64}{\sqrt{\varepsilon \rho T}}\Big( \sqrt{2\e\log T} + 32\Big)\Bigg]^q 
          \\
         & + \Delta(\varepsilon) \ .
    \end{align}
\end{theorem}
\textbf{Proof sketch.} The proof differs significantly from the standard UCB-type analyses. For vanilla UCB, the general proof uses the fact that selecting any suboptimal arm more than $O(\log T)$ times is unlikely, and hence, the overall regret is bounded by $O(\frac {1}{T}K\log T)$. However, in our analyses, we have the new dimension of estimating the mixing coefficients and need these estimates to converge to an optimal choice. Such convergence requires that all arms to be sampled at a rate linear in $T$  (unless an arm's mixing coefficient is 0). 
Hence, selecting any arm $O(\log T)$ times is insufficient.

The other key difference pertains to finding a bound on the mixing coefficient errors (estimation and convergence). Characterizing this bound hinges on two key steps: (i) the convergence of the UCB estimates ($\ba_t^{\rm U}$ for RS-UCB-M and $\ba_t^{\rm C}$ for CE-UCB-M) to the discrete optimal solution $\ba^\star$ in probability, and (ii) a sublinear regret incurred in the process of tracking the mixing coefficient estimates using under-sampling. The first step is analyzed in Appendix~\ref{appendix: UCB sampling estimation error}, where we show that the probability of error for the RS-UCB-M and CE-UCB-M algorithms in identifying the discrete optimal mixture is upper bounded by $T((\frac{1}{T^2} + 1)^K-1)$. The second step is analyzed in Lemma~\ref{lemma:undersampling} in Appendix~\ref{appendix: UCB sampling estimation error}, in which we show that the regret incurred by the tracking block of the RS-UCB-M and CE-UCB-M algorithms is of the order $O(K/T)$. The regret upper bound is a combination of the regrets in these results.

\textbf{Choosing $\varepsilon$.} Similarly to RS-ETC-M, we choose an $\varepsilon$ that ensures $T > T(\varepsilon)$ for the bound in Theorem \ref{theorem:UCB upper bound}. The choice 
\begin{align*}
    \varepsilon = \Theta \left(\left(K^{\frac{2}{q}}\; \frac{\log T}{T}\right)^{\kappa}\right)\ , \quad \mbox{where} \;\; \kappa \triangleq \Big(\frac{2\beta}{q}+2\Big)^{-1}\ .
\end{align*}
\begin{theorem}[RS/CE-UCB-M -- $\varepsilon$-independent]
\label{corollary:RS-UCB-M}
Under the conditions of Theorem \ref{theorem:UCB upper bound}, when $\beta$ exists and $r \leq \beta + \frac{q}{2} $ the regret upper bound for $\pi\in\{{\rm U,C}\}$ is
\begin{align}
\mathfrak{R}_{\bnu}^{\pi}(T) \leq O \left (K^{c_{\rm U}}  \Big[\frac{\log T}{T}\Big]^{r \kappa} \right) \ ,
\end{align}
where $c_{\rm U}\triangleq \max\{ r(1+ \frac{2\kappa}{q}),{1-\kappa} \}$.
\end{theorem}
\vspace{-0.1in}
In the case of Gini deviation, we observe that for discrete regret, the RS-ETC-M algorithm achieves \(O(K/T)\) regret while RS-UCB-M achieves \(O(K/\sqrt{T})\) regret, ignoring polylogarithmic factors. However, when we optimize the algorithms for the discretization level \(\varepsilon\), we observe an order-wise improvement in the regret of the RS-UCB-M and CE-UCB-M algorithms.  


\textbf{Regret bounds for important cases.} We specialize our general results to a number of widely-used DRs in Table~\ref{table:table_regrets} and Table~\ref{table:table_risks} (Appendix \ref{appendix:lit_comp}). Table~\ref{table:table_regrets} provides the regret bounds and Table~\ref{table:table_risks} specifies the \holder continuity constants $q$, $r$, and $\beta$. We remark that the $\varepsilon$-dependent regret bounds depend only on the \holder continuity constant $q$ (Theorems~\ref{theorem: ETC upper bound} and \ref{theorem:UCB upper bound}). On the other hand, the $\varepsilon$-independent regret bounds, will additionally depend on the mixture \holder continuity constant $r$ and $\beta$ (Theorems~\ref{theorem:RS-ETC-M} and \ref{corollary:RS-UCB-M}). 
It is worth noting that, 
for RS-ETC-M algorithm, the constant $\beta$ will be relevant only for the DR choices with mixture policies. The reason is that when we have a solitary arm policy, we can discretize the simplex into unit vectors along each coordinate. In this case $\Delta_{\min}(\varepsilon) = O(1)$, which does not scale with $\varepsilon$. Consequently, for RS-ETC-M and for solitary arms, the order of the regret is the same as the discrete regret. 

\begin{figure*}[ht!]
    \centering
    \begin{subfigure}[t]{0.40\textwidth}
        \centering
    \includegraphics[height=1.5in]{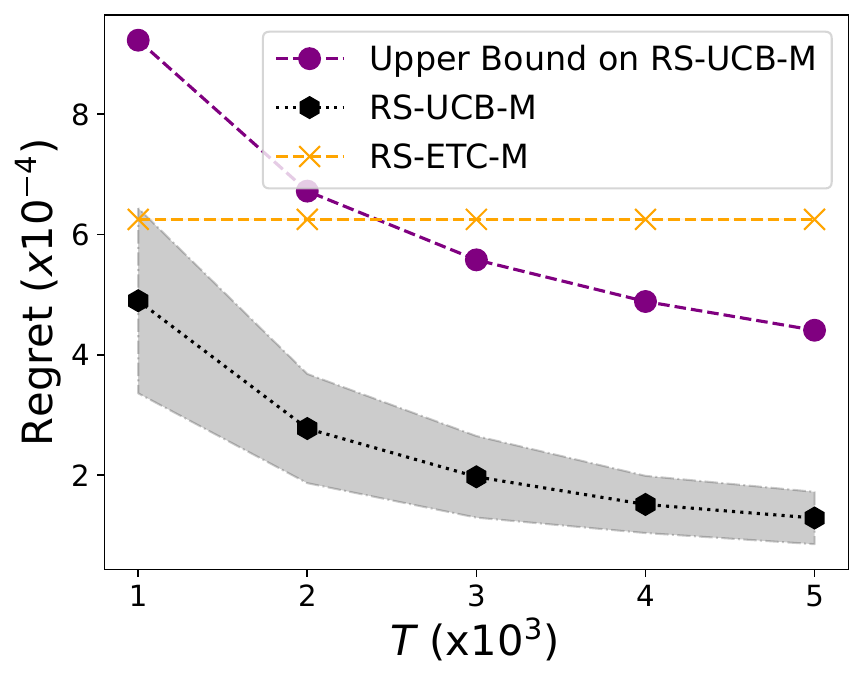}
        \caption{Regret versus time horizon $T$.}
        \label{fig:T_experiment}
    \end{subfigure}%
    ~ 
    \begin{subfigure}[t]{0.40\textwidth}
        \centering        \includegraphics[height=1.5in]{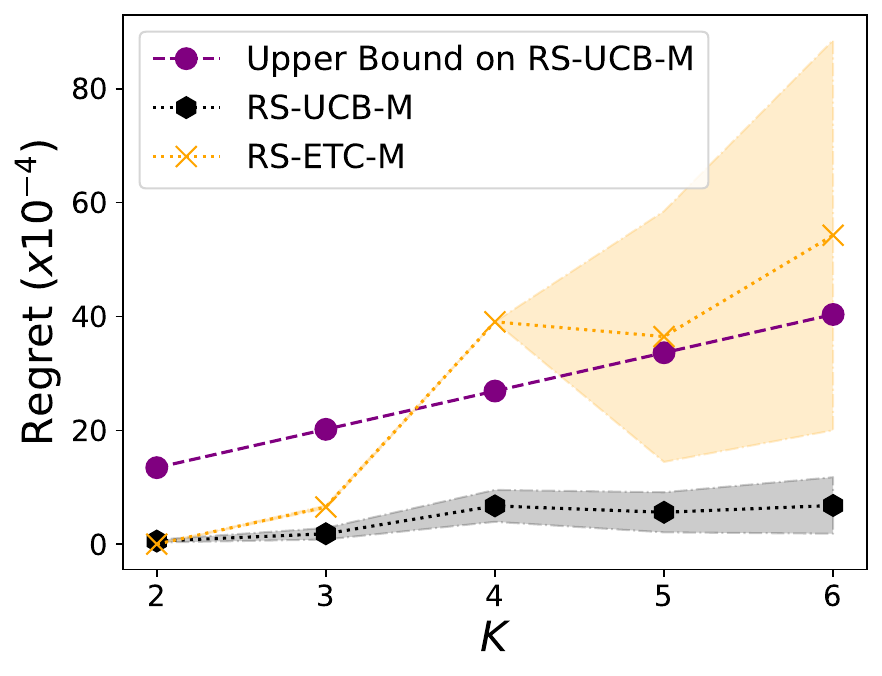}
        \caption{Regret versus number of arms $K$.}
        \label{fig:K_experiment}
    \end{subfigure}
    \caption{Regret of the algorithms for different parameters}
\end{figure*}
\paragraph{Relevance to the existing literature.} In Section~\ref{appendix:lit_comp} we provide a thorough discussion on the relevance of our general results to those of the existing risk-sensitive literature. In summary, the existing literature has studied only the case of monotone DRs, 
and solitary arm policies including \citet{liang2023distribution} and \citet{prashanth2022wasserstein}, which proposes Wasserstein distance based UCB algorithms. 
We present a detailed comparison of regret bounds, which depend on whether the underlying distributions are bounded or sub-Gaussian, in Table \ref{table:literature_comparison} (Appendix \ref{appendix:lit_comp}). For the well-investigated CVaR, the existing literature establishes a regret bound of $O(\log T/T)$ for bounded support \citep{baudry2021optimal} and $O(\sqrt{\log T/T)}$ for sub-Gaussian distributions \citep{prashanth2022wasserstein}. Under the same sub-Gaussian assumption, RS-ETC-M improves the regret bound to \(O(\log T/T)\).


\section{EMPIRICAL EVALUATIONS}
\label{sec: experiments}

We provide empirical evaluations of the RS-ETC-M and RS-UCB-M algorithms and compare them with the following baseline algorithm: a uniform sampling strategy that selects each arm $T/K$ times, where $T$ denotes the horizon. We also note that RS-ETC-M assumes knowing instance-specific gap information, while RS-UCB-M does not. The experiments are conducted for the DR Gini deviation, whose distortion function is $h(p) = p(1-p)$. We focus on Bernoulli bandits with mean vectors $\bp\in[0,1]^K$, in which case, the DR has a closed-form expression for the case of mixtures, given by $V(\balpha,\F)\;=\; \langle \balpha , \bp\rangle (1-\langle \balpha, \bp\rangle )$. All the experiments are averaged over $10^3$ trials. \footnote{https://github.com/MeltemTatli/Risk-sensitive-Bandits-Arm-Mixture-Optimality.git}

We empirically investigate regret's variations with varying levels of the horizon $T$ and number of arms $K$ as well as other properties of the algorithms, including their sensitivity to the exploration rate $\rho$ and instance-specific gap parameters. Regret evaluations are presented in this section, and analyzing the properties is presented in Appendix~\ref{Appendix: Additional Experiments}. 

\textbf{Regret versus horizon.} Our first experiment  considers a $2$-armed bandit instance $\bnu \triangleq \big({\rm Bern}(0.4), {\rm Bern}(0.9) \big)$. For this instance, note that the optimal mixing coefficient is $\balpha^\star = [0.8, 0.2]^\top$. For implementing the RS-UCB-M algorithm, we have set the discretization level to $\varepsilon = \sqrt{(K\log T)/T}$,
and explore each arm $\frac{T}{10}$ times. 
The results are presented in Figure~\ref{fig:T_experiment} and demonstrate the regret variations of RS-UCB-M and RS-ETC-M versus $T$.

For RS-UCB-M, we present the range of the results in the shaded region with their average specified by the dashed curve. We do the same for RS-UCB-M, but its range is notably narrow and invisible on the curve. We observe that RS-UCB-M uniformly outperforms RS-ETC-M. This is expected as the exploration phase of the RS-ETC-M algorithm explores an arm more frequently than the associated estimated mixing coefficient recommends.
Additionally, Figure~\ref{fig:regret_uniform} (Appendix~\ref{Appendix: Additional Experiments}) demonstrates that 
RS-UCB-M outperforms uniform sampling. 

\paragraph{Regret versus number of arms.} Next, we investigate the variations of RS-UCB-M's regret in the number of arms $K$. We choose $K\in\{2,3,4,5,6\}$. For fair comparisons, we create the instances so that they have uniform gaps between the arm means for every instance. For both RS-UCB-M and RS-ETC-M, we set the discretization level to $\varepsilon = \sqrt{(K\log T)/T}$; set the horizon to $T = 3\times 10^5$; and explore every arm $\frac{T}{20}$ times for RS-UCB-M algorithm.
Figure~\ref{fig:K_experiment} shows that, on average, the RS-UCB-M algorithm outperforms the RS-ETC-M, especially as the number of arms increases.

\section{CONCLUDING REMARKS}

We have provided a novel approach to viewing and analyzing stochastic bandits in which the decision-makers are expected to be conscious of the decision risks involved. This contrasts with the conventional risk-neutral approaches, which aim to optimize the average reward in a fully utilitarian way. Adopting a rich class of distortion riskmetrics, we observed that many of deviation measures and measures of variability are optimized by a mixture distribution over the arms. This is in sharp contrast to the commonly adopted premise that there exists a solitary best arm -- a premise shared by the existing literature on risk-sensitive bandit algorithms, too. Designing regret-efficient algorithms for such mixtures poses various technical and design challenges, mainly pertinent to identifying and tracking the optimal mixing coefficients of the arms. Based on the UCB and ETC principles, we have designed two sets of bandit algorithms and established regret results for a broad spectrum of commonly used distortion riskmetrics. A potential future direction is finding a general lower-bound for distortion riskmetrics in this setting, which is in general uninvestigated for risk-sensitive bandits.

 \subsubsection*{Acknowledgements}
A portion of this work was done when Prashanth L. A. was at IIT Bombay. The work of Meltem Tatl{\i}, Arpan Mukherjee, and Ali Tajer was supported in part by the U.S. National Science Foundation under Grants ECCS-193310 and DMS-2319996, and in part by the Rensselaer-IBM Future of Computing Research Collaboration (FCRC).






\bibliographystyle{unsrtnat}
\bibliography{BAIRef, references}

\section*{Checklist}



 \begin{enumerate}

 \item For all models and algorithms presented, check if you include:
 \begin{enumerate}
   \item A clear description of the mathematical setting, assumptions, algorithm, and/or model. Yes. Please see \ref{sec:setting} and \ref{sec:algorithm}.
   \item An analysis of the properties and complexity (time, space, sample size) of any algorithm. Yes. Please see \ref{sec:algorithm} on more discussion on algorithms.
   \item (Optional) Anonymized source code, with specification of all dependencies, including external libraries. Yes. 
 \end{enumerate}

 \item For any theoretical claim, check if you include:
 \begin{enumerate}
   \item Statements of the full set of assumptions of all theoretical results. Yes. All of our theorems were included in section \ref{sec:analysis}.
   \item Complete proofs of all theoretical results. Yes. Proof overview is given in \ref{sec:analysis}. Due to the space constraints, the proofs are postponed to the Appendix \ref{proof:lemma_example}, \ref{Appendix:W_finitess}, \ref{proof:Aux_Lemma}, \ref{Appendix:RS-ETC-M}, \ref{Appendix:Upper_last_UCB}, and \ref{proof:UCB upper bound}.
   \item Clear explanations of any assumptions. Yes. We have included our assumptions in section \ref{sec:setting}.   
 \end{enumerate}

 \item For all figures and tables that present empirical results, check if you include:
 \begin{enumerate}
   \item The code, data, and instructions needed to reproduce the main experimental results (either in the supplemental material or as a URL). Yes. We provided the code as an URL.
   \item All the training details (e.g., data splits, hyperparameters, how they were chosen). Not Applicable.
         \item A clear definition of the specific measure or statistics and error bars (e.g., with respect to the random seed after running experiments multiple times). Yes. We have provided the details in Section \ref{sec: experiments} and Appendix \ref{Appendix: Additional Experiments}.
         \item A description of the computing infrastructure used. (e.g., type of GPUs, internal cluster, or cloud provider). Yes. We have provided the details in Appendix \ref{Appendix: Additional Experiments}.
 \end{enumerate}

 \item If you are using existing assets (e.g., code, data, models) or curating/releasing new assets, check if you include:
 \begin{enumerate}
   \item Citations of the creator If your work uses existing assets. Not Applicable
   \item The license information of the assets, if applicable. Not Applicable
   \item New assets either in the supplemental material or as a URL, if applicable. Not Applicable.
   \item Information about consent from data providers/curators. Not Applicable
   \item Discussion of sensible content if applicable, e.g., personally identifiable information or offensive content. Not Applicable.
 \end{enumerate}

 \item If you used crowdsourcing or conducted research with human subjects, check if you include:
 \begin{enumerate}
   \item The full text of instructions given to participants and screenshots. Not Applicable
   \item Descriptions of potential participant risks, with links to Institutional Review Board (IRB) approvals if applicable. Not Applicable.
   \item The estimated hourly wage paid to participants and the total amount spent on participant compensation. Not Applicable.
 \end{enumerate}

 \end{enumerate}

 \newpage

\clearpage
\appendix
\onecolumn

\hsize\textwidth
 \linewidth\hsize 
\hrule height4pt
\vskip .25in
{\centering
  {\Large\bfseries Risk-sensitive Bandits:\\ Arm Mixture Optimality and Regret-efficient Algorithms\\ \vspace{.15 in}Supplementary Materials \par}}
\vskip .25in
\hrule height1pt
\vskip .25in

\doparttoc 
\faketableofcontents 

\part{} 
\parttoc 

\newpage

\section*{Organization of the Supplementary Material}

The supplementary material consists of eight parts grouped in Appendices \ref{Appendix:Specializing}-\ref{Appendix: Additional Experiments}. 
 We begin by presenting the \holder parameters used in the theorems for DRs listed in Table \ref{table:table_regrets}, along with a comparison to the risk-sensitive bandit literature (Appendix \ref{Appendix:Specializing}). The proof of Lemma~\ref{example utility} is presented in Appendix \ref{proof:lemma_example} and the proof of the finiteness of the Wasserstein-related constant $W$ is presented in Appendix~\ref{Appendix:W_finitess}. Subsequently, in Appendix \ref{proof:Aux_Lemma}, we provide several auxiliary lemmas that are used for the proofs of regret results presented in Theorems \ref{theorem: ETC upper bound}-\ref{corollary:RS-UCB-M} and the subsequent appendices. Next, we have Appendices \ref{Appendix:RS-ETC-M}-\ref{Appendix:Upper_last_UCB} to present the CE-UCB-M pseudo-algorithm (Appendix~\ref{sec:SE-UCB-M proofs}), as well as the performance guarantees for the RS-ETC-M (Appendix~\ref{Appendix:RS-ETC-M}), RS-UCB-M (Appendix~\ref{proof:UCB upper bound}), and CE-UCB-M (Appendix~\ref{sec:SE-UCB-M proofs}) algorithms. Finally, we provide additional empirical results in Appendix \ref{Appendix: Additional Experiments}.

\section{Specializing the Results to Distortion Riskmetrics}
\label{Appendix:Specializing}
In this section, we present the parameters used in Theorems \ref{theorem: ETC upper bound}-\ref{corollary:RS-UCB-M} for some of the widely-used DRs and compare our results with known regret guarantees in the literature.

\subsection{Distortion Riskmetric Examples and Parameters}
\label{Appendix:Additinal Tables}

\paragraph{\holder exponents:} In this section, we present the \holder exponents  $q,r$ for the DRs discussed in Table~\ref{table:table_regrets}. These constants are reported in Table~\ref{table:table_risks}. For the \holder exponents, we note that the expressions of $q$ and $r$ depend on the bandit model. Furthermore, based on the definitions of \holder continuity and mixture \holder continuity provided in \eqref{eq:Holder} and \eqref{eq:Holder_opt}, respectively, any given \holder continuity exponent $r$ is also a mixture \holder continuity exponent, that is given $q$, $r=q$ is always a valid choice for $r$. We characterize the \holder exponents and constants for general for
Bernoulli bandits in Lemmas \ref{lemma:example_utility_Holder}--\ref{lemma:Quadratic} in Appendix \ref{proof:Aux_Lemma}.  Table~\ref{table:table_risks} summarizes these results. 

\paragraph{Gap constant $\beta$:} We also analyze 
the gap constant $\beta$, which we have defined as 
\begin{align}
\beta= \lim_{\varepsilon \to 0}\frac{\log \Delta_{\min}(\varepsilon)}{\log \varepsilon} \ ,    
\end{align}
which plays an important role in characterizing $\varepsilon$-independent regret guarantees. We characterize $\beta$ for DRs with strictly increasing distortion functions (risk-neutral mean value, dual power, quadratic, and CVaR under a condition) in Lemma \ref{lemma:beta_lemma_monotone}.  Table~\ref{table:table_risks} lists these $\beta$ values.





\begin{table}[h]
\captionsetup{position=top}
\caption{\holder  continuity parameters and distortion functions of the DRs reported in Table~\ref{table:table_regrets}.}
\label{table:table_risks}
\begin{minipage}[t]{\textwidth}
\centering
\begin{tabular}{@{}|l|c|l|l|l|l|l|@{}}
\hline 
\textbf{Distortion Riskmetrics }  & $q $ \footnote{We report $q$ for $K$-arm Bernoulli bandit distributions except for CVaR and PHT Measure.} & $r$ \footnote{We report $r$  for $K$-arm Bernoulli bandit distributions.} & $\beta$ \footnote{We report $\beta$ for $K$-arm Bernoulli bandit distributions.} \footnote{The non-shaded DRs have strictly increasing DRs and their $\beta$ values are characterized analytically.} 
\\ \hline\hline 
Risk-Neutral Mean Value                      & $1$ & $1$ & 1\\ 
Dual Power  \citep{vijayan2021policy}                  &  $1$& $1$ & 1  \\
Quadratic    \citep{vijayan2021policy}                  &   1 & $1$  & 1 \\
CVaR   \citep{dowd2006}  & $1$ \footnote{\citet{prashanth2022wasserstein} reports $q=1$ for sub-gaussian distributions.} & $1$ & $1$ \footnote{When the arms' mean values are smaller than $1-\alpha$. This assumption is only needed for the analysis of $\beta$. Hence, it is needed for only the regret of RS-UCB-M.}\\ 
PHT Measure (\(s \in (0,1)\)) \citep{jones03}                   &   $s$ & $s$   & 1\\
\rowcolor{gray!30}
Mean-median deviation  \citep{jones03} \citep{Wang2020}    & $1$   & $1$  & $-$\\ 
\rowcolor{gray!30}
Inter-ES Range \({\rm IER}_{\alpha=1/2}\) \citep{Wang2020}& $1$ & $1$ &$-$ 
\\
\rowcolor{gray!30}
Wang's Right-Tail Deviation  \citep{jones03}     & $1/2$   & $1$  & $-$\\ 
\rowcolor{gray!30}
Gini Deviation & $1$ & $2$ &$-$\\
\bottomrule
\end{tabular}
\end{minipage}
\end{table}

\newpage

\subsection{Comparison with Risk-Sensitive Bandit Literature}
\label{appendix:lit_comp}

In this subsection, we discuss two broad categories of the existing studies. First, we discuss the existing studies focused on risk-sensitive bandits and establish their connection to our framework and regret results. Secondly, we discuss the empirical distribution performance measures (EDPM) framework~\citep{cassel2018general}.

\paragraph{Frameworks subsumed by the DR-centric framework.}   As discussed earlier, the DR-centric framework subsumes several existing frameworks for risk-sensitive bandits. we note that the existing literature has considered only monotone DRs and they include CVaR and distortion risk measures (DRM), which we discuss next.


\paragraph{CVaR:} For CVaR, the best-known regret is \(O(\log(T)/T)\) under the assumption of bounded support \citep{baudry2021optimal, cassel2018general, tamkin2019cvar}. When the assumption is loosened to sub-Gaussian, the regret weakens to $O(\sqrt{\log(T)/T})$ for \citep{prashanth2022wasserstein}. Compared to these known regret bounds, for sub-Gaussian models, RS-ETC-M achieves a regret bound of \(O(\log(T)/T)\) while for Bernoulli model RS-UCB-M achieves a regret bound of $O\left((\log(T)/T)\right)^{1/4}$. The reason that RS-ETC-M outperforms RS-UCB-M is that it assumes knowing the minimum suboptimality gap \(\Delta_{\min}\).

\paragraph{DRMs:} For the more general class of DRMs, under the assumption of bounded support, the study in \citep{chang2022, aditya2016weighted} reports a regret of \(O(\log(T)/T)\).
For the sub-Gaussian support, \citep{prashanth2022wasserstein} reports a regret bound of $O(\sqrt{\log(T)/T})$. In the sub-Gaussian setting, the regret bound of RS-ETC-M depends on \holder exponent \(q\), and hence, for the DRMs with \(q=1\), we achieve the same regret as in CVaR. For \(q < 1\), which is the case for PHT with \(s = 1/2\), we report a regret of $O(\sqrt{\log(T)/T})$ for RS-ETC-M and under Bernoulli model we report $O({(\log(T)/T)}^{1/5})$ for RS-UCB-M. All these results are summarized in Table \ref{table:literature_comparison}.

\paragraph{EDPM Framework.} The EDPM framework introduced in~\citep{cassel2018general} is, in principle, more general than the DR-centric framework, and it includes risk measures that cannot be modeled by DRs, e.g., variance and Sharpe Ratio. Nevertheless, the analysis of the EDPM is focused entirely on settings in which the optimal policy is necessarily a solitary arm policy. Such focus precludes risk measures with optimal mixture policies, e.g., Gini deviation and Wang's Right-tail deviation.

\begin{table}[h]
\captionsetup{position=top}
\caption{Comparison of the existing risk-sensitive bandit studies where.}
\label{table:literature_comparison}
\begin{minipage}[t]{\textwidth}
\centering
\begin{tabularx}{\textwidth}{@{}|l|c|X|l|@{}}
\hline 
\textbf{Work} & \textbf{DR} & \textbf{Distribution Assumptions} & \textbf{Regret ($\varpi(T)\triangleq  \sqrt{\frac{\log T}{T}}$)}  \\ \hline \hline
\citet{baudry2021optimal}     & CVaR        & Bounded & $O(\varpi^2(T))$      \\ \hline
\citet{tamkin2019cvar}    & CVaR        & Bounded & $O(\varpi^2(T))$ \footnote{This is a problem-dependent bound.}  \\  
\hline
\citet{cassel2018general}      & CVaR 
& Bounded       & $O(\varpi^2(T))$\footnote{This bound holds under some assumptions on the arm distribution densities.} %
\\
\hline
\citet{chang2022}      & DRM         & Bounded & $O(\varpi^2(T))$   \\ 
\hline
\citet{aditya2016weighted}       & DRM         & Bounded                & $O(\varpi^2(T))$ 
\\
\hline
\citet{liang2023distribution}      & DRM         & Bounded                  & $O(\varpi^2(T))$ \\
\hline
\citet{prashanth2022wasserstein}      & DRM         & Sub-Gaussian                   & $O(\varpi(T))$
\\
\hline
RS-ETC-M (This work)  & DRM       & Sub-Gaussian        & $O(\varpi^{2q}(T))$ \footnote{Assuming that the optimal solution is a solitary arm.} \\
\hline
RS-UCB-M (This work)     & DRM         & Sub-Gaussian                          & $O(\varpi^{2r  \kappa}(T))$  \\
\bottomrule
\end{tabularx}
\end{minipage}
\end{table}

\newpage

\section{Proof of Lemma \ref{example utility} (Gini Deviation)}
\label{proof:lemma_example}
In this section, we prove Lemma~\ref{example utility}, which states that the utility-maximizing solution for Gini deviation is a mixture of arm CDFs.
For any values $p_1,p_2\in(0,1)$, consider two Bernoulli distributions ${\sf Bern}(p_1)$ and ${\sf Bern}(p_2)$ with CDFs $\F_1$ and $\F_2$, respectively.  It can be readily verified that for the distribution $\F = {\sf Bern}(p)$ for any $p\in[0,1]$, we have
\begin{align} 
    U_h(\F) &= \int_{0}^{\infty} h(1-\F(x)) \diff{x} \\
    \label{eq:h_1_p}
    & = h(p)\ ,
\end{align}
which implies that $U_h(\F_1) = h(p_1)$ and $U_h(\F_2) = h(p_2)$. Owing to concavity of the distortion function \(h\), the maximizer
\begin{align}
    p^\star \triangleq \argmax_{p \in [0,1]} h(p)
\end{align}
is unique. Furthermore, due to the function being non-monotone, \(p^\star\) cannot lie at the boundaries, i.e., at $0$ or at  $1$. Hence, \(p^\star \in (0,1)\). Let us choose the mean values of the arms such that $p_1 < p^\star$ and $p_2>p^\star$. With these choices, there exists $\lambda\in(0,1)$ such that
\begin{align}
\label{eq:p_max_def}
    p^\star = \lambda p_1 + (1-\lambda)p_2 \ .
\end{align}
Accordingly, define the mixture distribution
\begin{align}
    \F^\star \triangleq \lambda \F_1 + (1-\lambda)\F_2\ .
\end{align}
Subsequently, we have
\begin{align}
    U_h(\F^\star ) = h(p^\star) > \max\{h(p_1), h(p_2)\} = \max\{U_h(\F_1),U_h(\F_2)\}\ .
\end{align}
This indicates that there exists a mixture of $F_1$ and $F_2$ whose DR value dominates those of $F_1$ and $F_2$.


\section{Finiteness of $W$ for sub-Gaussian Random Variables}
\label{Appendix:W_finitess}
We characterize an upper bound on the parameter $W$ defined in Section~\ref{sec:setting} and show that it is finite. Recall the definition of $W$:
\begin{align}
\label{eq:W}
W \triangleq \max \limits_{\balpha \neq \bbeta \in \Delta^K} \frac{1}{\lVert \balpha - \bbeta \rVert_1}\Big\|\sum_i \alpha_i \F_i-\sum_j \beta_j \F_j\Big\|_{\rm W}\ .
\end{align}
In order to show the finiteness of $W$ for $1$-sub-Gaussian random variables, we leverage the Kantorvich-Rubinstein duality of the $1$-Wasserstein measure, which is stated below.


\begin{theorem}[\cite{villani2009optimal}]
\label{theorem:Wasserstein_dual}
Let $\mcL^1(\Omega)$ denote the space of probability measures supported on $\Omega$ with finite first moment. For any $\P_1, \P_2\in \mcL^1(\Omega)$, we have
  \begin{align}
    \norm{\mathbb{P}_1 - \mathbb{P}_2}_{\rm W} = \sup \limits_{\lVert  f \rVert_L \leq 1} \left\{\displaystyle\int_{\Omega} f\diff \P_1 -\displaystyle\int_{\Omega} f\diff\P_2\right\}\ ,
  \end{align}
where $\lVert f \rVert_{L} \leq 1  $ denotes the space of all $1-$ Lipschitz functions $f:\mathbb{R} \rightarrow \mathbb{R}$. 
\end{theorem}
Based on the characterization of the $1$-Wasserstein distance in Theorem~\ref{theorem:Wasserstein_dual}, we next provide an upper bound on $W$ for sub-Gaussian random variables.

\begin{theorem}[Upper bound on $W$] \label{theorem:W}
Let $\{\mathbb{F}_i : i\in[K]\}$ be probability measures on $\Omega\subseteq \R$ that are $1$-sub-Gaussian and define $W$ as in~\eqref{eq:W}.  We have
\begin{align}
    W\;\leq\;\sqrt{2\pi}\ .
\end{align}
\end{theorem}
\begin{proof}
Let $\balpha,\bbeta\in\Delta^{K-1}$ denote two distinct probability mass functions on $[K]$. We have
\begin{align}
    \label{eq:W_1}
   \Big\|\sum_{i\in[K]} \Big(\alpha(i) - \beta(i)\Big) \mathbb{F}_i \Big\|_{\rm W} \;&=\;\sup \limits_{\lVert  f \rVert_L \leq 1} \sum_{i\in[K]} \Big(\alpha(i) - \beta(i)\Big) \E_{\mathbb{F}_i}\Big[f(X)\Big] \\
   \label{eq:W_2}
   \hfill & \leq \;\sup \limits_{\lVert  f \rVert_L \leq 1} \bigg\lvert \sum_{i\in[K]} \Big(\alpha(i) - \beta(i)\Big) \Big(\E_{\mathbb{F}_i}\Big[f(X) - f(0)\Big]  + f(0)\Big) \bigg\rvert \\ 
   \label{eq:W_3}
    \hfill & =\; \sup \limits_{\lVert  f \rVert_L \leq 1} \bigg\lvert \sum_{i\in[K]} \Big(\alpha(i) - \beta(i)\Big) \Big(\E_{\mathbb{F}_i}\Big[f(X) - f(0)\Big] \Big) \bigg\rvert \\ 
    \label{eq:W_4}
   \hfill & \leq\; \sup \limits_{\lVert  f \rVert_L \leq 1}
    \sum_{i\in[K]} \bigg\lvert  \Big(\alpha(i) - \beta(i)\Big) \E_{\mathbb{F}_i}\Big[f(X)-f(0)\Big] \bigg\rvert\\
    \label{eq:W_5}
    \hfill & \leq\; \sum_{i\in[K]} \Big\lvert \alpha(i) - \beta(i) \Big\rvert \cdot\sup \limits_{\lVert  f \rVert_L \leq 1} \Big\lvert \E_{\mathbb{F}_i}[f(X)-f(0)] \Big\rvert \\
    \label{eq:W_7}
    \hfill & \leq\; \sum_{i\in[K]} \Big\lvert \alpha(i) - \beta(i)\Big\rvert\cdot     \E_{\mathbb{F}_i}\Big[|X|\Big] \ ,
\end{align}
where,
\begin{itemize}
    \item the equality in \eqref{eq:W_1}~follows from Theorem~\ref{theorem:Wasserstein_dual};
    \item the transition~\eqref{eq:W_1}-\eqref{eq:W_2} holds since we take the absolute value;
    \item the transition~\eqref{eq:W_2}-\eqref{eq:W_3} follows from the fact that $\sum_{i\in[K]}\alpha(i) = \sum_{i\in[K]}\beta(i) = 1$;
    \item the transition~\eqref{eq:W_3}-\eqref{eq:W_4} follows from triangle inequality;
    \item the transition~\eqref{eq:W_4}-\eqref{eq:W_5} follows from the fact that for any two functions $f_1$ and $f_2$, we have $\sup_x \{f_1(x) + f_2(x) \}\leq \sup_x f_1(x) + \sup_x f_2(x)$;
    \item and the transition~\eqref{eq:W_5}-\eqref{eq:W_7} follows from $1$-Lipschitzness of $f$.
\end{itemize}

 For sub-Gaussian variables, $\E[|X|]$ is bounded in terms of the sub-Gaussian parameter. Since all distributions are $1-$sub- Gaussian, we have  
\begin{align}
    E_{\mathbb{F}_i}[|X|] &= \int_{0}^{+\infty} \P(|x| > u)du \\
    &\leq \int_{0}^{+\infty} 2\exp{-t^2/2} dt \\
    &= \int_{-\infty}^{+\infty} \exp{-t^2/2} dt \\
    &= \sqrt{2\pi} \underbrace{\int_{-\infty}^{+\infty} \frac{1}{\sqrt{2\pi}} \exp{-t^2/2} dt}_{= 1} \\
    &=  \sqrt{2\pi} 
\end{align}
 which implies that
\begin{align}
\label{eq:W_8}
     \Big\| \sum_{i\in[K]} \Big(\alpha(i) - \beta(i)\Big) \mathbb{F}_i \Big\|_{\rm W}\;\stackrel{\eqref{eq:W_7}}{\leq}\;\sqrt{2\pi}\cdot\norm{\balpha-\bbeta}_1\ .
\end{align}
Hence, from~\eqref{eq:W_8}, we have that
\begin{align}
    W\;=\; \max \limits_{\balpha \neq \bbeta \in \Delta^K} \frac{1}{\lVert \balpha - \bbeta \rVert_1}\Big\|\sum_i \alpha_i \F_i-\sum_j \beta_j \F_j\Big\|_{\rm W}\;\leq\;\sqrt{2\pi}\ .
\end{align}

\end{proof}


    

\section{Auxiliary Lemmas}
\label{proof:Aux_Lemma}

\subsection{Useful Properties}
In this section, we present some auxiliary lemmas that will be used in the proofs of Theorems \ref{theorem: ETC upper bound}, \ref{theorem:RS-ETC-M}, \ref{theorem:UCB upper bound}, and \ref{corollary:RS-UCB-M}. For the proofs, we use equivalent characterizations of the $1-$Wasserstein metric, which is provided in the lemma below. For the proof, the reader is referred to Lemma 2 \citep{prashanth2022wasserstein}.
\begin{lemma}
	\label{lemma:lipschitz-wasserstein}
	Consider random variables $X$ and $Y$ with CDFs $F_{X}$ and $F_{Y}$, respectively. Then,
	\begin{equation}
		\norm{F_X-F_Y}_{\rm{W}}=\sup_{\|f\|_L\leq 1} \Big|\E(f(X) - \E(f(Y))\Big|= \int_{-\infty}^{\infty}|F_{X}(s)-F_{Y}(s)|\mathrm{d}s=\int_{0}^{1}|F_{X}^{-1}(\beta)-F_{Y}^{-1}(\beta)|\mathrm{d}\beta \ . \label{lipwass2} 
	\end{equation}
\end{lemma}

\begin{lemma}[Concave Distortion Functions]
\label{lemma:risk-concave}
    For the DRs of the form stated in~\eqref{eq:wang_def}, if the distortion function $h : [0,1]\mapsto[0,1]$ is concave, then the DR evaluated for a mixture distribution over the arms is concave in the mixing coefficient.
\end{lemma}
\begin{proof}
    For any $\balpha\in\Delta^{K-1}$, we have
    \begin{align}
        U_h\left(\sum\limits_{i\in[K]} \alpha(i)\F_i\right)\;&=\;\displaystyle\int_0^\infty h\left(1-\sum\limits_{i\in[K]} \alpha(i)\F_i(x)\right)\diff x\\
        &=\;\displaystyle\int_0^\infty h\left(\sum\limits_{i\in[K]} \alpha(i)\Big(1-\F_i(x)\Big)\right)\diff x\\
        \label{eq:concave 1}
        &\geq\;\displaystyle\int_0^\infty \sum\limits_{i\in[K]} \alpha(i) h\Big( 1-\F_i(x)\Big)\diff x\\
        \label{eq: concave 2}
        &=\; \sum\limits_{i\in[K]} \alpha(i)\displaystyle\int_0^\infty h\Big( 1-\F_i(x)\Big)\diff x\\
        &=\;\sum\limits_{i\in[K]}\alpha(i)U_h(\F_i)\ ,
    \end{align}
    where,
    \begin{itemize}
        \item the inequality in~\eqref{eq:concave 1} follows from the concavity of the distortion function $w$; 
        \item and the equality in~\eqref{eq: concave 2} follows from the Fubini-Tonelli's theorem.
    \end{itemize}
\end{proof}

\subsection{\holder Exponents and Constants}

\begin{lemma}[Gini Deviation \holder Constants]
\label{lemma:example_utility_Holder}
Consider $K$ Bernoulli distributions $\{{\sf Bern}(p(i)):i\in[K]\}$ with CDFs $\{\F_i:i\in[K]\}$. Consider the optimal mixture $\F^\star = \sum_{i=1}^{K}{\alpha^\star(i)\F_i}$, and for a given $\balpha\in\Delta^{K-1}$, consider the mixture $\G = \sum_{i=1}^{K}{\alpha(i)\F_i}$. 
For the Gini deviation DR, i.e., $h(u)=u(1-u)$,  we have the following properties.
\begin{enumerate}
\item The \holder continuity exponent is $q=1$.
\item The \holder mixture exponent is $r=1$ if $\max_{i\in[K]} p(i) < 0.5$ or $\min_{i\in[K]} p(i) > 0.5$, and otherwise it is $r=2$.
\item The \holder constant is $\mcL = \max\{\mcL_{\rm MH}, \mcL_{\rm H}\} = 1$.
\end{enumerate}
\end{lemma}
\begin{proof}
Consider the mixture distributions $\F$ and $\G$ which have Bernoulli distributions with parameters
\begin{align}
    p_\F \triangleq  \sum_{i=1}^{K}{\alpha_{\F}(i)p(i)}\ ,
\qquad \mbox{and} \qquad     p_\G = \sum_{i=1}^{K}{\alpha_\G(i)p(i)}\ ,
\end{align}
respectively.
For the Gini DR, it can be easily verified that 
\begin{align}
\label{eq:lemma_5_def_gd}
    U_h\left(\F\right)\;=\; p_\F(1-{p_\F})\ ,\qquad\text{and}\qquad U_h\left(\G\right)\;=\; p_\G(1-p_\G) \ .
\end{align}
Hence, we have
\begin{align}
\label{eq:LALL2_1}
    U_h\left(\F  \right) - U_h\left(\G  \right)\; &\stackrel{\eqref{eq:lemma_5_def_gd}}{=}\; p_\F(1-p_\F)-p_\G(1-p_\G) \\
    &=\; (p_\F-p_\G)(1-p_\F-p_\G) \\
    &\leq\; | p_\F-p_\G | | 1- p_\F -p_\G | \\
    \label{eq:lemma_5_p}
    &\leq\; | p_\F-p_\G | \\
    &=\; \norm{\F-\G}_{\rm W}\ ,
    \label{eq:last_lemma_2}
\end{align}
where, \eqref{eq:lemma_5_p} follows from the fact that \(p_\F \leq 1\) and \(p_\G \leq 1\) and hence, \(|1-p_\F-p_\G| \leq 1 \).
Hence, for the DR considered, the \holder  constant is \(\mcL_{\rm H} = 1\) and the exponent is \(q=1\). 

Let us denote the optimal mixture as \(\F^\star\) with parameter \(p^\star \triangleq  \sum_{i=1}^{K}{\alpha^\star(i)p(i)}\), based on which $U_h(\F^\star)=p^\star(1-p^\star)$. Without any constraints on $p^\star$, this term will be maximized at $p^\star=\frac{1}{2}$. However, $p^\star$ is constrained to be a mixture of $\{p(i):i\in[K]\}$. This means that $p^\star=\frac{1}{2}$ is not viable when all the $p_i$'s are either larger than $1/2$ or smaller than $1/2$. Hence, depending on the values of $\{p(i):i\in[K]\}$, we analyze two separate cases. 
\begin{enumerate}
    \item {\bf Case 1: $\min_{i\in[K]}p(i)>0.5$ or $\max_{i\in[K]}p(i) <0.5$:} In either of these cases, for the Gini deviation DR and a $K$-armed Bernoulli bandit instance, it is easy to see that the optimal solution is a solitary arm. Specifically, this arm is given by 
    \begin{align}
        a_{\min}  \;\in\; \argmin\limits_{i\in[K]}\; p(i) \qquad \mbox{if} \quad \min_{i\in[K]}p(i)>0.5\ ,
    \end{align}
    or 
    \begin{align}
        a_{\max} \;\in\; \argmax\limits_{i\in[K]}\; p(i) \qquad \mbox{if} \quad \max_{i\in[K]}p(i) <0.5\ .
    \end{align}
    We will characterize the \holder exponent $q$ for the case that the optimal arm is $a_{\max}$, and the analysis for $a_{\min}$ follows similarly. For $\F^\star$ and $\G$ we have
    \begin{align}
        U_h(\F^\star) - U_h(\G)
        &=\; a_{\max}(1-a_{\max}) - p_{\G}(1-p_{\G})\\
        \label{eq:Gini_case1}
        &\leq |a_{\max} - p_{\G}|\\
        &= \| \F^\star - \G\|_{\rm W}\ ,
    \end{align}
    where~\eqref{eq:Gini_case1} follows from the exact steps as the transitions~\eqref{eq:LALL2_1}-\eqref{eq:last_lemma_2}. This indicates that $q=1$ and $\mcL_{\rm H} =1$.
    \item {\bf Case 2: $\exists i\in[K]$ such that $p(i) \leq \frac{1}{2}$ and $\exists i \in[K]$ such that $p(i) \geq \frac{1}{2}$:} In this case, it can be readily verified that $p^\star = 1/2$ is a viable solution, in which case $U_h(\F^\star) = 1/4$. Accordingly, we have 
    \begin{align}
        \label{eq:LALL}
            U_h\left(\F^\star  \right) - U_h\left(\G  \right)&= \frac{1}{4}-p_{\G}(1-p_{\G}) \\
            &= \left(\frac{1}{2} - p_{\G} \right)^2 \\
            &=\norm{\F^\star-\G}_{\rm W}^2\ .
            \label{eq:last_lemm_1}
        \end{align}
    This shows that $\mcL_{\rm MH}=1$ and $r=2$ in this case.
\end{enumerate}
Hence, in summary $\mcL=1$, $q=1$, and $r=1$.
\end{proof}

\begin{lemma}[PHT Measure \holder Constants]
\label{lemma:PHT}
    Consider two distinct CDFs $\F$ and $\G$ supported on $[0,\tau]$. For the PHT measure, i.e., $h(u) =  u^s$ for some $s\in (0,1)$, we have the following properties. 
    \begin{enumerate}
\item The \holder continuity exponent is $q=s$.
\item The \holder mixture exponent is $r=s$.
\item The \holder constant is $\mcL=\max\{\mcL_{\rm MH}, \mcL_{\rm H}\}=1$.
\end{enumerate}
\end{lemma}
\begin{proof}
We have
\begin{align*}
U_h(\F)-U_h(\G)&=  \int_{0}^{\tau}(1-\F(x))^{s}-(1-\G(x))^{s}\mathrm{d}x \nonumber\\
&\leq   \int_{0}^{\tau}|\F(x)-\G(x)|^{s}\mathrm{d}x \\
&\leq   \left[\int_{0}^{\tau}|\F(x)-\G(x)|\mathrm{d}x\right]^{s} \\
& \leq \norm{\F-\G}_{\rm W}^{s},
\end{align*}
where we used Jensen's inequality for the penultimate inequality and Lemma \ref{lemma:lipschitz-wasserstein} for the final inequality. This indicates that $q=s$ and $\mcL_{\rm H}=1$. We remark that based on the definitions of \holder continuity and mixture \holder continuity provided in \eqref{eq:Holder} and \eqref{eq:Holder_opt}, respectively, any given \holder continuity exponent $r$ is also a mixture \holder continuity exponent, that is given $q$, $r=q$ is always a valid choice for $r$. Hence, we have $r=q=s$ and $\mcL= \mcL_{\rm H} = \mcL_{\rm MH}=1$.
\end{proof}

\begin{lemma}[Wang's Right-Tail Deviation \holder Constants]
\label{lemma:RTD}
Consider $K$ Bernoulli distributions $\{{\sf Bern}(p_i):i\in[K]\}$ with CDFs $\{\F_i:i\in[K]\}$. Consider the optimal mixture $\F^\star = \sum_{i=1}^{K}{\alpha^\star(i)\F_i}$, and for given $\balpha_\F, \balpha_\G\in\Delta^{K-1}$, consider the mixtures $\F = \sum_{i=1}^{K}{\alpha_\F(i)\F_i}$ and $\G = \sum_{i=1}^{K}{\alpha_\G (i)\F_i}$. 
For the Wang's right-tail deviation, i.e., $h(u)=\sqrt{u}-u$, we have the following properties.
\begin{enumerate}
\item The \holder continuity exponent is $q=1/2$.
\item The \holder mixture exponent is $r=1$ if $\max_{i\in[K]} p(i) > 0.25$ and $\min_{i\in[K]} p(i) < 0.25$, and otherwise it is $r=1/2$.
\item The \holder constant is $\mcL = \max\{\mcL_{\rm MH}, \mcL_{\rm H}\} =1$.
\end{enumerate}
\end{lemma}

\begin{proof}
Note that that the mixture distributions $\F$ and $\G$ have Bernoulli distributions with parameters $\sum_{i=1}^{K}{\alpha_\F(i)p(i)}$ and $\sum_{i=1}^{K}{\alpha_\G(i) p(i)}$, respectively.
Let us define
\begin{align}
    p_\F \triangleq  \sum_{i=1}^{K}{\alpha_\F(i)p(i)}\ ,
\qquad \mbox{and} \qquad     p_\G \triangleq \sum_{i=1}^{K}{\alpha_\G(i)p(i)}\ .
\end{align}

It can be easily verified that
\begin{align}
\label{eq:lemma_5_def}
    U_h\left(\F\right)\;=\; \sqrt{p_\F}-p_\F\ ,\qquad\text{and}\qquad U_h\left(\G\right)\;=\; \sqrt{p_\G}-p_\G \ .
\end{align}
 Note that for any \(a, b \in [0, +\infty)\) we have
\begin{align}
\label{eq:WRTD_lemma_0}
      \Big| \sqrt{a}-\sqrt{b} \Big| \leq \sqrt{| a-b |} \ .
\end{align}
Hence, we obtain
\begin{align}
\label{eq:LALL2}
    U_h\left(\F  \right) - U_h\left(\G \right)\; & \stackrel{\eqref{eq:lemma_5_def}}{=}\; \sqrt{p_\F}(1-\sqrt{p_\F})-\sqrt{p_\G}(1-\sqrt{p_\G}) \\
    &=\; (\sqrt{p_\F}-\sqrt{p_\G})(1-\sqrt{p_\F}-\sqrt{p_\G}) \\
    &\leq\; | \sqrt{p_\F}-\sqrt{p_\G} | | 1-\sqrt{p_\F}-\sqrt{p_\G} | \\
    \label{eq:lemma_RTD_p}
    &\leq\; | \sqrt{p_\F}-\sqrt{p_\G} | \\
    & \stackrel{\eqref{eq:WRTD_lemma_0}}{\leq}\; | p_\F-p_\G |^{1/2} \\
    &=\; \norm{\F-\G}_{\rm W}^{1/2}\ ,
    \label{eq:last_lemma_rtd}
\end{align}
where, \eqref{eq:lemma_RTD_p} follows from the fact that \(p_\F \leq 1\) and \(p_\G \leq 1\) and hence, \(|1-\sqrt{p_\F}-\sqrt{p_\G}| \leq 1 \).
Hence, we have \(\mcL_{\rm H} = 1\) and \(q=1/2\). Next, we characterize the mixture \holder exponent. For this purpose, we denote the parameter for the optimal mixture by \(p_\F^\star \triangleq \sum_{i=1}^{K}{\alpha^\star(i)p(i)}\).
\begin{enumerate}
    \item {\bf Case 1: $\max_{i\in[K]} p(i) > 0.25$ and $\min_{i\in[K]} p(i) < 0.25$:} In this case,  it can be readily verified that \(p_\F^\star = \frac{1}{4}\) and \(U_h(\F^\star) = \frac{1}{4}\). Note that
\begin{align}
\label{eq:LALL_1}
    U_h\left(\F^\star  \right) - U_h\left(\G  \right) &= \frac{1}{4}-\sqrt{p_\G} + p_\G \\
    &= \left(\frac{1}{2} - \sqrt{p_\G} \right)^2 \\
    &= \left(\sqrt{p_\F^\star} - \sqrt{p_\G} \right)^2 \\
    &\stackrel{\eqref{eq:WRTD_lemma_0}}{\leq}  | p_\F^\star - p_\G|  \\
    &= \norm{\F^\star-\G}_{\rm W}\ ,
    \label{eq:last_lemm_1_2}
\end{align}
Hence, from~\eqref{eq:LALL_1}, we can conclude that  \(\mcL_{\rm MH} = 1\) and  \(r=1\).

\item {\bf Case 2: $\max_{i\in[K]} p(i) < 0.25$ or $\min_{i\in[K]} p(i) > 0.25$:} In this case, the optimal policy is a solitary arm, in which case \(r = q\), i.e., \(r=1/2\) and \(\mcL_{\rm MH} = 1\).
\end{enumerate}
Hence, in summary  $\mcL=1$, $q=1$, and $r=1/2$ or $r=1$ as specified.

\end{proof}

\begin{lemma}[Mean-median Deviation \holder Constants]
\label{lemma:mean_median_Holder}
Consider $K$ Bernoulli distributions $\{{\sf Bern}(p_i):i\in[K]\}$ with CDFs $\{\F_i:i\in[K]\}$. Consider the optimal mixture $\F^\star = \sum_{i=1}^{K}{\alpha^\star(i)\F_i}$, and for given $\balpha_\F, \balpha_\G\in\Delta^{K-1}$, consider the mixtures $\F = \sum_{i=1}^{K}{\alpha_\F(i)\F_i}$ and $\G = \sum_{i=1}^{K}{\alpha_\G (i)\F_i}$. 
For the mean-median deviation, i.e., $h(u)=\min \{u, (1-u) \}$, we have the following properties.
\begin{enumerate}
\item The \holder continuity exponent is $q=1$.
\item The \holder mixture exponent is $r=1$.
\item The \holder constant is $\mcL = \max\{\mcL_{\rm MH}, \mcL_{\rm H} \}=1$.
\end{enumerate}
\end{lemma}

\begin{proof}
Let us define
\begin{align}
    p_\F \triangleq  \sum_{i=1}^{K}{\alpha_\F(i)p(i)}\ ,
\qquad \mbox{and} \qquad     p_\G = \sum_{i=1}^{K}{\alpha_\G(i)p(i)}\ .
\end{align}

It can be easily verified that
\begin{align}
\label{eq:lemma_5_def_mean}
    U_h\left(\F\right)\;=\; \min \{p_\F, (1-p_\F) \}\ ,\qquad\text{and}\qquad U_h\left(\G\right)\;=\;  \min \{p_\G, (1-p_\G) \}\ .
\end{align}
\begin{enumerate}
    \item {\bf Case 1:} When \(p_\F > 1/2, \; p_\G> 1/2\) or \(p_\F < 1/2, \; p_\G< 1/2\) we have
\begin{align}
    U_h\left(\F  \right) - U_h\left(\G  \right)\; &\stackrel{\eqref{eq:lemma_5_def_mean}}{=}\; \min \{p_\F, (1-p_\F) \}\ - \min \{p_\G, (1-p_\G) \}\ \\ \
    \label{eq:lemma_5_p_mean_1_MD}
    &\leq\; | p_\F-p_\G | \ .
\end{align}
\item {\bf Case 2:} When \(p_\F < 1/2, \; p_\G > 1/2\) we have
\begin{align}
    U_h\left(\F  \right) - U_h\left(\G  \right)\; &\stackrel{\eqref{eq:lemma_5_def_mean}}{=}\; \min \{p_\F, (1-p_\F) \}\ - \min \{p_\G, (1-p_\G) \}\ \\ \
    & = \; p_\F-1+p_\G \\
    \label{eq:lemma_5_p_mean_2_MD}
    &\leq\; | p_\F-p_\G | \ .
\end{align}
\item {\bf Case 3:} When \(p_\F > 1/2, \; p_\G < 1/2\) we have
\begin{align}
    U_h\left(\F  \right) - U_h\left(\G  \right)\; 
    &\stackrel{\eqref{eq:lemma_5_def_mean}}{=}\; \min \{p_\F, (1-p_\F) \}\ - \min \{p_\G, (1-p_\G) \}\ \\ \
    & = \; 1-p_\F-p_\G \\
    \label{eq:lemma_5_p_mean_3_MD}
    &\leq\; | p_\F-p_\G | \ .
\end{align}

\end{enumerate}
Hence, overall
\begin{align}
     U_h\left(\F  \right) - U_h\left(\G  \right)\leq | p_\F-p_\G | = \norm{\F-\G}_{\rm W} \ .
\end{align}
This indicates that \(\mcL_{\rm H} = 1\)  and \(q=1\). Next, we characterize the mixture \holder exponent $r$. For this purpose, let us denote the parameter for the optimal mixture by \(p_\F^\star \triangleq \sum_{i=1}^{K}{\alpha^\star(i)p(i)}\).
\begin{enumerate}
\item {\bf Case 1: $\max_{i\in[K]} p(i) > 0.5$ and $\min_{i\in[K]} p(i) < 0.5$:} 
 It can be readily verified that for \(p^\star = \frac{1}{2}\) and \(U_h(p^\star) = 1/2\). Note that
\begin{align}
\label{eq:LALL_mean_MD}
    U_h\left(\F^\star  \right) - U_h\left(\G  \right) &= \frac{1}{2}-\min \{p_\G, (1-p_\G) \} \\
    &= \left| \frac{1}{2} - p_\G \right| \\
    &= \norm{\F-\G}_{\rm W}\ ,
    \label{eq:last_lemm_1_mean}
\end{align}
\item {\bf Case 2: $\max_{i\in[K]} p(i) < 0.5$ or $\min_{i\in[K]} p(i) > 0.5$:} In this case, the optimal solution is a solitary arm, in which case \(r = q\), i.e., \(r=1\) and \(\mcL_{\rm MH} = 1\).
\end{enumerate}
Hence, in summary $\mcL=1$, $q=1$, and $r=1$.

\end{proof}

\begin{lemma}[Inter-ES \holder Constants]
\label{lemma:inter_ES_Holder}
Consider $K$ Bernoulli distributions $\{{\sf Bern}(p_i):i\in[K]\}$ with CDFs $\{\F_i:i\in[K]\}$. Consider the optimal mixture $\F^\star = \sum_{i=1}^{K}{\alpha^\star(i)\F_i}$, and for given $\balpha_\F, \balpha_\G\in\Delta^{K-1}$, consider the mixtures $\F = \sum_{i=1}^{K}{\alpha_\F(i)\F_i}$ and $\G = \sum_{i=1}^{K}{\alpha_\G (i)\F_i}$. 
For the Inter-ES range with \(\alpha=0.5\), denoted by {\rm IER}\(_{0.5}\), i.e., $$h(u)=\min \Big\{\frac{u}{1-\alpha}, 1 \Big\} + \min \Big\{\frac{\alpha-u}{1-\alpha}, 0 \Big\}\ ,$$ 
we have the following properties.
\begin{enumerate}
\item The \holder continuity exponent is $q=1$.
\item The \holder mixture exponent is $r=1$.
\item The \holder constant is $\mcL = \max\{\mcL_{\rm MH}, \mcL_{\rm H} \}=2$.
\end{enumerate}
\end{lemma}

\begin{proof}
The distortion function for inter-ES range for \(\alpha=1/2\), denoted by {\rm IER}\(_{0.5}\), is 
    \begin{align}
        h(u) &= \min \{2u, 1 \} + \min \{1-2u, 0 \} \\
        &= 2 \min \{u, 1-u \} \ .
    \end{align}

Let us define
\begin{align}
    p_\F \triangleq  \sum_{i=1}^{K}{\alpha_\F(i)p(i)}\ ,
\qquad \mbox{and} \qquad     p_\G = \sum_{i=1}^{K}{\alpha_\G(i)p(i)}\ .
\end{align}
It can be easily verified that
\begin{align}
\label{eq:lemma_5_def_ES}
    U_h\left(\F\right)\;=\; 2\min \{p_\F, (1-p_\F) \}\ ,\qquad\text{and}\qquad U_h\left(\G\right)\;=\;  2\min \{p_\G, (1-p_\G) \}\ .
\end{align}

\begin{enumerate}
    \item {\bf Case 1:} For \(p_\F > 1/2, \; p_\G> 1/2\) or \(p_\F < 1/2, \; p_\G< 1/2\) we have
\begin{align}
    U_h\left(\F  \right) - U_h\left(\G  \right)\; &\stackrel{\eqref{eq:lemma_5_def_mean}}{=}\; 2\min \{p_\F, (1-p_\F) \}\ - 2\min \{p_\G, (1-p_\G) \}\ \\ \
    \label{eq:lemma_5_p_mean_1}
    &\leq\; 2| p_\F-p_\G | \ .
\end{align}
\item {\bf Case 2:} For \(p_\F < 1/2, \; p_\G > 1/2\),
\begin{align}
    U_h\left(\F  \right) - U_h\left(\G  \right)\; &\stackrel{\eqref{eq:lemma_5_def_mean}}{=}\; 2\min \{p_\F, (1-p_\F) \}\ - 2\min \{p_\G, (1-p_\G) \}\ \\ \
    & = \;2( p_\F-1+p_\G) \\
    \label{eq:lemma_5_p_mean_2}
    &\leq\; 2| p_\F-p_\G | \ .
\end{align}
\item {\bf Case 3:} For \(p_\F > 1/2, \; p_\G < 1/2\),
\begin{align}
    U_h\left(\F  \right) - U_h\left(\G  \right)\; 
    &\stackrel{\eqref{eq:lemma_5_def_mean}}{=}\; 2\min \{p_\F, (1-p_\F) \}\ - 2\min \{p_\G, (1-p_\G) \}\ \\ \
    & = \; 2(1-p_\F-p_\G) \\
    \label{eq:lemma_5_p_mean_3}
    &\leq\; 2| p_\F-p_\G | \ .
\end{align}

\end{enumerate}
Hence, overall, we have
\begin{align}
    | p_\F-p_\G | = 2\norm{\F-\G}_{\rm W} \ .
\end{align}
This indicates that \(\mcL_{\rm H} = 2\) and \(q=1\). Next, we characterize the mixture \holder exponent. For this purpose, let us denote the parameter for the optimal mixture by \(p_\F^\star \triangleq \sum_{i=1}^{K}{\alpha^\star(i)p(i)}\).
\begin{enumerate}
\item {\bf Case 1: $\max_{i\in[K]} p(i) > 0.5$ and $\min_{i\in[K]} p(i) < 0.5$:} 
 It can be readily verified that for \(p^\star = \frac{1}{2}\) and \(U_h(p^\star) = 1\). Note that
\begin{align}
\label{eq:LALL_mean}
    U_h\left(\F^\star  \right) - U_h\left(\G  \right) &= 2(\frac{1}{2}-\min \{p_\G, (1-p_\G) \}) \\
    &= 2\left| \frac{1}{2} - p_\G \right| \\
    &= 2\norm{\F-\G}_{\rm W}\ ,
    \label{eq:last_lemm_1_ES}
\end{align}
\item {\bf Case 2: $\max_{i\in[K]} p(i) < 0.5$ or $\min_{i\in[K]} p(i) > 0.5$:} In this case, the optimal solution is a solitary arm, in which case \(r = q\), i.e., \(r=1\) and \(\mcL_{\rm MH} = 2\).
\end{enumerate}

Hence, in summary $\mcL=2$, $q=1$, and $r=1$.

\end{proof}

\begin{lemma}[Dual Power \holder Constants]
\label{lemma:Dual_Power}
Consider $K$ Bernoulli distributions $\{{\sf Bern}(p_i):i\in[K]\}$ with CDFs $\{\F_i:i\in[K]\}$. Consider the optimal mixture $\F^\star = \sum_{i=1}^{K}{\alpha^\star(i)\F_i}$, and for given $\balpha_\F, \balpha_\G\in\Delta^{K-1}$, consider the mixtures $\F = \sum_{i=1}^{K}{\alpha_\F(i)\F_i}$ and $\G = \sum_{i=1}^{K}{\alpha_\G (i)\F_i}$. 
For dual power with the parameter $s \geq 2$, i.e., $h(u)= 1-(1-u)^s$, we have the following properties.
\begin{enumerate}
\item The \holder continuity exponent is $q=1$.
\item The \holder mixture exponent is $r=1$.
\item The \holder constant is $\mcL = \max\{\mcL_{\rm MH}, \mcL_{\rm H}\} = s$.
\end{enumerate}
\end{lemma}

\begin{proof}
    Let us define
\begin{align}
    p_\F \triangleq  \sum_{i=1}^{K}{\alpha_\F(i)p(i)}\ ,
\qquad \mbox{and} \qquad     p_\G = \sum_{i=1}^{K}{\alpha_\G(i)p(i)}\ .
\end{align}
It can be easily verified that
\begin{align}
\label{eq:lemma_5_def_DP}
    U_h\left(\F\right)\;=\; 1-(1-p_\F)^s
    ,\qquad\text{and}\qquad U_h\left(\G\right)\;=\;  1-(1-p_\G)^s \ .
\end{align}

Let us define $p_{\rm M} \in (\min \{p_\F, p_\G \}, \max \{p_\F, p_\G \} )$. We have
\begin{align}
    U_h\left(\F\right) - U_h\left(\G\right) & \stackrel{\eqref{eq:lemma_5_def_DP}}{=} 1-(1-p_\F)^s - (1-(1-p_\G)^s) \\ 
    &= (1-p_\G)^s - (1-p_\F)^s \\
    \label{eq:dp_beforelast}
    &= (p_\F - p_\G)\cdot s \cdot (1-p_{\rm M})^{s-1} \\
    &\leq |p_\F - p_\G|\cdot s \cdot (1-p_{\rm M})^{s-1} \\\label{eq:dp_last}&\leq |p_\F - p_\G| \cdot s \\
    \label{eq:dp_final}
    &= s \norm{\F - \G}_{\rm W}
\end{align}
where 
\begin{itemize}
    \item \eqref{eq:dp_beforelast} follows from observing that the distortion function is continuous and then applying the mean-value theorem, 
    \item \eqref{eq:dp_last} follows from the facts that $s\geq 2$ and $p_{\rm M} \leq 1$.
\end{itemize}

\eqref{eq:dp_final} indicates that $q=1$ and $\mathcal{L}_{\rm H}=s$. With a similar argument to Lemma \ref{lemma:PHT}, $r=q$, and $\mathcal{L}_{\rm MH} = \mathcal{L}_{\rm H}$ which means $\mathcal{L} = s$.
\end{proof}

\begin{lemma}[Quadratic \holder Constants]
\label{lemma:Quadratic}
Consider $K$ Bernoulli distributions $\{{\sf Bern}(p_i):i\in[K]\}$ with CDFs $\{\F_i:i\in[K]\}$. Consider the optimal mixture $\F^\star = \sum_{i=1}^{K}{\alpha^\star(i)\F_i}$, and for given $\balpha_\F, \balpha_\G\in\Delta^{K-1}$, consider the mixtures $\F = \sum_{i=1}^{K}{\alpha_\F(i)\F_i}$ and $\G = \sum_{i=1}^{K}{\alpha_\G (i)\F_i}$. 
For quadratic with the parameter $s \in [0,1]$, i.e., $h(u)= (1+s)u-su^2$, we have the following properties.
\begin{enumerate}
\item The \holder continuity exponent is $q=1$.
\item The \holder mixture exponent is $r=1$.
\item The \holder constant is $\mcL = \max\{\mcL_{\rm MH}, \mcL_{\rm H}\} = 1+s$.
\end{enumerate}
\end{lemma}

\begin{proof}
    Let us define
\begin{align}
    p_\F \triangleq  \sum_{i=1}^{K}{\alpha_\F(i)p(i)}\ ,
\qquad \mbox{and} \qquad     p_\G = \sum_{i=1}^{K}{\alpha_\G(i)p(i)}\ .
\end{align}
It can be easily verified that
\begin{align}
\label{eq:lemma_5_def_quadratic}
    U_h\left(\F\right)\;=\; (1+s)p_\F-sp_\F^2 
    ,\qquad\text{and}\qquad U_h\left(\G\right)\;=\;  (1+s)p_\G-sp_\G^2 \ .
\end{align}
Then, we have
\begin{align}
    U_h\left(\F\right) - U_h\left(\G\right) & \stackrel{\eqref{eq:lemma_5_def_quadratic}}{=}(1+s)p_\F-sp_\F^2 - ((1+s)p_\G-sp_\G^2)  \\
    &= (1+s)(p_\F - p_\G) -s(p_\F - p_\G)(p_\F + p_\G) \\
    &= (p_\F - p_\G) ((1+s) - s(p_\F + p_\G)) \\
    \label{eq:quadratic_lastline_2}
    &\leq (1+s) |p_\F - p_\G| \\
    \label{eq:quadratic_lastline}
    &= (1+s) \norm{\F - \G}_{\rm W}\ .
\end{align}
\eqref{eq:quadratic_lastline} implies that $q=1$ and $\mathcal{L}_{\rm H} = (1+s)$. Similar to Lemma~\ref{lemma:PHT}, here we choose $r=q$ and $\mathcal{L}_{\rm MH} = \mathcal{L}_{\rm H}$, and hence, $\mathcal{L} = (1+s)$.

\end{proof}

\subsection{Algorithms' Properties}


\begin{lemma}[Discretization Error]
\label{lemma:Delta_error}
The discretization error \(\Delta({\varepsilon})\) is upper bounded as
\begin{align}
    \Delta(\varepsilon) \leq \mcL (KW)^r \left(\frac{\varepsilon}{2}\right)^r \ .
\end{align}

\end{lemma}

\begin{proof}
Let us define $\bar\ba$ as the discrete mixing coefficient that has the least $L_1$ distance to the optimal coefficient $\balpha^\star$, i.e.,
\begin{align}
    \bar\ba\;\in\; \argmin\limits_{\ba\in\Delta_{\varepsilon}^{K-1}} \norm{\balpha^\star - \ba}_1\ . 
\end{align}
Accordingly, we have
\begin{align}
    \Delta(\varepsilon)& = \E_{\bnu}^{\pi} \left[ U_h\left(\sum_{i\in[K]} \alpha^\star (i)\F_i\right) - U_h\left(\sum_{i\in[K]} a^\star(i)\F_i\right) \right] \\
    & \leq  U_h\left(\sum_{i\in[K]} \alpha^\star (i)\F_i\right) - U_h\left(\sum_{i\in[K]} \bar a(i)\F_i\right) \\
    \label{eq:delta_error_43}
    & \leq \mcL\|\sum_{i\in[K]}  \Big(\alpha^\star(i)- \bar a(i)\Big)  \F_i\|_{\rm W}^r \\
    \label{eq:delta_error_5}
    & \leq \mcL  \norm{\balpha^\star- \bar\ba}_1^r W^r \\
    & \leq \mcL K^r \left(\frac{\varepsilon}{2}\right)^r  W^r  ,
    \label{eq:Delta_error_2}
\end{align}
where,
\begin{itemize}
    \item~\eqref{eq:delta_error_43} follows from Definition~\eqref{assumption:Holder_opt};
    \item~\eqref{eq:delta_error_5} follows from the definition of W in~\eqref{eq:W};
    \item and~\eqref{eq:Delta_error_2} follows from the fact that $\bar\balpha$ may lie at most $\varepsilon/2$ away from the optimal coefficient $\balpha^\star$ along each coordinate.
\end{itemize}
\end{proof}

\begin{lemma}
\label{lemma:beta_lemma_monotone}
    Consider a K-arm Bernoulli bandit instance with mean values \(\bp = \left(p(1) \cdots p(K)\right)\). For a DR with concave and strictly monotone distortion function, we have $\beta  = 1$.
\end{lemma}
\begin{proof}
Let \(\ba^\prime\) denote the second best discrete optimal solution, i.e,
\begin{align}
    \ba^\prime \triangleq \argmax_{\ba \in \Delta_\varepsilon^{K-1}: \ba \neq \ba^\star } U_h\left(\sum_{i \in [K]} a(i) \F_i\right)\ .
\end{align}
For a Bernoulli bandit instance, we have the following simplification. 
\begin{align}
    U_h \left(\sum_{i \in [K]}a(i)\F_i \right) = h\left(\sum_{i \in [K]}a(i) p_i\right) \ . 
\end{align}
We order the mean values from smallest to largest, i.e., \(p(1) < p(2) \cdots < p(K)\). Since the distortion function is strictly increasing, the optimal solution is a solitary arm, i.e,
\begin{align}
    \max_{\alpha \in \Delta^{K-1}} h\left(\sum_{i \in [K]} \alpha(i)p_i\right ) = h(p_K) \ .
\end{align} 
Equivalently, we have
\begin{align}
    \alpha^\star(i) =
    \begin{cases}
        1, \quad \text{if} \; i = K \\
        0, \quad \text{if} \; i \neq K
    \end{cases} \ .
\end{align}
Furthermore, $\forall \ba \in \Delta^{K-1}_{\varepsilon}$ we have
\begin{align}
\label{eq:solitary_best}
    h\left(\sum_{i \in [K]} a(i)p_i\right) < h(p_K) \ .
\end{align}
It can be readily verified that the best and the second best discrete mixing coefficients $\ba^\star$ and $\ba^\prime$ are obtained as follows.
\begin{align}
\label{eq:astar_lemma14}
    a^\star(i) =
    \begin{cases}
        1-(K-1)\varepsilon/2, \quad \text{if} \; i = K \\
        \varepsilon/2, \quad \text{if} \; i \neq K
    \end{cases} \ ,
\end{align}
and,
\begin{align}
\label{eq:aprime_lemma14}
    a^\prime(i) =
    \begin{cases}
        1-(K+1)\varepsilon/2, \quad \text{if} \; i = K \\
        3\varepsilon/2, \quad \text{if} \; i = K-1 \\
        \varepsilon/2, \quad \text{if} \; i \neq K-1, K \\
    \end{cases} \ .
\end{align}
Note that~\eqref{eq:astar_lemma14} and~\eqref{eq:aprime_lemma14} are obtained by first choosing the convex combination (in terms of the probabilities $\{p(i):i\in[K]\}$) with the maximal value, and subsequently leveraging the monotonicity of the distortion function $h$. Next, we have
\begin{align}
   \Delta_{\min}(\varepsilon) &= U_h \left(\sum_{i \in [K]}a^\star(i)\F_i \right) - U_h \left(\sum_{i \in [K]}a^\prime(i)\F_i \right)
   \\ &= h\left(\sum_{i \in [K]}a^\star(i)p(i)\right) - h\left(\sum_{i \in [K]}a^\prime(i)p(i)\right)\ \\
   \label{eq:concavity_h}
   &\geq h^\prime\left(\sum_{i \in [K]} a^\star(i)p(i) \right) \left(\sum_{i \in [K]} (a^\star(i) -a^\prime(i))p(i) \right) \\
   \label{eq:best_second_deff}
   &= h^\prime\left(\sum_{i \in [K]}a^\star(i)p(i) \right) \left(\varepsilon (p(K)-p(K-1))\right)  \\
   \label{eq:delta_min_conc_mon_analysis}
   &\geq h^\prime(p(K))  (p(K)-p(K-1))\cdot \varepsilon\ ,
\end{align}
where 
\begin{itemize}
    \item \eqref{eq:concavity_h} follows from concavity of $h$;
    \item \eqref{eq:best_second_deff} follows from the definitions of the best and second best discrete mixing coefficients;
    \item and \eqref{eq:delta_min_conc_mon_analysis} follows from the fact that leveraging concavity, we have $h^\prime(a) \geq h^\prime(b)$ for $a \leq b$, together with~\eqref{eq:solitary_best}.
\end{itemize}
For a strictly increasing function, the derivative $\forall p \in (0,1)$, $h^\prime(p) \neq 0$. From \eqref{eq:delta_min_conc_mon_analysis}, we can conclude that for DRs with strictly increasing and concave distortion functions \(\Delta_{\min}(\varepsilon) = \Omega(\varepsilon)\), and hence \(\beta = 1\).
\end{proof}

\section{Risk-sensitive Explore Then Commit for Mixtures (RS-ETC-M) Algorithm}
\label{Appendix:RS-ETC-M}
In this section, we provide 
the proofs of Theorems \ref{theorem: ETC upper bound} and \ref{theorem:RS-ETC-M}, which characterize the upper bound on the algorithm's regret.

\subsection{Regret Decomposition}
\label{proof:ETC upper bound}
Throughout this subsection, we prove Theorem~\ref{theorem: ETC upper bound}. We use the decomposition~\eqref{eq:regret_decomposition}  to provide a regret bound on the discrete regret \(\bar{\mathfrak{R}}_{\bnu}^{\rm E}(T)\). Let us define the {\em set} of discrete optimal mixtures of the DR $U_h$ as
\begin{align}
    {\rm OPT}_{\varepsilon}\;\triangleq\; \argmax\limits_{\ba\in\Delta_{\varepsilon}^{K-1}}\; U_h\left (\sum\limits_{i\in[K]} a(i)\F_i\right)\ ,
\end{align}
and the {\em set} of optimistic mixtures computed from the estimated CDFs at time instant $t$ as
\begin{align}
    \widehat{\rm OPT}_{\varepsilon,t}\;\triangleq\; \argmax\limits_{\ba \in\Delta_{\varepsilon}^{K-1}} U_h\left ( \sum\limits_{i\in[K]}a(i)\F_{i,t}^{\rm E}\right )\ .
\end{align}
In the regret decomposition we have provided in~\eqref{eq:regret_decomposition}, the discretization error \(\Delta(\varepsilon)\) is upper bounded in Lemma~\ref{lemma:Delta_error}. Hence, for regret analysis, we focus on the discrete regret. 
We can decompose the discrete regret into three main parts as follows.
\begin{align}
\label{eq: Regret_DISCRET_ETC}
    \bar{\mathfrak{R}}_{\bnu}^{\rm E}(T) &= \underbrace{\E_{\bnu}^{\rm E} \left[ \left( U_h\left(\sum_{i\in[K]} a^\star(i)\F_i\right) - U_h\left(\sum_{i\in[K]} a^{\rm E}_{N(\varepsilon)}(i)\F_i\right)\right)\mathds{1}\{\ba^{\rm E}_{N(\varepsilon)} \in{\rm OPT}_{\varepsilon} \}\right]}_{\triangleq A_1(T)}\nonumber \\
    &\qquad + \underbrace{\E_{\bnu}^{\rm E} \left[  \left(U_h\left(\sum_{i\in[K]} a^\star(i)\F_i\right) - U_h\left(\sum_{i\in[K]} a^{\rm E}_{N(\varepsilon)}(i)\F_i\right)\right)\mathds{1}\{\ba^{\rm E}_{N(\varepsilon)} \notin{\rm OPT}_{\varepsilon}\}\right]}_{\triangleq A_2(T)}\nonumber \\
    &\quad\qquad + \underbrace{\E_{\bnu}^{\rm E} \left[ U_h\left(\sum_{i\in[K]} a^{\rm E}_{N(\varepsilon)}(i)\F_i\right)  - U_h\left(\sum_{i\in[K]} \frac{\tau_t^{\rm E}(i)}{T}\F_i\right) \right]}_{\triangleq A_3(T)}.
\end{align}
It can be readily verified that \(A_1(T)=0\). The term $A_2(T)$ captures the {\em mixing coefficient estimation error} when the RS-ETC-M algorithm generates an incorrect mixing coefficient at the end of its exploration phase. Finally, the term $A_3(T)$ captures the {\em sampling estimation error}, i.e., the error in matching the arm selection fractions to the estimated mixing coefficient at the end of the exploration phase. Next, we provide an upper bound for \(A_2(T)\) and \(A_3(T)\).

\subsection{Upper Bound on the RS-ETC-M Mixing Coefficient Estimation Error $A_2(T)$}
\label{appendix:RS-ETC-M mixing error}

Expanding the mixing coefficient estimation error term $A_2(T)$, we obtain
\begin{align}
    A_2(T) &= {\E_{\bnu}^{\rm E} \left[ \left( U_h\left(\sum_{i\in[K]} a^\star(i)\F_i\right) - U_h\left(\sum_{i\in[K]}a^{\rm E}_{N(\varepsilon)}(i)\F_i\right) \right)\mathds{1}\{\ba^{\rm E}_{N(\varepsilon)} \notin{\rm OPT}_{\varepsilon}\}\right]} \\
    &= \E_{\bnu}^{\rm E} \left[ U_h\left(\sum_{i\in[K]} a^\star(i)\F_i\right) - U_h\left(\sum_{i\in[K]} a^{\rm E}_{N(\varepsilon)}(i)\F_i\right)\;\bigg\lvert\;\ba^{\rm E}_{N(\varepsilon)} \notin{\rm OPT}_{\varepsilon}\right] \nonumber \\ & \qquad \qquad \times \P_{\bnu}^{\rm E}\left(\ba^{\rm E}_{N(\varepsilon)} \notin{\rm OPT}_{\varepsilon}\right)\ .
\end{align}
Furthermore, owing to the fact that the DR is bounded above by $B$, we have 
\begin{align}
\label{eq: A2first}
    A_2(T) \leq B \cdot \P_{\bnu}^{\rm E}\left(\ba^{\rm E}_{N(\varepsilon)} \notin{\rm OPT}_{\varepsilon}\right)\ .
\end{align}
Next, we bound the probability of forming an incorrect estimate of the mixing coefficients. For this, for any $t$, we define the following events.
\begin{align}
    \mcE_{1,t}(x)\;&\triangleq\; \Bigg\{ \bigg\lvert U_h\bigg(\sum_{i\in[K]} a^\star(i)\F_{i,t}^{\rm E}\bigg) - U_h\bigg(\sum_{i\in[K]} a^\star (i)\F_i\bigg)\bigg \rvert\leq x\Bigg\}\ ,\\
    \mcE_{2,t}(x)\;&\triangleq\; \Bigg\{ \bigg\lvert U_h\bigg(\sum_{i\in[K]} a^{\rm E}_{N(\varepsilon)}(i)\F_{i,t}^{\rm E}\bigg) - U_h\bigg(\sum_{i\in[K]} a^{\rm E}_{N(\varepsilon)}(i)\F_i\bigg)\bigg \rvert\leq x\Bigg\}\ ,\\
    \text{and,}\quad\mcE_t(x)\;&\triangleq\;\mcE_1(x,t)\bigcap\mcE_2(x,t)\ .
\end{align}

Note that for any $\ba^\star\in{\rm OPT}_{\varepsilon}$, we have
\begin{align}
    & \P_{\bnu}^{\rm E}\Big(\ba^{\rm E}_{N(\varepsilon)} \notin{\rm OPT}_{\varepsilon}\Big)
    \\ &\leq\;\P_{\bnu}^{\rm E}\left ( U_h\bigg(\sum_{i\in[K]} a^\star (i)\F_i\bigg) \geq U_h\bigg(\sum_{i\in[K]} a^{\rm E}_{N(\varepsilon)}(i)\F_i\bigg) + \Delta_{\min}(\varepsilon)\right)\\
    \label{eq:ETC_A21}
    &=\;\P_{\bnu}^{\rm E}\left ( U_h\bigg(\sum_{i\in[K]} a^\star(i)\F_i\bigg) \geq U_h\bigg(\sum_{i\in[K]} a^{\rm E}_{N(\varepsilon)}(i)\F_i\bigg) + \Delta_{\min}(\varepsilon)\;\bigg\lvert\;\mcE_{N(\varepsilon)}\bigg(\frac{1}{2}\Delta_{\min}(\varepsilon)\bigg)\right)\nonumber\\
    &\qquad\qquad\qquad\times\P_{\bnu}^{\rm E} \left(\mcE_{N(\varepsilon)}\bigg(\frac{1}{2}\Delta_{\min}(\varepsilon)\bigg)\right)\\
    &\;\;+\P_{\bnu}^{\rm E} \left ( U_h\bigg(\sum_{i\in[K]} a^\star(i)\F_i\bigg) \geq U_h\bigg(\sum_{i\in[K]} a^{\rm E}_{N(\varepsilon)}(i)\F_i\bigg) + \Delta_{\min}(\varepsilon)\;\bigg\lvert\; \overline{\mcE}_{N(\varepsilon)}\bigg(\frac{1}{2}\Delta_{\min}(\varepsilon)\bigg)\right)\nonumber\\
    &\qquad\qquad\qquad\times\P_{\bnu}^{\rm E} \bigg( \overline{\mcE}_{N(\varepsilon)}\bigg(\frac{1}{2}\Delta_{\min}(\varepsilon)\bigg)\bigg)\ .
    \label{eq:ETC_A22}
\end{align}
Note that the first term, i.e., ~\eqref{eq:ETC_A21} is upper bounded by
\begin{align}
    &\P_{\bnu}^{\rm E} \left ( U_h\bigg(\sum_{i\in[K]} a^\star(i)\F_i\bigg) \geq U_h\bigg(\sum_{i\in[K]} a^{\rm E}_{N(\varepsilon)}(i)\F_i\bigg) + \Delta_{\min}(\varepsilon)\;\bigg\lvert\; \mcE_{N(\varepsilon)}\bigg(\frac{1}{2}\Delta_{\min}(\varepsilon)\bigg)\right)\\
    \label{eq:ETC_1}
    &\leq \P_{\bnu}^{\rm E} \Bigg ( U_h\bigg(\sum_{i\in[K]} a^\star(i)\F_{i,N(\varepsilon)}^{\rm E}\bigg) +  \frac{1}{2}\Delta_{\min}(\varepsilon)\geq U_h\bigg(\sum_{i\in[K]} a^{\rm E}_{N(\varepsilon)}(i)\F_{i,N(\varepsilon)}^{\rm E}\bigg)\nonumber\\
    &\qquad\qquad + \frac{1}{2}\Delta_{\min}(\varepsilon)\;\bigg\lvert\;\mcE_{N(\varepsilon)}\bigg(\frac{1}{2}\Delta_{\min}(\varepsilon)\bigg)\Bigg)\\
    &=0\ ,
    \label{eq:ETC_2}
\end{align}
where
\begin{itemize}
    \item[(i)]~\eqref{eq:ETC_1} holds due to the conditioning on the event $\mcE_{N(\varepsilon)}(\Delta_{\min}(\varepsilon)/2)$, and,
    \item[(ii)]~\eqref{eq:ETC_2} holds since $\ba^{\rm E}_{N(\varepsilon)}\in\widehat{\rm OPT}_{\varepsilon, N}$.
\end{itemize}

The second term, i.e., the term in~\eqref{eq:ETC_A22} is upper bounded by
\begin{align}
    \P_{\bnu}^{\rm E} \bigg( \overline{\mcE}_{N(\varepsilon)}\Big(\frac{1}{2}\Delta_{\min}(\varepsilon)\Big )\bigg)\;\leq\;\P_{\bnu}^\pi\bigg( \overline{\mcE}_{1,N}\Big (\frac{1}{2}\Delta_{\min}(\varepsilon)\Big)\bigg) + \P_{\bnu}^{\rm E} \bigg( \overline{\mcE}_{2,N}\Big (\frac{1}{2}\Delta_{\min}(\varepsilon)\Big)\bigg)\ .
\end{align}
Expanding each term, we have
\begin{align}
    \P_{\bnu}^{\rm E} \bigg( \overline{\mcE}_{1,N}\Big(\frac{1}{2}\Delta_{\min}(\varepsilon)\Big)\bigg)&= \P\left( \left\lvert U_h\bigg(\sum_{i\in[K]} a^\star(i)\F_{i, N(\varepsilon)}^{\rm E}\bigg) - U_h\bigg(\sum_{i\in[K]} a^\star(i)\F_i\bigg)\right\rvert > \frac{1}{2}\Delta_{\min}(\varepsilon)\right )\\
    &\stackrel{\eqref{eq:Holder}} {\leq} \P_{\bnu}^{\rm E} \left(\mcL \sum_{i\in[K]} \norm{a^\star (i)^\star(\F_{i,N(\varepsilon)}^{\rm E} - \F_i)}_{\rm W}^q > \frac{1}{2}\Delta_{\min}(\varepsilon) \right)\\
    & {\leq}  \sum_{i\in[K]}  \P_{\bnu}^{\rm E} \left(\mcL\underbrace{\big( a^\star(i)\big)^q}_{\leq 1}\norm{\F_{i,N(\varepsilon)}^{\rm E} - \F_i }_{\rm W}^q > \frac{1}{2K}\Delta_{\min}(\varepsilon) \right)\\
    \label{eq:alpha_to_1}
    &\leq \sum_{i\in[K]} \P_{\bnu}^{\rm E} \left ( \norm{\F_{i,N(\varepsilon)}^{\rm E}-\F_i}_{\rm W} > \left( \frac{1}{2K\mcL}\Delta_{\min}(\varepsilon)\right)^{\frac{1}{q}}  \right )\\
    &\stackrel{\eqref{eq:meta_concentration}}{\leq} \sum_{i\in[K]} 2\exp\left ( -\frac{\tau^{\rm E} _{N(\varepsilon)}(i)}{256e}\Bigg(\left(\frac{1}{2K\mcL( a^\star(i))^q}\Delta_{\min}(\varepsilon)\right)^{1/q} - \frac{512}{\sqrt{\tau^{\rm E} _{N(\varepsilon)}}}\Bigg)^2 \right) \\
     &{=} \sum_{i\in[K]} 2\exp\left ( -\frac{\frac{N(\varepsilon)}{K}}{256e}\Bigg(\left(\frac{1}{2K\mcL( a^\star(i))^q}\Delta_{\min}(\varepsilon)\right)^{1/q} - \frac{512}{\sqrt{\frac{N(\varepsilon)}{K}}}\Bigg)^2 \right) \\
    &= 2K \exp\left ( -\frac{\frac{N(\varepsilon)}{K}}{256e}\Bigg(\left(\frac{1}{2K\mcL( a^\star(i))^q}\Delta_{\min}(\varepsilon)\right)^{1/q} - \frac{512}{\sqrt{\frac{N(\varepsilon)}{K}}}\Bigg)^2 \right) \ .
    \label{eq:E1C}
\end{align}
Furthermore, we have
\begin{align}
&\P_{\bnu}^{\rm E} \bigg( \overline{\mcE}_{2,N(\varepsilon)}\Big(\frac{1}{2}\Delta_{\min}(\varepsilon)\Big)\bigg)\nonumber\\
&\quad= \P_{\bnu}^{\rm E} \left( \bigg\lvert U_h\bigg(\sum_{i\in[K]}a^{\rm E}_{N(\varepsilon)}(i)\F_{i,N(\varepsilon)}^{\rm E}\bigg) - U_h\bigg(\sum_{i\in[K]}a^{\rm E}_{N(\varepsilon)}(i)\F_i \bigg)\bigg\rvert > \frac{1}{2}\Delta_{\min}(\varepsilon)\right )\\
    &\quad\leq\P_{\bnu}^{\rm E} \left(\bigcup_{\ba^{\rm E}_{N(\varepsilon)}\in\widehat{\rm OPT}_{\varepsilon, N(\varepsilon)}}  \bigg\lvert U_h\bigg(\sum_{i\in[K]} a^{\rm E}_{N(\varepsilon)}(i)\F_{i,N(\varepsilon)}^{\rm E}\bigg) - U_h\bigg(\sum_{i\in[K]} a^{\rm E}_{N(\varepsilon)}(i)\F_i\bigg) \bigg \rvert > \frac{1}{2}\Delta_{\min}(\varepsilon)\right)\\
    &\quad\leq\sum_{\ba^{\rm E}_{N(\varepsilon)}\in\widehat{\rm OPT}_{\varepsilon, N(\varepsilon)}}\P_{\bnu}^{\rm E} \left(\bigg\lvert U_h\bigg(\sum_{i\in[K]} a^{\rm E}_{N(\varepsilon)}(i)\F_{i,N(\varepsilon)}^{\rm E}\bigg) - U_h\bigg(\sum_{i\in[K]} a^{\rm E}_{N(\varepsilon)}(i)\F_i \bigg) \bigg \rvert > \frac{1}{2}\Delta_{\min}(\varepsilon)\right)\\
    &\quad\stackrel{\eqref{eq:Holder}}{\leq} \sum_{\ba^{\rm E}_{N(\varepsilon)}\in\widehat{\rm OPT}_{\varepsilon, N(\varepsilon)}} \sum_{i\in[K]} \P_{\bnu}^{\rm E} \left ( \norm{\F_{i,N(\varepsilon)}^{\rm E}-\F_i}_{\rm W} > \left( \frac{1}{2K\mcL\underbrace{\big(a^{\rm E}_{N(\varepsilon)}(i)\big)^q}_{\leq 1}}\Delta_{\min}(\varepsilon)\right)^{\frac{1}{q}}  \right )\\
    &\quad {\leq} \sum_{\ba^{\rm E}_{N(\varepsilon)}\in\widehat{\rm OPT}_{\varepsilon, N(\varepsilon)}} \sum_{i\in[K]} \P_{\bnu}^{\rm E} \left ( \norm{\F_{i,N(\varepsilon)}^{\rm E}-\F_i}_{\rm W} > \left( \frac{1}{2K\mcL}\Delta_{\min}(\varepsilon)\right)^{\frac{1}{q}}  \right )\\
    &\quad\stackrel{\eqref{eq:meta_concentration}}{\leq}\sum_{\ba^{\rm E}_{N(\varepsilon)}\in\widehat{\rm OPT}_{\varepsilon, N(\varepsilon)}} \sum_{i\in[K]} \exp\left ( -\frac{\tau^{\rm E} _{N(\varepsilon)}(i)}{256e}\Bigg(\left(\frac{1}{2K\mcL( a^\star(i))^q}\Delta_{\min}(\varepsilon)\right)^{1/q} - \frac{512}{\sqrt{\tau^{\rm E} _{N(\varepsilon)}}}\Bigg)^2 \right)\\
    & \quad\leq \sum_{\ba^{\rm E}_{N(\varepsilon)}\in\widehat{\rm OPT}_{\varepsilon, N(\varepsilon)}} 2K \exp\left ( -\frac{\frac{N(\varepsilon)}{K}}{256e}\Bigg(\left(\frac{1}{2K\mcL( a^\star(i))^q}\Delta_{\min}(\varepsilon)\right)^{1/q} - \frac{512}{\sqrt{\frac{N(\varepsilon)}{K}}}\Bigg)^2 \right) \\
    &\quad\leq 2K\left(\frac{1}{\varepsilon} \right)^{K-1} \exp\left ( -\frac{\frac{N(\varepsilon)}{K}}{256e}\Bigg(\left(\frac{1}{2K\mcL( a^\star(i))^q}\Delta_{\min}(\varepsilon)\right)^{1/q} - \frac{512}{\sqrt{\frac{N(\varepsilon)}{K}}}\Bigg)^2 \right), 
    \label{eq:E2C}
\end{align}
where~\eqref{eq:E2C} follows from the fact that the total number of discrete $\ba$ values may not exceed \(\left(\frac{1}{\varepsilon} \right)^{K-1}\). Combining~\eqref{eq:E1C} and~\eqref{eq:E2C}, we obtain
\begin{align}
    &\P_{\bnu}^{\rm E} \bigg( \overline{\mcE}_{1,N(\varepsilon)}\Big(\frac{1}{2}\Delta_{\min}(\varepsilon)\Big)\bigg) + \P_{\bnu}^{\rm E} \bigg( \overline{\mcE}_{2,N(\varepsilon)}\Big(\frac{1}{2}\Delta_{\min}(\varepsilon)\Big)\bigg) \nonumber \\
    & \quad\leq 2K \exp\left ( -\frac{\frac{N(\varepsilon)}{K}}{256e}\Bigg(\left(\frac{1}{2K\mcL( a^\star(i))^q}\Delta_{\min}(\varepsilon)\right)^{1/q} - \frac{512}{\sqrt{\frac{N(\varepsilon)}{K}}}\Bigg)^2 \right) \nonumber \\
    & \qquad+ \left(\frac{1}{\varepsilon} \right)^{K-1} 2K \exp\left ( -\frac{\frac{N(\varepsilon)}{K}}{256e}\Bigg(\left(\frac{1}{2K\mcL( a^\star(i))^q}\Delta_{\min}(\varepsilon)\right)^{1/q} - \frac{512}{\sqrt{\frac{N(\varepsilon)}{K}}}\Bigg)^2 \right)  \\
    &\quad= 2K \exp\left ( -\frac{\frac{N(\varepsilon)}{K}}{256e}\Bigg(\left(\frac{1}{2K\mcL( a^\star(i))^q}\Delta_{\min}(\varepsilon)\right)^{1/q} - \frac{512}{\sqrt{\frac{N(\varepsilon)}{K}}}\Bigg)^2 \right) \cdot \left( \left(\frac{1}{\varepsilon} \right)^{K-1} +1 \right)\ .
    \label{eq:E3C}
\end{align}

At this step, for some $\delta\in(0,1)$, choosing \( N(\varepsilon) \) as
\begin{align}
\label{eq:number of samples}
    N(\varepsilon) \triangleq 256K  \e \left(\frac{2K \mcL}{\Delta_{\min}(\varepsilon)}\right)^{\frac{2}{q}} \bigg[\frac{32}{\sqrt{\rm e}} + \log^{\frac{1}{2}} \Big( 2K\delta   \big(\varepsilon^{-(K-1)} + 1\big) \Big) \bigg]^2\ , 
\end{align}
 and leveraging~\eqref{eq:E3C}, it can be readily verified that 
\begin{align}
    \P_{\bnu}^{\rm E} \bigg( \overline{\mcE}_{1,N(\varepsilon)}\Big(\frac{1}{2}\Delta_{\min}(\varepsilon)\Big)\bigg) + \P_{\bnu}^{\rm E} \bigg( \overline{\mcE}_{2,N(\varepsilon)}\Big(\frac{1}{2}\Delta_{\min}(\varepsilon)\Big)\bigg)\;\leq\;\delta\ ,
\end{align}
which implies that $\P_{\bnu}^{\rm E}\left(\ba^{\rm E}_{N(\varepsilon)}\notin{\rm OPT}_{\varepsilon}\right)  \leq \delta$. Hence, we have
\begin{align}
\label{eq: ETC A2T final}
    A_2(T) \stackrel{\eqref{eq: A2first}}{\leq} B \cdot \delta\ .
\end{align}

\subsection{Upper Bound on the RS-ETC-M Sampling Estimation Error $A_3(T)$}

Next, we turn to analyzing \(A_3(T)\), which captures the sampling estimation error. We have
\begin{align}
    A_3(T) &= \E_{\bnu}^{\rm E} \left[ U_h\left(\sum_{i\in[K]} a^{\rm E}_{N(\varepsilon)}(i)\F_i\right)  - U_h\left(\sum_{i\in[K]} \frac{\tau_T^{\rm E}(i)}{T}\F_i\right) \right] \\ 
    \label{eq:lemma_8_2}
&\leq  \mcL\E_{\bnu}^{\rm E} \left[ \norm{\ba^{\rm E}_{N(\varepsilon)} - \frac{\btau_T^{\rm E}}{T} }_1^q W^q \right] \\
&\leq \mcL K W^q  \E_{\bnu}^{\rm E} \bigg[ \max_{i \in [K]} \left|a^{\rm E}_{N(\varepsilon)}(i) - \frac{\tau_t^{\rm E}(i)}{T}\right|^q  \bigg] ,
\end{align}
where~\eqref{eq:lemma_8_2} follows from the \holder continuity of the utility stated in~\eqref{eq:Holder} and recalling the definition of $W$ stated in~\eqref{eq:W}. We require to upper bound the estimation error due to the policy's sampling proportions. Note that as a result of the ``commitment'' process of the ETC algorithm stated in~\eqref{eq: ETC_samp}, for the first $K-1$ arms, the RS-ETC-M commits its arms selection such that it is pulled \( \max\{0,\lfloor Ta_{N(\varepsilon)}^{\rm E}(i)\rfloor - N(\varepsilon)/K\}\) times post exploration, and it allocates the remaining rounds to arm $K$. Let us define the set
\begin{align}
\label{eq:ETC set}
    \mcS^\prime\;\triangleq\; \Big\{i\in[K] : a^{\rm E}_{N(\varepsilon)}(i)< N(\varepsilon)/KT\Big\}\ .
\end{align}
Accordingly, for any arm $i\notin\mcS^\prime$, if $i\neq K$, we have the following bound.
\begin{align}
    \bigg \lvert \frac{\tau_t^{\rm E}(i)}{T} - a^{\rm E}_{N(\varepsilon)}(i)\bigg\rvert\;&\stackrel{\eqref{eq:ETC set}}{=}\; a^{\rm E}_{N(\varepsilon)}(i) - \frac{\tau_t^{\rm E}(i)}{T}\\
    &\stackrel{\eqref{eq: ETC_samp}}{=}\; a^{\rm E}_{N(\varepsilon)}(i) - \frac{\lfloor Ta^{\rm E}_{N(\varepsilon)}(i)\rfloor}{T}\\
    &\leq\; a^{\rm E}_{N(\varepsilon)}(i) - \frac{Ta^{\rm E}_{N(\varepsilon)}(i) - 1}{T}\\
    &<\;\frac{1}{T}\ .
    \label{eq: ETC commit 3}
\end{align}
Alternatively, if $K\notin\mcS^\prime$, we have
\begin{align}
    \bigg \lvert \frac{\tau_T^{\rm E}(K)}{T} - a^{\rm E}_{N(\varepsilon)}(K)\bigg\rvert\;&\stackrel{\eqref{eq:ETC set}}{=}\; a^{\rm E}_{N(\varepsilon)}(K) - \frac{\tau_T^{\rm E}(K)}{T}\\
    &\stackrel{\eqref{eq: ETC_samp}}{=}\; a^{\rm E}_{N(\varepsilon)}(K) - \frac{T - \sum\limits_{i\neq K}\tau_t^{\rm E}(i)}{T}\\
    &=\; a^{\rm E}_{N(\varepsilon)}(K) - 1 - \frac{1}{T} \left (\sum\limits_{i\in\mcS^\prime:i\neq K} \tau_t^{\rm E}(i) + \sum\limits_{i\notin\mcS^\prime: i \neq K} \tau_t^{\rm E}(i)\right)\\
    \label{eq: ETC commit 1}
    &=\; a^{\rm E}_{N(\varepsilon)}(K) - 1 - \frac{1}{T}\left( \frac{N(\varepsilon)}{K}|\mcS^\prime| + \sum\limits_{i\notin\mcS^\prime:i\neq K} \lfloor Ta^{\rm E}_{N(\varepsilon)}(i)\rfloor\right)\\
    &\leq\;a^{\rm E}_{N(\varepsilon)}(K) - 1 - \frac{1}{T}\left( \frac{N(\varepsilon)}{K}|\mcS^\prime| + \sum\limits_{i\notin\mcS^\prime:i\neq K}  Ta^{\rm E}_{N(\varepsilon)}(i)\right)\\
    &=\;\sum\limits_{i\notin\mcS^\prime} a^{\rm E}_{N(\varepsilon)}(i) - 1 + \frac{N(\varepsilon)|\mcS^\prime|}{KT}\\
    &= \;\frac{N(\varepsilon)|\mcS^\prime|}{KT} - \sum\limits_{i\in\mcS^\prime}a^{\rm E}_{N(\varepsilon)}(i) \\
    \label{eq: ETC commit 2}
    &\leq\; \sum\limits_{i\in\mcS^\prime}\frac{N(\varepsilon)}{KT} + \frac{N(\varepsilon)|\mcS^\prime|}{KT}\\
    &= \frac{2N(\varepsilon)|\mcS^\prime|}{KT}\\
    &\leq\; \frac{2N(\varepsilon)}{T}\ ,
    \label{eq: ETC commit 4}
\end{align}
where,
\begin{itemize}
    \item~\eqref{eq: ETC commit 1} follows from the fact that for every $i\in\mcS^\prime$, since these arms have already been over-explored, they are not going to get sampled further by the RS-ETC-M algorithm;
    \item and~\eqref{eq: ETC commit 2} follows from the definition of the set $\mcS$, which dictates that for every arm $i\in\mcS^\prime$, we have \(a^{\rm E}_{N(\varepsilon)}(i)< N(\varepsilon)/KT\).
\end{itemize}
Finally, for every $i\in\mcS^\prime$, we have
\begin{align}
     \bigg \lvert \frac{\tau_T^{\rm E}(i)}{T} - a^{\rm E}_{N(\varepsilon)}(i)\bigg\rvert\;&\stackrel{\eqref{eq:ETC set},\eqref{eq: ETC_samp}}{=}\; \bigg\lvert \frac{N(\varepsilon)}{KT} - a^{\rm E}_{N(\varepsilon)}(i)\bigg\rvert\\
     &\stackrel{\eqref{eq:ETC set}}{=}\; \frac{N(\varepsilon)}{KT} - a^{\rm E}_{N(\varepsilon)}(i)\\
     &\;< \frac{N(\varepsilon)}{KT}\ .
     \label{eq: ETC commit 5}
\end{align}
Hence, combining~\eqref{eq: ETC commit 3},~\eqref{eq: ETC commit 4}, and~\eqref{eq: ETC commit 5}, we conclude that for any $i\in[K]$, we have
\begin{align}
     \bigg \lvert \frac{\tau_T^{\rm E}(i)}{T} - a^{\rm E}_{N(\varepsilon)}(i)\bigg\rvert\;\leq\; \frac{2N(\varepsilon)}{T}\ .
     \label{eq: ETC commit 6}
\end{align}

Leveraging~\eqref{eq: ETC commit 6}, we can now upper bounded \(A_3(T)\) as follows.
\begin{align}
    A_3(T) &=\; \E_{\bnu}^{\rm E} \left[ U_h\left(\sum_{i\in[K]} a^{\rm E}_{N(\varepsilon)}(i)\F_i\right)  - U_h\left(\sum_{i\in[K]} \frac{\tau_t^{\rm E}(i)}{T}\F_i\right) \right] \\ 
&\leq\; \mcL KW^q\E_{\bnu}^{\rm E} \left[ \max_{i \in [K]} \left|a^{\rm E}_{N(\varepsilon)}(i) - \frac{\tau_t^{\rm E}(i)}{T}\right|^q  \right]\\
&\stackrel{\eqref{eq: ETC commit 6}}{\leq}\; \mcL KW^q \left(\frac{2N(\varepsilon)}{T}\right)^q \ .
\label{eq: ETC A3T final}
\end{align}

\subsection{Proof of Theorem \ref{theorem: ETC upper bound}}
We can upper bound the discretized regret leveraging the upper bounds on $A_1(T)$, $A_2(T)$ in~\eqref{eq: ETC A2T final}, and $A_3(T)$ in~\eqref{eq: ETC A3T final} and considering \(\delta = 1/T^2 \). We obtain
\begin{align}
    \Bar{\mathfrak{R}}_{\bnu}^{\rm E}(T) &\stackrel{\eqref{eq: Regret_DISCRET_ETC}}{=}  A_1(T) + A_2(T) + A_3(T) \\
    & \leq  \frac{B}{T^2} + \mcL K \left(3W\frac{N(\varepsilon)}{T}\right)^q \\
    & \leq  (\mcL K + W^{-q}) \left(3W\frac{N(\varepsilon)}{T}\right)^q \ ,
    \label{eq:UP_ETC_eqn1}
\end{align}

where~\eqref{eq:UP_ETC_eqn1} follows from the fact that $B/T^2$ is upper bounded by $(N(\varepsilon)/T)^q$. Finally, combining all the terms $A_1(T)$, $A_2(T)$ in~\eqref{eq: ETC A2T final}, $A_3(T)$ in~\eqref{eq: ETC A3T final}, and the discretization error $\Delta(\varepsilon)$ in~\eqref{lemma:Delta_error}, we obtain
\begin{align}
\label{eq:disc_up_ETC}
    \mathfrak{R}_{\bnu}^{\rm E}(T) &\stackrel{\eqref{eq:regret_decomposition}}{=} \Delta(\varepsilon) + \Bar{\mathfrak{R}}_{\bnu}^{\rm E}(T) \\
    & \leq \Delta(\varepsilon) + (\mcL K +W^{-q}) \left(\frac{2N(\varepsilon)}{T}\right)^q  \\
    \label{eq:ETC_M_eps}
    & = \Delta(\varepsilon) +  (\mcL K +W^{-q}) \Bigg( 3WM(\varepsilon)\frac{\log T}{T} \Bigg)^q\ ,
\end{align}
where~\eqref{eq:ETC_M_eps} follows by defining \(M(\varepsilon) \triangleq N(\varepsilon)/\log T\).

\subsection{Proof of Theorem~\ref{theorem:RS-ETC-M}}
Let us set $\varepsilon = \Theta((K^{2+2/q}T^{-\gamma}\log T)^{q/2\beta})$ for \(\gamma \in (0,1)\). Assuming \(\Delta_{\min}(\varepsilon) = \Omega(\varepsilon^\beta)\ \), from~\eqref{eq:discrete_gap}, the exploration horizon is $N(\varepsilon)=\Theta(T^\gamma)$. Hence, for discrete regret, from~\eqref{eq:ETC_M_eps} we have
\begin{align}
\label{eq:Discrete_order}
    \Bar{\mathfrak{R}}_{\bnu}^{\rm E}(T) = O(KT^{(\gamma-1)q})\ .
\end{align}

Also, from Lemma~\eqref{lemma:Delta_error}, we have
\begin{align}
    \Delta(\varepsilon) &\leq \mcL (KW)^r \left(\frac{\varepsilon}{2}\right)^r  \ .
\end{align}
As a result, \(\Delta(\varepsilon) = O(\varepsilon^r)\). Hence, by our choice of $\varepsilon$, we obtain 
\begin{align}
\label{eq:delta_order}
    \Delta(\varepsilon) = O\left(K^{r + \frac{r(q+1)}{\beta}}T^{-\frac{rq\gamma}{2\beta}}(\log T)^{\frac{rq}{2\beta}}\right).
\end{align}

From~\eqref{eq:disc_up_ETC},~\eqref{eq:Discrete_order}, and~\eqref{eq:delta_order} the regret of the RS-ETC-M algorithm is upper bounded by
    \begin{align}
    \label{eq:ETC upper bound gamma appendix}
        \mathfrak{R}_{\bnu}^{\rm E}(T) = O \left (\max\Big\{ K^{r + \frac{r(q+1)}{\beta}}T^{-\frac{rq\gamma}{2\beta}}(\log T)^{\frac{rq}{2\beta}}\;,\;K^qT^{(\gamma-1)q}\Big\} \right)\ .
    \end{align}
Finally, setting $\gamma = \frac{2\beta}{2\beta + r}$, the regret of the RS-ETC-M algorithm is upper bounded by 
\begin{align}
     \mathfrak{R}_{\bnu}^{{\textnormal{\rm E}}}(T)\ & \leq O\Big(\Big[K^{2(q+\beta+1)}\; T^{-\frac{q}{1+r/2\beta}}\; (\log T)^{q}\Big]^{\frac{r}{2\beta}}\Big)\ .
\end{align}

\section{Risk-sensitive Upper-Confidence Bound (RS-UCB-M) Algorithm}
\label{proof:UCB upper bound}
In this section, we provide the proofs of Theorems \ref{theorem:UCB upper bound} and \ref{corollary:RS-UCB-M}, which characterize the upper bound on the algorithm's regret.

\subsection{Regret Decomposition}
In~\eqref{eq:regret_decomposition}, the regret is decomposed into a discretization error component, and the discrete regret. Lemma~\ref{lemma:Delta_error} shows an upper bound for the discretization error. In this section, we provide an upper bound on discrete regret $\bar{\mathfrak{R}}_{\bnu}^{\rm U}(T)$ of Algorithm~\ref{algorithm:B-UCB-M} as follows. Recall that we have defined \(\tau_t^{\rm U}(i)\) as the number of times RS-UCB-M selects arm \(i \in [K]\) up to time $t$. We have
\begin{align}
    \Bar{\mathfrak{R}}_{\bnu}^{\rm U}(T) &=   U_h\left(\sum_{i\in[K]} a^\star(i)\F_i\right)  - \E_{\bnu}^{\rm U} \left[U_h\left(\sum_{i\in[K]} \frac{\tau_T^{\rm U}(i)}{T}\F_i\right)\right] \\
    &\leq \underbrace{\sum_{\mcS \subseteq [K]: \mcS \neq \emptyset}
    \E_{\bnu}^{\rm U} \left[\Bigg( U_h\left(\sum_{i\in[K]} a^\star(i)\F_i\right)  -  U_h\left(\sum_{i\in[K]} \frac{\tau_T^{\rm U}(i)}{T}\F_i\right) \Bigg) \mathds{1}\{\F_i \notin \mcC_T(i): i \in \mcS\}\right]}_{\triangleq B_1(T)} \\
    &\qquad + \underbrace{\E_{\bnu}^{\rm U} \left[ \Bigg( U_h\left(\sum_{i\in[K]} a^\star(i)\F_i\right)  - U_h\left(\sum_{i\in[K]} \frac{\tau_T^{\rm U}(i)}{T}\F_i\right) \Bigg) \mathds{1}\{\F_i \in \mcC_T(i) \;\forall i \in [K]\} \right]}_{\triangleq B_2(T)} \ .
\end{align}

Expanding $B_1(T)$, we have
\begin{align}
    B_1(T)&=\sum_{\mcS \subseteq [K]: \mcS \neq \emptyset}\P_{\bnu}^{\rm U}\Big ( \F_i \notin \mcC_T(i): i \in \mcS\Big )\nonumber\\
    &\times \quad
    \E_{\bnu}^{\rm U} \left[U_h\left(\sum_{i\in[K]} a^\star(i)\F_i\right)  -  U_h\left(\sum_{i=1}^{K} \frac{\tau_T^{\rm U}(i)}{T}\F_i\right)\bigg\lvert\; \F_i \notin \mcC_T(i): i \in \mcS\right]\ .
\end{align}
Leveraging the fact that
\begin{align}
    \sum_{i=1}^{K} \binom{K}{i} x^i = (x+1)^K-1\ ,
\end{align}
along with the large deviation bound on the confidence sequences from Lemma~\ref{lemma: concentration}, i.e., $\P_{\bnu}^{\rm U}(\F_i\notin\mcC_T(i))\leq 1/T^2$ for every $i\in[K]$, we have
\begin{align}
\label{eq:probi}
    \sum\limits_{\mcS\subseteq[K] : \mcS\neq\emptyset}\P_{\bnu}^{\rm U} \Big(\F_i \notin \mcC_T(i):  i \in \mcS\Big) = \left(\frac{2}{T^2}+1\right)^K-1\ .
\end{align}
Furthermore, owing to the fact that $U_h(\sum_{i\in[K]}\alpha(i)\F_i)\leq B$ for any $\balpha\in\Delta^{K-1}$, we have the following upper bound on $B_1(T)$.
\begin{align}
\label{eq:A_1 bound refer}
    B_1(T)\;\leq\; B\left(\left(\frac{2}{T^2}+1\right)^K-1\right)\ .
\end{align}
Next, we will upper bound $B_2(T)$. For this purpose, we begin by defining the upper confidence bound of a parameter $\balpha\in\Delta^{K-1}$ at time $t\in\N$ as 
 \begin{align}
 \label{eq:UCB_alpha_def}
     {\rm UCB}_t(\balpha)\;\triangleq\; \max\limits_{\eta_1\in\mcC_t(1),\cdots,\mcC_t(K)}\; U\left ( \sum\limits_{i\in[K]}\alpha(i)\eta_i\right )\ .
\end{align}
Furthermore, let us define 
the optimistic CDF estimates which maximize the upper confidence bound for every arm $i\in[K]$ as 
\begin{align}  \widetilde\F_{i,t}\;\triangleq\;\argmax_{\eta_i\in\mcC_t(i)}\;\max\limits_{\eta_j\in\mcC_t(j) : j\neq i} {\rm UCB}_t(\ba^{\rm U}_t)\ .
\end{align}


Expanding $B_2(T)$, we have 
\begin{align}
    B_2(T) &= \left . \E_{\bnu}^{\rm U}\left[ U_h\left(\sum_{i\in[K]} \bar{\alpha}^\star(i) \F_i \right) - U_h\left(\sum_{i\in[K]} \frac{\tau_T^{\rm U}(i)}{T} \F_i \right) \right\lvert \F_i \in \mcC_T(i) \;\forall i \in [K] \right]\\
    & \stackrel{\eqref{eq:UCB_alpha_def}}{\leq} \E_{\bnu}^{\rm U}\left[\left .{\rm UCB}_T(\ba^\star) - U_h\left(\sum_{i\in[K]} \frac{\tau_T^{\rm U}(i)}{T}\F_i\right)\right\lvert \F_i \in \mcC_T(i) \;\forall i \in [K]\right] \\
    & \stackrel{\eqref{eq:UCB_alpha}}{\leq} \E_{\bnu}^{\rm U}\left[\left . {\rm UCB}_T(\ba^{\rm U}_T) - U_h\left(\sum_{i\in[K]} \frac{\tau_T^{\rm U}(i)}{T}\F_i\right)\right\lvert \F_i \in \mcC_T(i) \;\forall i \in [K]\right] \\
    & = \underbrace{\left .\E_{\bnu}^{\rm U}\left[{\rm UCB}_T(\ba^{\rm U}_T) - U_h\left(\sum_{i\in[K]} a^{\rm U}_T(i)\F_i\right)\right\lvert \F_i \in \mcC_T(i) \;\forall i \in [K]\right]}_{\triangleq B_{21}(T)}\nonumber\\
    &\quad +\underbrace{\left .\E_{\bnu}^{\rm U}\left[U_h\left(\sum_{i\in[K]} a^{\rm U}_T(i)\F_i\right) - U_h\left(\sum_{i\in[K]} \frac{\tau_T^{\rm U}(i)}{T}\F_i\right)\right\lvert \F_i \in \mcC_T(i)\; \forall i \in [K]\right]}_{B_{22}(T)}\ . 
\end{align}
Note that the term $B_{21}(T)$ captures the {\em CDF estimation error} incurred by the UCB algorithm, and the term $B_{22}(T)$ reflects the {\em sampling estimation error} incurred by the UCB algorithm. Subsequently, we analyze each of these components.

\subsection{Upper Bound on the CDF Estimation Error}
Expanding $B_{21}(T)$ we obtain
\begin{align}
    B_{21}(T) &\leq\left .\E_{\bnu}^{\rm U}\left[U_h\left(\sum_{i\in[K]} a_{T}^{\rm U}(i)\widetilde\F_{i,T}\right) - U_h\left(\sum_{i\in[K]} a_{T}^{\rm U}(i)\F_i\right)\right\lvert \F_i \in \mcC_T(i) \;\forall i \in [K]\right]\\
    &\stackrel{\eqref{eq:Holder_opt}}{\leq} \E_{\bnu}^{\rm U} \left .\left[\mathcal{L} \sum_{i\in[K]} \left(a_{T}^{\rm U}(i)\right)^q \norm{ \widetilde{\F}_{i,T}-\F_i}_{\rm W}^q  \right\lvert \F_i \in \mcC_T(i) \;\forall i \in [K] \right] \\
&\stackrel{\eqref{eq:UCB_confidence_sets}}{\leq} \E_{\bnu}^{\rm U} \left . \left[\mcL (32)^q\Big(\sqrt{2\e\log T} + 32\big)^q \sum\limits_{i\in[K]} \big(a_T^{\rm U}(i)\big)^q\left( \frac{1}{\tau_T^{\rm U}(i)}\right)^{\frac{q}{2}} \right\lvert \F_i \in \mcC_T(i) \;\forall i \in [K] \right]\\
    \label{eq:UUB_1}
    &\leq  \E_{\bnu}^{\rm U} \left . \left[ \mcL\left( 32 \Big( \sqrt{2\e\log T}+32\Big)\right)^q\cdot \sum\limits_{i\in[K]} \left (\frac{1}{\tau_T^{\rm U}(i)} \right)^{\frac{q}{2}} \right\lvert \F_i \in \mcC_T(i) \;\forall i \in [K] \right]\\
    \label{eq:UUB_2}
    &\leq \mcL\left( 32 \Big( \sqrt{2\e\log T}+32\Big)\right)^q\cdot \sum\limits_{i\in[K]} \left (\frac{1}{\frac{\rho}{4} T \varepsilon}\right)^{\frac{q}{2}}\\
    &=\frac{\mcL K}{T} \left( \frac{32}{\sqrt{\frac{\rho}{4}}}\right)^q T^{1-\frac{q}{2}} \left(\frac{\Big(\sqrt{2\e\log T} + 32\Big)}{\sqrt{\varepsilon}}\right)^q\ ,
    \label{eq:A21}
\end{align}
where,
\begin{itemize}
    \item~\eqref{eq:UUB_1} follows from the fact that $a_T^{\rm U}(i)\leq 1$ for every $i\in[K]$;
    \item and~\eqref{eq:UUB_2} follows from the explicit exploration of each arm $i\in[K]$ in Algorithm~\ref{algorithm:UCB-M}. 
\end{itemize}

\subsection{Upper Bound on the Sampling Estimation Error}
\label{appendix: UCB sampling estimation error}

Finally, we will bound $B_{22}(T)$. The key idea in upper bounding the sampling estimation error is to show that after a certain instant, the RS-UCB-M algorithm outputs an estimate $\balpha_t^{\rm U}$ of the mixing coefficients that belongs to the set of discrete optimal mixtures with a very high probability. Subsequently, we show that the undersampling routine of the RS-UCB-M algorithm ensures that once it has identified a correct optimal mixture, it navigates its sampling proportions to match the estimated coefficient. Let us denote the vector of arm sampling fractions by $\btau_T^{\rm U}\triangleq [\tau_T^{\rm U}(1),\cdots,\tau_T^{\rm U}(K)]^\top$. For the first step, note that
\begin{align}
    B_{22}(T)\;&\stackrel{\eqref{eq:Holder}}{\leq}\;\mcL \E_{\bnu}^{\rm U}\left [ \norm{\sum\limits_{i\in[K]}\left(a^{\rm U}_T(i) - \frac{\tau_T^{\rm U}(i)}{T}\right)\F_i}^q_{\rm W} \;\bigg\lvert\; \F_i\in \mcC_T(i)\;\forall i \in [K]\right ]\\
    &\stackrel{\eqref{eq:W}}{\leq}\;\mcL\E_{\bnu}^{\rm U}\left[ W^q\cdot \norm{\ba_T^{\rm U} - \frac{1}{T}\btau_T^{\rm U}}_1^q\;\bigg\lvert\; \F_i\in \mcC_T(i)\;\forall i \in [K]\right]\ .
    \label{eq:UCB_UB2}
\end{align}
In order to bound $B_{22}(T)$, we will leverage a low probability event (which we will call $\mcE_{3,T}$), and expand~\eqref{eq:UCB_UB2} by conditioning on this event. In order to define the low probability event, let us first lay down a few definitions. First, let us define the event 
\begin{align}
    \mcE_{0,t} \triangleq \Big\{\F_i \in \mcC_t(i) \quad \forall i \in [K]\Big\}\ .
\end{align}
Let us denote the {\em set} of discrete optimal mixtures of the DR $U_h$ by
\begin{align}
    {\rm OPT}_{\varepsilon}\;\triangleq\; \argmax\limits_{\ba\in\Delta_{\varepsilon}^{K-1}}\; U_h\left (\sum\limits_{i\in[K]} a(i)\F_i\right)\ ,
\end{align}
and the {\em set} of optimistic mixtures at each instant $t$ by
\begin{align}
    \widetilde{\rm OPT}_{\varepsilon,t}\;\triangleq\; \argmax\limits_{\ba\in\Delta_{\varepsilon}^{K-1}}\;\max\limits_{\eta_i\in\mcC_t(i)\;\forall i\in[K]}\; U_h\left ( \sum\limits_{i\in[K]}a(i)\eta_i\right )\ .
\end{align}
Furthermore, for any $\ba^\star\in{\rm OPT}_{\varepsilon}$ and $\widetilde\balpha_t\in\widetilde{\rm OPT}_{\varepsilon,t}$, let us define the events
\begin{align}
    \mcE_{1,t}(x) \triangleq \left\{\left| U_h\left(\sum_{i\in[K]} a^\star(i)\widetilde{\F}_{i,t}\right) - U_h\left(\sum_{i\in[K]} a^\star(i) \F_i\right) \right| < x \right\}\ ,
\end{align}
\begin{align}
    \mcE_{2,t}(x) \triangleq \left\{\left| U_h\left(\sum_{i\in[K]} a^{\rm U}_t(i)\widetilde{\F}_{i,t}\right) - U_h\left(\sum_{i\in[K]} a^{\rm U}_t(i)\F_i\right) \right| < x \right\}\ ,
\end{align}
and
\begin{align}
    \mcE_t(x) \triangleq \mcE_{1,t}(x) \bigcap \mcE_{2,t}(x)\ .
\end{align}
Note that
\begin{align}
    \P_{\bnu}^{\rm U} \Big ( \widetilde{\rm OPT}_{\varepsilon,t} \neq {\rm OPT}_{\varepsilon}\Big )\;= \;\P_{\bnu}^{\rm U}\Big (\exists \ba\in\widetilde{\rm OPT}_{\varepsilon,t} : \ba\notin {\rm OPT}_{\varepsilon} \Big )\ .
\end{align}
Let us denote the mixing coefficients that is contained in $\widetilde{\rm OPT}_{\varepsilon,t}$ and yet not in ${\rm OPT}_{\varepsilon}$ by $\ba_t^{\rm U}$. Accordingly, we have
\begin{align}
    \P_{\bnu}^{\rm U}\Big(\ba^{\rm U}_t\notin{\rm OPT}_{\varepsilon}\Big) &= \P_{\bnu}^{\rm U}\Big (\ba^{\rm U}_t\notin{\rm OPT}_{\varepsilon}\mid \bar \mcE_{0,t}\Big )\P_{\bnu}^{\rm U}(\bar \mcE_{0,t}) + \P_{\bnu}^{\rm U}\Big(\ba^{\rm U}_t\notin{\rm OPT}_{\varepsilon} \mid \mcE_{0,t}\Big)\P(\mcE_{0,t})\\
    & \leq \P_{\bnu}^{\rm U}(\bar \mcE_{0,t}) + \P_{\bnu}^{\rm U}\Big(\ba^{\rm U}_t\notin{\rm OPT}_{\varepsilon} \mid \mcE_{0,t}\Big)\\
    &\stackrel{\eqref{eq:probi}}{\leq} \left(\left(\frac{2}{T^2}+1\right)^K-1\right) + \P_{\bnu}^{\rm U}\Big(\ba^{\rm U}_t\notin{\rm OPT}_{\varepsilon} \mid \mcE_{0,t}\Big)\ .
    \label{eq:UCB_UBh1}
\end{align}
Next, note that    
\begin{align}
    &\P_{\bnu}^{\rm U}\Big(\ba^{\rm U}_t\notin{\rm OPT}_{\varepsilon} \mid \mcE_{0,t}\Big)\nonumber\\
    &\;= \P_{\bnu}^{\rm U}\left(U_h\left(\sum_{i\in[K]} a^\star(i) \F_i\right) > U_h\left(\sum_{i\in[K]} a^{\rm U}_t(i)\F_i\right) + \Delta_{\min}(\varepsilon) \bigg\lvert \mcE_{0,t}\right) \\
    &\; = \P_{\bnu}^{\rm U}\left(U_h\left(\sum_{i\in[K]} a^\star(i) \F_i\right) > U_h\left(\sum_{i\in[K]}  a^{\rm U}_t(i)\F_i\right) + \Delta_{\min}(\varepsilon) \Bigg\lvert \mcE_{0,t},\; \mcE_t\left(\frac{1}{2}\Delta_{\min}(\varepsilon)\right)\right)\nonumber\\
    &\qquad\qquad\qquad\qquad\times\P_{\bnu}^{\rm U}\left(\mcE_t\left(\frac{1}{2}\Delta_{\min}(\varepsilon)\right)\bigg\lvert \mcE_{0,t}\right) \\
    &\quad + \P_{\bnu}^{\rm U}\left(U_h\left(\sum_{i\in[K]} a^\star(i) \F_i\right) > U_h\left(\sum_{i\in[K]} a^{\rm U}_t(i)\F_i\right) + \Delta_{\min}(\varepsilon) \Bigg\lvert \mcE_{0,t},\; \bar\mcE_t\left(\frac{1}{2}\Delta_{\min}(\varepsilon)\right)\right)\nonumber\\
    &\qquad\qquad\qquad\qquad\times\P_{\bnu}^{\rm U}\left(\bar\mcE_t\left(\frac{1}{2}\Delta_{\min}(\varepsilon)\right)\bigg\lvert \mcE_{0,t}\right) \\
    & \;\leq \P_{\bnu}^{\rm U}\left(U_h\left(\sum_{i\in[K]} a^\star(i) \F_i\right) > U_h\left(\sum_{i\in[K]} a^{\rm U}_t(i)\F_i\right) + \Delta_{\min}(\varepsilon) \Bigg\lvert \mcE_{0,t},\; \mcE_t\left(\frac{1}{2}\Delta_{\min}(\varepsilon)\right)\right)\nonumber\\
    &\qquad\qquad\qquad\qquad+\P_{\bnu}^{\rm U}\left(\bar\mcE_t\left(\frac{1}{2}\Delta_{\min}(\varepsilon)\right)\bigg\lvert \mcE_{0,t}\right)   \\
    \label{eq:UCB_UB3}
    &\;\leq \P_{\bnu}^{\rm U} \left( U_h\left(\sum_{i\in[K]} a^\star(i)\widetilde{\F}_{i,t}\right)  >  U_h\left(\sum_{i\in[K]} a^{\rm U}_t(i)\widetilde{\F}_{i,t}\right) \Bigg\lvert \mcE_{0,t},\; \bar\mcE_t\left ( \frac{1}{2}\Delta_{\min}(\varepsilon)\right)\right) \nonumber\\
    &\qquad\qquad\qquad\qquad+\P_{\bnu}^{\rm U}\left(\bar\mcE_t\left(\frac{1}{2}\Delta_{\min}(\varepsilon)\right)\bigg\lvert \mcE_{0,t}\right)   \\
    \label{eq:UCB_UB4}
    &\;=\P_{\bnu}^{\rm U}\left(\bar\mcE_t\left(\frac{1}{2}\Delta_{\min}(\varepsilon)\right)\bigg\lvert \mcE_{0,t}\right)   \\
    \label{eq:UCB_UB5}
    & \;\leq \P_{\bnu}^{\rm U}\left(\bar\mcE_{1,t}\left(\frac{1}{2}\Delta_{\min}(\varepsilon)\right)\bigg\lvert \mcE_{0,t}\right) + \P_{\bnu}^{\rm U}\left(\bar\mcE_{2,t}\left(\frac{1}{2}\Delta_{\min}(\varepsilon)\right)\bigg\lvert \mcE_{0,t}\right)\ ,
\end{align}
where
\begin{itemize}
    \item \eqref{eq:UCB_UB3} follows from the definition of the event $\mcE_t(\Delta_{\min}(\frac{1}{2}\varepsilon))$;
    \item \eqref{eq:UCB_UB4} follows from the definitions of $\ba^\star$ (which is an optimizer in the set ${\rm OPT}_{\varepsilon}$) and $\balpha_t^{\rm U}$ (which is an optimizer in the set $\widetilde{\rm OPT}_{\varepsilon,t}$);
    \item and~\eqref{eq:UCB_UB5} follows from a union bound.
\end{itemize}

Next, we will find upper bounds on the two probability terms in~\eqref{eq:UCB_UB5}. Note that for any $t>K\lceil\rho T\varepsilon/4\rceil$ we have
\begin{align}
    &\P_{\bnu}^{\rm U}\left(\bar\mcE_{1,t}\left(\frac{1}{2}\Delta_{\min}(\varepsilon)\right)\bigg\lvert \mcE_{0,t}\right)\nonumber\\
    &\qquad= \P_{\bnu}^{\rm U}\left( \left| U_h\left(\sum_{i\in[K]} a^\star(i)\widetilde{\F}_{i,t}\right) - U_h\left(\sum_{i\in[K]} a^\star(i) \F_i\right) \right| \geq \frac{1}{2}\Delta_{\min}(\varepsilon) \bigg\lvert \mcE_{0,t} \right )\\
    &\qquad\stackrel{\eqref{eq:Holder}}{\leq} \P_{\bnu}^{\rm U} \left(\sum_{i\in[K]} (a^\star(i))^q  \norm{ \widetilde{\F}_{i,t}-\F_i}_{\rm W}^q \geq \frac{1}{2\mathcal{L}}\Delta_{\min}(\varepsilon) \bigg\lvert \mcE_{0,t} \right )\\
    &\qquad\stackrel{\eqref{eq:UCB_confidence_sets}}{=} \P_{\bnu}^{\rm U}\left ( \sum\limits_{i\in[K]}\big(a^\star(i))^q \norm{\widetilde\F_{i,t} - \F_i}_{\rm W}^q\bigg\lvert \mcE_{0,t}\right)\\
    &\qquad\stackrel{\eqref{eq:UCB_confidence_sets}}{\leq} \P_{\bnu}^{\rm U}\left( \sum\limits_{i\in[K]}\underbrace{\big( a^\star(i)\big)^q}_{\leq 1}\left(16\sqrt{\frac{1}{\tau_T^{\rm U}}(i)}\cdot\Big( \sqrt{2\e\log T} + 32\Big)\right)^q>\frac{1}{2\mcL}\Delta_{\min}(\varepsilon)\bigg\lvert \mcE_{0,t}\right)\\
    \label{eq:E3_1}
    &\qquad\leq \sum\limits_{i\in[K]} \P_{\bnu}^{\rm U}\left(\left(\sqrt{\frac{1}{\tau_T^{\rm U}}(i)}\cdot\Big( \sqrt{2\e\log T} + 32\Big)\right)^q>\frac{1}{2K\mcL(16)^q}\Delta_{\min}(\varepsilon)\bigg\lvert \mcE_{0,t}\right)\\
    \label{eq:E3_2}
    &\qquad\leq \sum\limits_{i\in[K]} \P_{\bnu}^{\rm U}\left(\left(\frac{\Big( \sqrt{2\e\log T} + 32\Big)}{\sqrt{\frac{\rho}{4} T \varepsilon}}\right)^q>\frac{1}{2K\mcL(16)^q}\Delta_{\min}(\varepsilon)\bigg\lvert \mcE_{0,t}\right),
\end{align}
where the transition~\eqref{eq:E3_1}-\eqref{eq:E3_2} follows from the explicit exploration phase of the RS-UCB-M algorithm.
Now, let us define a time instant $T_0(\varepsilon)$ as follows.
\begin{align}
\label{eq:T_epsilon}
    T_0(\varepsilon)\; \triangleq \;\inf \left\{ t\in\N :\left(\frac{\Big( \sqrt{2\e\log s} + 32\Big)}{\sqrt{ \frac{\rho}{4} s \varepsilon }}\right) \leq \frac{1}{16} \left( \frac{\Delta_{\min}(\varepsilon)}{2K\mcL}\right)^{\frac{1}{q}} \quad \forall s \geq t\right\}\ .
\end{align}

Hence, $\forall t \geq T_0(\varepsilon)$ we have 
\begin{align}
\label{eq:UCB_UB8}
     \P_{\bnu}^{\rm U}\left(\bar\mcE_{1,t}\left(\frac{1}{2}\Delta_{\min}(\varepsilon)\right)\bigg\lvert \mcE_{0,t}\right)\;=\;0\ .
\end{align}
Using a similar line of arguments, we may readily show that for all $t \geq T_0(\varepsilon)$,
\begin{align}
\label{eq:UCB_UB9}
    \P_{\bnu}^{\rm U}\left(\bar\mcE_{2,t}\left(\frac{1}{2}\Delta_{\min}(\varepsilon)\right)\bigg\lvert \mcE_{0,t}\right)\;=\;0\ .
\end{align}
Combining~\eqref{eq:UCB_UB8} and~\eqref{eq:UCB_UB9} we infer that for all $t \geq T_0(\varepsilon)$,
\begin{align}
\label{eq:UCB_UB10}
    \P_{\bnu}^{\rm U}\Big(\ba^{\rm U}_t\notin{\rm OPT}_{\varepsilon} \mid \mcE_{0,t}\Big)\;=\;0\ ,
\end{align}
which implies that
\begin{align}
\label{eq:UCB_UB10-1}
    \P_{\bnu}^{\rm U}\Big(\widetilde{\rm OPT}_{\varepsilon,t}\neq{\rm OPT}_{\varepsilon} \mid \mcE_{0,t}\Big)\;=\;0\ .
\end{align}
We are now ready to define the low probability event $\mcE_{3,T}$. Let us define
\begin{align}
    \mcE_{3,T}\triangleq\Big\{\exists t \in[T_0(\varepsilon), T]: \widetilde{\rm OPT}_{\varepsilon,t}\neq{\rm OPT}_{\varepsilon}\Big\}\ .
\end{align}
Accordingly, we have the following lemma.

\begin{lemma} 
\label{lemma:probability of E3}
For event \(\mcE_{3,T}\), we have
\begin{align}
    \mathbb{P}_{\bnu}^ {\rm U}\left(\mcE_{3,T}\right) \leq T\left(\left(\frac{2}{T^2}+1\right)^K-1\right)\ .
\end{align}
\end{lemma}

\begin{proof}
Note that
    \begin{align}
\mathbb{P}_{\bnu}^{\rm U}\Big(\exists t>T_0(\varepsilon): \widetilde{\rm OPT}_{\varepsilon,t}\neq{\rm OPT}_{\varepsilon}\Big)  &\leq \sum_{t=T_0(\varepsilon)+1}^T \mathbb{P}\left(\widetilde{\rm OPT}_{\varepsilon,t}\neq{\rm OPT}_{\varepsilon}\right) \\
& \stackrel{\eqref{eq:UCB_UBh1}}{\leq} \sum_{t=T_0(\varepsilon)+1}^T \left(\left(\frac{2}{T^2}+1\right)^K-1\right) \nonumber \\ & \qquad + \P_{\bnu}^{\rm U}\Big(\ba^{\rm U}_t\notin{\rm OPT}_{\varepsilon}\med \mcE_{o,t}\Big)\\
& \stackrel{\eqref{eq:UCB_UB10}}{\leq} T\left(\left(\frac{2}{T^2}+1\right)^K-1\right)\ .
\end{align}
\end{proof}

First, thanks to the discretization, all arms are sampled at least once, according to under-sampling.
Let us assume that after the explicit exploration, an arm is never under sampled.
\begin{align}
\label{eq:samlingequiv}
    \tau_t(i) = \frac{\rho}{4} T \epsilon
\end{align}
If an arm is not under sampled that would mean 
\begin{align}
\label{eq:oversampledproof}
    \tau_t(i) \geq t a_t^{U}(i) \ .
\end{align}
\begin{align}
\label{eq:replace}
    \frac{\rho}{4} T \epsilon t^{-1} & \geq a_t^{U}(i) \\
    \label{eq:midpointsmatter}
    & \geq \frac{\varepsilon}{2} 
\end{align}
where 
\begin{itemize}
    \item \eqref{eq:replace} follows from \eqref{eq:samlingequiv} and \eqref{eq:oversampledproof};
    \item and \eqref{eq:midpointsmatter} follows from the definition of discretization.
\end{itemize}
\begin{align}
    \frac{\rho}{2} & \geq \frac{t}{T} \ .
\end{align}
For $t > \frac{\rho T}{2}$, this inequality does not hold. Therefore, if an arm is not sampled after the explicit exploration, it becomes under sampled at the latest at the time instant $\frac{\rho T}{2}$.

Next, we will show that under the event $\bar\mcE_{3,T}$, i.e., when the RS-UCB-M algorithm only selects a mixture distribution from the set of optimal mixtures, the RS-UCB-M arm selection routine, using the {\em under-sampling} procedure, eventually converges to an optimal mixture distribution. This is captured in the following lemma. Prior to stating the lemma, note that under the event $\bar\mcE_{3,T}$, the set of estimated mixing coefficients is the same in each iteration between $T_0(\varepsilon)$ and $T$. Hence, the RS-UCB-M algorithm aims to track {\em one} of these optimal mixing coefficients, which we denote by $\ba^\star\in{\rm OPT}_{\varepsilon}$.

\begin{lemma}
\label{lemma:undersampling}
Under the event $\bar\mcE_{3,T}$, there exists a time instant $T(\varepsilon)<+\infty$, and $\ba^\star\in{\rm OPT}_{\varepsilon}$ such that we have
\begin{align}
    \left| \frac{\tau_t^{\rm U}(i)}{t} - a^\star(i)\right| < \frac{K}{t}  \quad \forall t \geq T(\varepsilon)\ .
\end{align}
\end{lemma}

\begin{proof}
We begin by by defining a set of {\em over-sampled} arms as follows. We define the set $\mcO_t$ as
\begin{align}
    \mcO_t \;\triangleq\; \left\{i \in [K]: \frac{\tau_t^{\rm U}(i)}{t} > a^\star(i) + \frac{1}{t}\right\}\ .
\end{align}
We will leverage Lemma~\ref{lemma:oversampling}, in which we show that there exists a finite time instant, which we call $T(\varepsilon)$, such that for all $t>T(\varepsilon)$ the set of over-sampled arms is {\em empty}. Specifically, leveraging Lemma~\ref{lemma:oversampling} we obtain
\begin{align}
    \frac{\tau_t^{\rm U}(i)}{t} - a^\star(i) \;\leq\; \frac{1}{t}  \quad \forall t \geq T(\varepsilon)\ ,
\end{align}
which also implies that
\begin{align}
\label{eq:UK_oneside}
    \frac{\tau_t^{\rm U}(i)}{t} - a^\star(i) \;<\; \frac{K}{t}  \quad \forall t \geq T(\varepsilon)\ .
\end{align}
Next, let us assume that there exists some $j\in[K]$ such that
\begin{align}
    \frac{\tau_t^{\rm U}(j)}{t} < a^\star(j)  - \frac{K}{t} \ .
\end{align}
In this case, we have
\begin{align}
    \sum\limits_{i\in[K]} \frac{\tau_t^{\rm U}(i)}{t} &= \sum\limits_{i \neq j} \frac{\tau_t^{\rm U}(i)}{t} + \frac{\tau_t^{\rm U}(j)}{t} \\
    & \stackrel{\eqref{eq:UK_oneside}}{\leq}  \sum\limits_{i \neq j} a^\star(i)  + \frac{K-1}{t} + a^\star(j)  - \frac{K}{t}  \\
    & = 1 - \frac{1}{t}\ ,
\end{align}
which is a contradiction, since we should have 
\begin{align}
    \sum\limits_{i\in[K]} \frac{\tau_t^{\rm U}(i)}{t} \;=\; 1\ .
\end{align}
Hence, we can conclude that 
\begin{align}
  \left| \frac{\tau_t^{\rm U}(i)}{t} - a^\star(i)\right| < \frac{K}{t}  \quad \forall t \geq T(\varepsilon) \ . 
\end{align}
\end{proof}

\begin{lemma}
\label{lemma:oversampling}
Under the event $\bar\mcE_{3,T}$, there exists a finite time $T(\varepsilon)<+\infty$ such that for every $t>T(\varepsilon)$, we have $\mcO_t = \emptyset$.
\end{lemma}
\begin{proof}
    For any time $t>T_0(\varepsilon)$, 
    and some arm $j\notin\mcO_t$, we have
    \begin{align}
        \frac{\tau_t^{\rm U}(j)}{t}\;\leq\; a^\star(j) + \frac{1}{t}\ .
    \end{align}
    Based on this we have the following two regimes of the arm selection fraction for arm $j$.

    \textbf{Case (a): ($\tau_t^{\rm U}(j)/t \leq \bar\alpha_j^\star$).} Since the RS-UCB-M algorithm only samples under-sampled arms, it is probable that the arm $j$ is sampled at time $t+1$. In that case, we would have
    \begin{align}
    \frac{\tau_{t+1}^{\rm UCB}(j)}{t+1} \;&\leq\; \frac{\tau_t^{\rm U}(j)+1}{t+1}  \\
    & =\; \frac{\tau_t^{\rm U}(j)}{t+1} + \frac{1}{t+1} \\
    \label{eq:oversampling1}
    & \leq\; \frac{a^\star(j) t}{t+1} + \frac{1}{t+1} \\
    \label{eq:UK_1}
    & \leq\; a^\star(j) + \frac{1}{t+1}\ ,
    \end{align}
    where~\eqref{eq:oversampling1} follows from the fact that $\tau_t^{\rm U}(j)\leq t\bar\alpha_j^\star$. Clearly, we see that the arm $j$ does not enter the set of over-sampled arms $\mcO_t$ in the subsequent round $t+1$.

    \textbf{Case (b): ($\bar\alpha_j^\star\leq\tau_t^{\rm U}(j)/t \leq \bar\alpha_j^\star + 1/t$).} In this case, there always exists at least one arm $k\in[K]$ which satisfies that $\tau_t^{\rm U}(k)/t\leq a^\star_k$. This implies that the arm $j$ is not sampled by the RS-UCB-M arm selection rule, since it is not the most under-sampled arm. In this case, we have
    \begin{align}
    \frac{\tau_{t+1}^{\rm UCB}(j)}{t+1} \;&=\; \frac{\tau_t^{\rm U}(j)}{t+1}  \\
    & =\; \frac{\tau_t^{\rm U}(j)}{t}\frac{t}{t+1} \\
    \label{eq:oversampling2}
    & \leq\; \left(a^\star(j) + \frac{1}{t}\right)\frac{t}{t+1} \\
    & =\; a^\star(j)\frac{t}{t+1} + \frac{1}{t+1} \\ 
    \label{eq:UK_2}
    & \;\leq a^\star(j) + \frac{1}{t+1}\ ,
\end{align}
where~\eqref{eq:oversampling2} from the definition of case (b). Again, we have concluded that $j\notin\mcO_{t+1}$. Combining~\eqref{eq:UK_1} and~\eqref{eq:UK_2}, we conclude that none of the arms which are not already contained in the set $\mcO_t$ ever enters this set. Hence, all we are left to show is the existence of the time instant after which the set $\mcO_t$ becomes an empty set. Evidently, since the RS-UCB-M algorithm never samples from the set of over-sampled arms $\mcO_t$, we will derive an upper bound $m$ such that after $t > T_0(\varepsilon) + m$, all arms leave the set $\mcO_t$. Notably, for any arm $i\in\mcO_t$, it holds that
\begin{align}
    \frac{\tau_t^{\rm U}(i)}{t}\; >\; a^\star(i) + \frac{1}{t}\ .
\end{align}
Let $m$ be such that
\begin{align}
\label{eq:oversampling3}
    \frac{\tau_{T_0(\varepsilon)+m}^{\rm U}(i)}{T_0(\varepsilon)+m}\;\leq\;a^\star(i) + \frac{1}{T_0(\varepsilon)+m}\ ,
\end{align}
which implies that the arm $i$ has left the set $\mcO_{T(\varepsilon)+m}$. Furthermore, since arm $i$ is never sampled between times instants $T(\varepsilon)$ and $T(\varepsilon)+m$ (as it belongs to the over-sampled set),~\eqref{eq:oversampling3} can be equivalently written as
\begin{align}
\label{eq:m_finite}
    \frac{\tau_{T_0(\varepsilon)}^{\rm U}(i)}{T_0(\varepsilon)+m}\;\leq\;a^\star(i)+ \frac{1}{T_0(\varepsilon)+m}\ .
\end{align}
Next, noting that for any arm $i\in[K]$, we have the upper bound $\tau_t^{\rm U}(i)\leq t$ on the number of times $i$ is chosen up to time $t$, and that 
$a^\star(i)\geq\varepsilon / 2$ for every $i\in[K]$, we have the following choice for $m$.  
\begin{align}
\label{eq:m}
    m \; = \; \left(\frac {2}{\varepsilon}-1 \right)T_0(\varepsilon) - \frac{2}{\varepsilon}\ . 
\end{align}
The proof is completed by defining 
\begin{align}
\label{eq:T_epsilon_actual}
    T(\varepsilon) \triangleq T_0(\varepsilon) + m\ .
\end{align}
\end{proof}


Next, from~\eqref{eq:UCB_UB2}, we have
\begin{align}
    B_{22}(T)\;&\leq\; \mcL W^q \E_{\bnu}^{\rm U} \left [ \sum_{i\in[K]} \left\lvert a_T^{\rm U}(i) - \frac{\tau_T^{\rm}(i)}{T}\right\rvert^q\;\Bigg\lvert\; \F_i\in\mcC_T(i)\;\forall\; i\in[K]\right ]\\
    \label{eq:UCB_f1}
    &\leq\; \frac{1}{\P_{\bnu}^{\rm U}\Big( \F_i\in\mcC_T(i)\;\forall\; i\in[K]\Big)}\mcL W^q \E_{\bnu}^{\rm U} \left [ \sum_{i\in[K]}\left \lvert a^{\rm U}_T(i) - \frac{\tau_T^{\rm U}(i)}{T}\right\rvert^q\right]\\
    \label{eq:UCB_f2}
    &\leq\; \frac{1}{2-\Big( \frac{1}{T^2} + 1\Big)^K}\mcL W^q \E_{\bnu}^{\rm U} \left [ \sum_{i\in[K]}\left \lvert a^{\rm U}_T(i) - \frac{\tau_T^{\rm U}(i)}{T}\right\rvert^q\right]\ , 
\end{align}
where,
\begin{itemize}
    \item[(i)]~\eqref{eq:UCB_f1} follows from the fact that $\E[A|B] = (\E[A] - \E[A|\bar B]\cdot\P(\bar B))/\P(B) \leq \E[A]/\P(B)$, and,
    \item[(ii)]~\eqref{eq:UCB_f2} follows from~\eqref{eq:probi}. 
\end{itemize}
Furthermore, for all $T\geq 3$ and $K\geq 2$, it can be readily verified that $2 - (2/T^2 + 1)^K > 1/2$, which implies that
\begin{align}
    B_{22}(T) & \stackrel{\eqref{eq:UCB_f2}}{\leq} 2\mcL W^q \E_{\bnu}^{\rm U}\left [ \sum_{i\in[K]}\left\lvert a_T^{\rm U} - \frac{\tau_T^{\rm U}(i)}{T}\right\rvert^q\right]
    \\
    &= 2\mcL W^q \E_{\bnu}^{\rm U}\left [ \sum_{i\in[K]}\left\lvert a_T^{\rm U} - \frac{\tau_T^{\rm U}(i)}{T}\right\rvert^q\;\bigg\lvert \mcE_{3,T}\right]\P_{\bnu}^{\rm U}\Big( \mcE_{3,T}\Big) + 2\mcL W^q \E_{\bnu}^{\rm U}\left [ \sum_{i\in[K]}\left\lvert a_T^{\rm U} - \frac{\tau_T^{\rm U}(i)}{T}\right\rvert^q\;\bigg\lvert\;\bar\mcE_{3,T}\right]\P_{\bnu}^{\rm U}\Big( \bar\mcE_{3,T}\Big)\\
    \label{eq:UCB_UB11}
    &\leq 2\mcL K W^q \left(\P_{\bnu}^{\rm U}(\mcE_{3,T}) + K\left(\frac{K}{T}\right)^q\right)\\
    \label{eq:UCB_UB12}
    &\leq 2\mcL K W^q  \left(T\left(\left(\frac{2}{T^2}+1\right)^K-1\right) + \left(\frac{K}{T}\right)^q\right)\ ,
\end{align}
where~\eqref{eq:UCB_UB11} follows from Lemma~\ref{lemma:oversampling} and~\eqref{eq:UCB_UB12} follows from Lemma~\ref{lemma:probability of E3}.

\subsection{Proof of Theorem~\ref{theorem:UCB upper bound} for RS-UCB-M}

Finally, we add up all the terms to find an upper bound on the discrete regret. We have
\begin{align}
    \label{eq:Discretized_regret}
    \Bar{\mathfrak{R}}_{\bnu}^{\rm U}(T)\;&=\; B_1(T) + B_{21}(T) + B_{22}(T) \\
    &\leq\; \frac{1}{T} \Bigg [ BT\left( \left( \frac{2}{T^2}+1\right)^K - 1\right) \\ & \qquad +{\mcL K} \left( \frac{32}{\sqrt{\rho}}\right)^q T^{1-\frac{q}{2}} \Big(\sqrt{2\e\log T} + 32\Big)^q\nonumber\\
    &\qquad +2\mcL W^q  \left(KT^2\left(\left(\frac{2}{T^2}+1\right)^K-1\right) + K^{1+q}T^{1-q}\right) \Bigg]\ . 
\end{align}

Leveraging the upper bounds on discretization error and the discrete regret, we state an upper bound on the regret of the RS-UCB-M algorithm.
\begin{align}
&\mathfrak{R}_{\bnu}^{\rm U}(T) = \Delta(\varepsilon)  + \Bar{\mathfrak{R}}_{\bnu}^{\rm U}(T) \\ & \leq \Delta(\varepsilon) + \frac{1}{T} \Bigg [ BT\left( \left( \frac{2}{T^2}+1\right)^K - 1\right)\nonumber\\
    &\quad + 2\mcL W^q \left(KT^2\left(\left(\frac{2}{T^2}+1\right)^K-1\right) + K^{1+q}T^{1-q}\right)\nonumber\\
    & \qquad +\underbrace{ {\mcL K} \left( \frac{64}{\sqrt{\rho {\varepsilon}}}\right)^q T^{1-\frac{q}{2}} \Big(\sqrt{2\e\log T} + 32\Big)^q\Bigg] }_{B_3(T)} \ . \\
    \end{align}
\(B_3(T)\) is the dominating term for any $T>\e^K$. Hence, we can simplify the upper bound as follows.
    \begin{align}
    \label{eq:theorem_lastline}
 & \mathfrak{R}_{\bnu}^{\rm U}(T) \leq \Delta(\varepsilon) + (B + \mcL(W^q+1))\Big({64\varepsilon^{-\frac{1}{2}}} \rho^{-\frac{1}{2}} T^{-\frac{1}{2}}\Big( \sqrt{2\e\log T} + 32\Big)\Big)^q \ .
\end{align}

\subsection{Proof of Theorem~\ref{corollary:RS-UCB-M} for RS-UCB-M}
Let us set 
\begin{align}
\varepsilon = \Theta \left(\left(K^{\frac{2}{q}}\log T / T\right)^{\kappa}\right)\ ,
\end{align}
 where $\kappa = \frac{1}{\frac{2\beta}{q}+2}$.  
Assuming \(\Delta_{\min}(\varepsilon) = \Omega(\varepsilon^\beta)\ \), it can be readily verified that this choice of the discretization level $\varepsilon$ satisfies the condition in~\eqref{eq:T_epsilon_actual}. Hence, for discrete regret, from~\eqref{eq:Discretized_regret} we have
\begin{align}
\label{eq:Discrete_order_UCB_RS}
    \Bar{\mathfrak{R}}_{\bnu}^{\rm U}(T) &= O \left ((K^{\frac{2}{q}}\log T/T)^{\left(1-\kappa\right)\frac{q}{2}} \right) 
    \ .
\end{align}
Furthermore, from Lemma~\eqref{lemma:Delta_error}, we have
\begin{align}
    \Delta(\varepsilon) &\leq \mcL (KW)^r \left(\frac{\varepsilon}{2}\right)^r  \ .
\end{align}
Hence, \(\Delta(\varepsilon) = O(\varepsilon^r)\). Consequently, we have
\begin{align}
\label{eq:delta_order_UCB_RS}
    \Delta(\varepsilon) = O\left(K^{r \left( 1 + \frac{2\kappa}{q} \right)}\left(\log T / T\right)^{r \kappa} \right).
\end{align}
From~\eqref{eq:Discrete_order_UCB_RS} and \eqref{eq:delta_order_UCB_RS} the regret of the RS-UCB-M is bounded from above by
    \begin{align}
    \label{eq:UCB upper bound appendix}
        \mathfrak{R}_{\bnu}^{\rm U}(T) &= O \left (\max\Big\{ K^{r \left( 1 + \frac{2\kappa}{q} \right)}\left(\log T / T\right)^{r \kappa}  \;,\;(K^{\frac{2}{q}}\log T/T)^{\left(1-\kappa\right)\frac{q}{2}} \Big\} \right) \\
        & = O \left (\max\Big\{K^{r \left( 1 + \frac{2\kappa}{q} \right)}, K^{1-\kappa} \Big\} \max\Big\{ \left(\log T / T\right)^{r \kappa}  \;,\;  (\log T/T)^{(1-\kappa)\frac{q}{2}}\Big\} \right)\ \ .
    \end{align}
When $r\leq \beta + q/2$, this bound becomes
\begin{align}
    \mathfrak{R}_{\bnu}^{\rm U}(T) \leq O \left (\max\Big\{K^{r \left( 1 + \frac{2\kappa}{q} \right)}, K^{1-\kappa} \Big\}   (\log T/T)^{r\kappa} \right)\ .
\end{align}

\section{CE-UCB-M Algorithm and Its Performance}
\label{sec:SE-UCB-M proofs}
Finally, in this section, we provide the pseudo-code of the CE-UCB-M algorithm and the proof of Theorem \ref{theorem:UCB upper bound} which characterize the upper bound on the algorithm's regret.
\label{Appendix:Upper_last_UCB}
\subsection{CE-UCB-M Algorithm}
The detailed steps of the RS-UCB-M algorithm with the sampling rule specified as in \eqref{eq:UCB_sampling_rule} are presented in Algorithm \ref{algorithm:UCB-M}.
\label{Appendix:CE-UCB-M-alg}
				\begin{algorithm}[h]
		\caption{Risk-Sensitive Upper Confidence Bound for Mixtures (CE-UCB-M)}
		\label{algorithm:UCB-M}
		
 		\begin{algorithmic}[1]
            \STATE \textbf{Input:} Exploration rate $\rho$, horizon $T$
            \STATE Sample each arm $\lceil \rho T \varepsilon/ 4 \rceil$ times and obtain observation sequences $\mcX_{\lceil \rho T \varepsilon/ 4\rceil}(1),\cdots,\mcX_{\lceil \rho T \varepsilon/ 4 \rceil}(K)$
            \STATE \textbf{Initialize:} $\tau_{K\lceil \rho T \varepsilon/ 4 \rceil}^{\rm C}(i) = \lceil \rho T \varepsilon/ 4 \rceil\;\forall\;i\in[K]$, arm CDFs $\F^{\rm C}_{1, \lceil K \rho T \varepsilon/ 4 \rceil}(1),\cdots,\F^{\rm C}_{K, K\lceil \rho T \varepsilon/ 4\rceil}$, confidence sets $\mcC_{K \lceil \rho T \varepsilon/ 4 \rceil}(1),\cdots\mcC_{K\lceil \rho T \varepsilon/ 4\rceil}(K)$ according to~\eqref{eq:UCB_confidence_sets}
			\FOR{$t= K \lceil \rho T \varepsilon/ 4\rceil+1,\cdots,T$}
			    \STATE Select an arm $A_{t}$ specified by~(\ref{eq:UCB_sampling_rule}) and obtain reward $X_t$\\
                \STATE Update the empirical CDF $\F^{\rm C}_{A_t,t}$ according to~\eqref{eq:empirical_CDF}
                \STATE Compute the optimistic estimate $\ba^{\rm C}_t$ according to~\eqref{eq:UCB_alpha2}
			\ENDFOR
 		\end{algorithmic}
	\end{algorithm}

\subsection{Proof of Theorem~\ref{theorem:UCB upper bound} for CE-UCB-M}
\label{Appendix:CE-UCB-M_algorithm}
The analysis of the CE-UCB-M algorithm closely follows that of the RS-UCB-M algorithm with some minute differences. We will briefly state the steps in the RS-UCB-M analysis, and in the process, highlight the key distinctions in the analysis of the CE-UCB-M algorithm. Henceforth, we will use $\pi={\rm C}$ to denote the policy CE-UCB-M. 
Similarly to the RS-UCB-M analysis and~\eqref{eq:regret_decomposition}, we decompose the regret into a discretization error component $\Delta(\varepsilon)$, and the discrete regret $\bar{\mathfrak{R}}_{\bnu}^{\rm C}(T)$. Leveraging Lemma~\ref{lemma:Delta_error}, it may be readily verified that
\begin{align}
    \Delta(\varepsilon) \leq \mcL (KW)^r \left(\frac{\varepsilon}{2}\right)^r \ .
\end{align}

Next, we digress to upper bounding the discrete regret $\bar\R_{\bnu}^{\rm C}(T)$. We have
\begin{align}
    \Bar{\mathfrak{R}}_{\bnu}^{\rm C}(T) &=   U_h\left(\sum_{i\in[K]} a^\star(i)\F_i\right)  - \E_{\bnu}^{\rm C} \left[U_h\left(\sum_{i\in[K]} \frac{\tau_t^{\rm C}(i)}{T}\F_i\right)\right] \\
    &\leq \underbrace{\sum_{\mcS \subseteq [K]: \mcS \neq \emptyset}
    \E_{\bnu}^{\rm C} \left[U_h\left(\sum_{i\in[K]} a^\star(i)\F_i\right)  -  U_h\left(\sum_{i\in[K]} \frac{\tau_t^{\rm C}(i)}{T}\F_i\right)\mathds{1}\{\F_i \notin \mcC_T(i): i \in \mcS\}\right]}_{\triangleq C_1(T)} \\
    &\qquad + \underbrace{\E_{\bnu}^{\rm C} \left[ U_h\left(\sum_{i\in[K]} a^\star(i)\F_i\right)  - U_h\left(\sum_{i\in[K]} \frac{\tau_t^{\rm C}(i)}{T}\F_i\right) \mathds{1}\{\F_i \in \mcC_T(i) \;\forall i \in [K]\} \right]}_{\triangleq C_2(T)}\ ,
\end{align}
where $\mcC_t(i)$ has the same definition as in~\eqref{eq:UCB_confidence_sets} for every $i\in[K]$. Furthermore, we have already upper bounded the term $A_1(T)$ in Appendix~\ref{proof:UCB upper bound} (equation ~\eqref{eq:A_1 bound refer}), which is given by
\begin{align}
    C_1(T)\;\leq\; B\left(\left(\frac{1}{T^2}+1\right)^K-1\right)\ .
\end{align}
We now resort to upper bounding the term $C_2(T)$. Note that
\begin{align}
    C_2(T)\;&=\; \E_{\bnu}^{\rm C} \left[ \left(U_h\left(\sum_{i\in[K]} a^\star(i)\F_i\right)  - U_h\left(\sum_{i\in[K]} \frac{\tau_t^{\rm C}(i)}{T}\F_i\right) \right)\mathds{1}\{\F_i \in \mcC_T(i) \;\forall i \in [K]\} \right]\\
    &\leq\; \E_{\bnu}^{\rm C} \left[ U_h\left(\sum_{i\in[K]} a^\star(i)\F_i\right)  - U_h\left(\sum_{i\in[K]} \frac{\tau_t^{\rm C}(i)}{T}\F_i\right) \;\Big\lvert\; \F_i \in \mcC_T(i) \;\forall i \in [K]\right]\\
    \label{eq:CEUCB1}
    &\leq \E_{\bnu}^{\rm C} \Bigg[ U_h\left(\sum_{i\in[K]} a^\star(i)\F_{i,T}^{\rm C}\right)  - U_h\left(\sum_{i\in[K]} \frac{\tau_t^{\rm C}(i)}{T}\F_i\right) \nonumber\\
    &\qquad\qquad\qquad +\mcL\sum\limits_{i\in[K]}a^\star(i)^q \left(16\ \frac{\sqrt{2 {\rm e} \log T } + 32}{\sqrt{\tau^{\rm C}_t(i)}}  \right)^{q}\;\Big\lvert\; \F_i \in \mcC_T(i) \;\forall i \in [K]\Bigg]\\
    \label{eq:CEUCB2}
    &\leq \E_{\bnu}^{\rm C}\Bigg[ U_h\left(\sum_{i\in[K]} a^{\rm C}_T(i)\F_{i,T}^{\rm C}\right)  - U_h\left(\sum_{i\in[K]} \frac{\tau_t^{\rm C}(i)}{T}\F_i\right) \nonumber\\
    &\qquad\qquad\qquad +\mcL\sum\limits_{i\in[K]}a^{\rm C}_T(i)^q\left(16\ \frac{\sqrt{2 {\rm e} \log T } + 32}{\sqrt{\tau^{\rm C}_t(i)}}  \right)^{q}\;\Big\lvert\; \F_i \in \mcC_T(i) \;\forall i \in [K]\Bigg]\\
    &\leq \underbrace{\E_{\bnu}^{\rm C} \left[ U_h\left(\sum_{i\in[K]} a^{\rm C}_T(i)\F_{i,T}^{\rm C}\right)  - U_h\left(\sum_{i\in[K]} \frac{\tau_t^{\rm C}(i)}{T}\F_i\right) \;\Big\lvert\; \F_i \in \mcC_T(i) \;\forall i \in [K]\right]}_{\triangleq \;C_3(T)}\nonumber\\
    &\qquad\qquad\qquad +\mcL K\left(32\ \frac{\sqrt{2 {\rm e} \log T } + 32}{\sqrt{\rho T {\varepsilon}}}  \right)^{q}\ ,
    \label{eq:CEUCB3}
\end{align}
where,
\begin{itemize}
    \item~\eqref{eq:CEUCB1} follows from \holder  defined in Definition~\ref{assumption:Holder} and the conditioning on the fact that $\F_i\in\mcC_T(i)$ for every $i\in[K]$; 
    \item~\eqref{eq:CEUCB2} follows from the upper confidence bound in~\eqref{eq:UCB_alpha2};
    \item and~\eqref{eq:CEUCB3} follows from the explicit exploration phase of the RS-UCB-M algorithm.
\end{itemize}
Furthermore, note that $C_3(T)$ can be expanded as
\begin{align}
    C_3(T)\;&=\; \E_{\bnu}^{\rm C} \Bigg[ U_h\left(\sum_{i\in[K]} a^{\rm C}_T(i)\F_{i,T}^{\rm C}\right)  - U_h\left(\sum\limits_{i\in[K]} \alpha_T(i)\F_i\right)\nonumber\\
    &\qquad\qquad +U_h\left(\sum\limits_{i\in[K]}a^{\rm C}_T(i)\F_i\right) - U_h\left(\sum_{i\in[K]} \frac{\tau_t^{\rm C}(i)}{T}\F_i\right) \;\Big\lvert\; \F_i \in \mcC_T(i) \;\forall i \in [K]\Bigg]\\
    &\leq \underbrace{\E_{\bnu}^{\rm C} \Bigg[U_h\left(\sum\limits_{i\in[K]} a^{\rm C}_T(i)\F_i\right) - U_h\left(\sum_{i\in[K]} \frac{\tau_t^{\rm C}(i)}{T}\F_i\right) \;\Big\lvert\; \F_i \in \mcC_T(i) \;\forall i \in [K]\Bigg]}_{C_4(T)}\nonumber\\
    &\qquad\qquad\qquad+\mcL K\left(32\ \frac{\sqrt{2 {\rm e} \log T } + 32}{\sqrt{\rho T {\varepsilon}}}  \right)^{q}\ .
\end{align}
Finally, note that the term $C_4(T)$ is similar to the term $B_{22}(T)$ in the RS-UCB-M analysis in Appendix~\ref{proof:UCB upper bound}, and can be handled in the exact same way. Recall the finite time instant $T(\varepsilon) = T_0(\varepsilon) + m$, where $m$ has been defined in~\eqref{eq:m} and $T_0(\varepsilon)$ has been defined in~\eqref{eq:T_epsilon}. For all $T>T(\varepsilon)$, we have the following bound on $C_4(T)$.
\begin{align}
    C_4(T)\;\leq\; 2\mcL K W^q  \left(T\left(\left(\frac{1}{T^2}+1\right)^K-1\right) + \left(\frac{K}{T}\right)^q\right)\ .
\end{align}

Aggregating $C_1(T)--C_4(T)$, we have that for all $T>T(\varepsilon)$,

\begin{align}
\label{eq:discrete_error_CE}
    \mathfrak{R}_{\bnu}^{\rm C}(T)\;&\leq\; \Delta(\varepsilon) + \frac{1}{T}\Bigg[ \underbrace{2 T^{1-q/2} \mcL K\left({32}\ \frac{\sqrt{2 {\rm e} \log T } + 32}{\sqrt{\rho {\varepsilon}} }  \right)^{q}}_{\triangleq\; C_5(T)} \nonumber\\
    &\qquad\qquad\qquad + 2\mcL K W^q  \left(T^2\left(\left(\frac{1}{T^2}+1\right)^K-1\right) + K^q T^{1-q} \right) \nonumber\\
    & \qquad\qquad\qquad\qquad + BT\bigg(\bigg( \frac{1}{T^2}+1\bigg)^K -1\bigg)\Bigg]\ .
\end{align}


Furthermore, similar to the RS-UCB-M analysis, it can be readily verified that $C_5(T)$ is the dominating term for any $T>\e^K$. Hence, we can simplify the upper bound as follows.
\begin{align}
\label{eq:theorem_lastline_CE}
    \mathfrak{R}_{\bnu}^{\rm C}(T)\; &\leq\; \Delta(\varepsilon) + 5 \mcL KW\Big(32 \varepsilon^{-\frac{1}{2}} \rho^{-\frac{1}{2}} T^{-\frac{1}{2}}\Big( \sqrt{2\e\log T} + 32\Big)\Big)^q 
\end{align}
where \eqref{eq:theorem_lastline_CE} follows from \(W > 1\) for sub-Gaussian distributions as shown in Appendix \ref{Appendix:W_finitess} and \(q \in (0, 1]\).

\subsection{Proof of Theorem~\ref{corollary:RS-UCB-M} for CE-UCB-M}

Let us set 
\begin{align}
\varepsilon = \Theta \left(\left(K^{\frac{2}{q}}\log T / T\right)^{\kappa}\right)\ ,    
\end{align}
where $\kappa = \frac{1}{\frac{2\beta}{q}+2}$. Assuming \(\Delta_{\min}(\varepsilon) = \Omega(\varepsilon^\beta)\ \), it can be readily verified that this choice of the discretization level $\varepsilon$ satisfies the condition in~\eqref{eq:T_epsilon}. Hence, for discrete regret, from~\eqref{eq:discrete_error_CE} we have
\begin{align}
\label{eq:Discrete_order_CE}
    \Bar{\mathfrak{R}}_{\bnu}^{\rm C}(T) &= O \left ((K^{\frac{2}{q}}\log T/T)^{\left(1-\kappa\right)\frac{q}{2}} \right) 
    \ .
\end{align}
Furthermore, from Lemma~\eqref{lemma:Delta_error}, we have
\begin{align}
    \Delta(\varepsilon) &\leq \mcL (KW)^r \left(\frac{\varepsilon}{2}\right)^r  \ .
\end{align}
Hence, \(\Delta(\varepsilon) = O(\varepsilon^r)\). Consequently, we have
\begin{align}
\label{eq:delta_order_ce}
    \Delta(\varepsilon) = O\left(K^{r \left( 1 + \frac{2\kappa}{q} \right)}\left(\log T / T\right)^{r \kappa} \right).
\end{align}
From~\eqref{eq:Discrete_order_CE} and \eqref{eq:delta_order_ce} the regret of the CE-UCB-M is bounded from above by
    \begin{align}
    \label{eq:UCB upper bound appendix CE}
        \mathfrak{R}_{\bnu}^{\rm C}(T) &= O \left (\max\Big\{ K^{r \left( 1 + \frac{2\kappa}{q} \right)}\left(\log T / T\right)^{r \kappa}  \;,\;(K^{\frac{2}{q}}\log T/T)^{\left(1-\kappa\right)\frac{q}{2}} \Big\} \right) \\
        & = O \left (\max\Big\{K^{r \left( 1 + \frac{2\kappa}{q} \right)}, K^{1-\kappa} \Big\} \max\Big\{ \left(\log T / T\right)^{r \kappa}  \;,\;  (\log T/T)^{(1-\kappa)\frac{q}{2}}\Big\} \right)\ \ .
    \end{align}
When $r\leq \beta + q/2$, this bound becomes
\begin{align}
    \mathfrak{R}_{\bnu}^{\rm C}(T) \leq O \left (\max\Big\{K^{r \left( 1 + \frac{2\kappa}{q} \right)}, K^{1-\kappa} \Big\}   (\log T/T)^{r\kappa} \right)\ .
\end{align}

\section{Additional Experiments}
\label{Appendix: Additional Experiments}


In this section, we provide figures mentioned in the main paper, additional details of the experiments, some additional comparisons, and some details about computing resources. 

\textbf{Computing Resources.} All experiments are conducted on Mac Mini 2023 equipped with 24 Gigabytes of RAM, and 2 CPU cores have been used for each experiment.

\textbf{General Setting.} We have experimented using Gini deviation. For a Bernoulli distributed \(K-\)arm bandit instance, we have run 1000 independent trials for each configuration. 

\textbf{Regret versus horizon.} The configuration of the experiment considered in Figure ~\ref{fig:regret_uniform} is the following. We chose the number of arms as \(K=2\), the exploration coefficient as \(\rho=0.1\), the arm distributions as \(\text{Bern}(0.4)\) and \(\text{Bern}(0.9)\). 
Figure~\ref{fig:regret_uniform} shows that uniform sampling is not regret-efficient in the premise of mixtures.

\textbf{Regret versus gaps.} In Figure~\ref{fig:ucb_eta}, for the \(2-\)armed bandit instance, we examine the regret of RS-UCB-M for varying mean values of the arms. Specifically, we set the mean of the first arm as \(p_1 = 0.55\) and vary the mean values of the second arm \(p_2\).  The exploration coefficient is set to \(\rho=0.1\) and time horizon is set to \(T=10^5\). From~\ref{fig:ucb_eta},  we observe that with increasing difference between the means of the two arms, the regret increases. Additionally, we run a similar experiment for \(p_1=0.59\) which demonstrates similar results.

\textbf{Regret versus exploration coefficient.} In Figure~\ref{fig:ucb_zeta}, we show the behaviour of the regret of the RS-UCB-M algorithm for increasing values of the exploration coefficient \(\rho\). We fix the horizon at \(T=7.5 \cdot 10^4\) and the number of arms is set to \(K=3\). We observe that when the exploration coefficient \(\rho\) increases, there is an increase in the regret.

\begin{figure}[h]
    \centering
    \begin{subfigure}[b]{0.3\textwidth} 
        \centering
        \includegraphics[width=\textwidth]{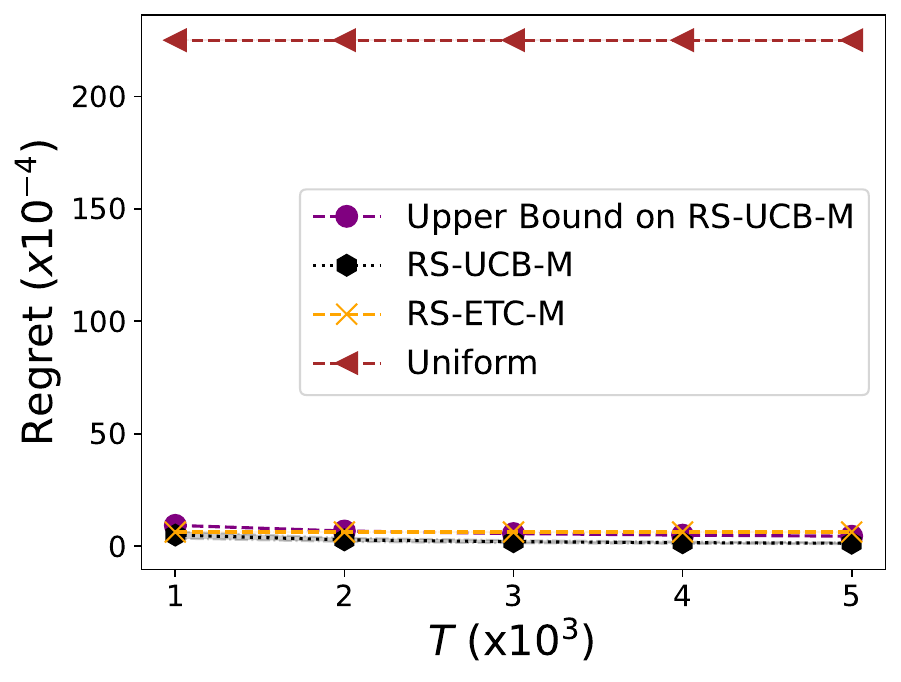}
        \caption{Regret for algorithms RS-UCB-M, RS-ETC-M and Uniform}
        \label{fig:regret_uniform}
    \end{subfigure}
    \hspace{0.03\textwidth} 
    \begin{subfigure}[b]{0.3\textwidth} 
        \centering
        \includegraphics[width=\textwidth]{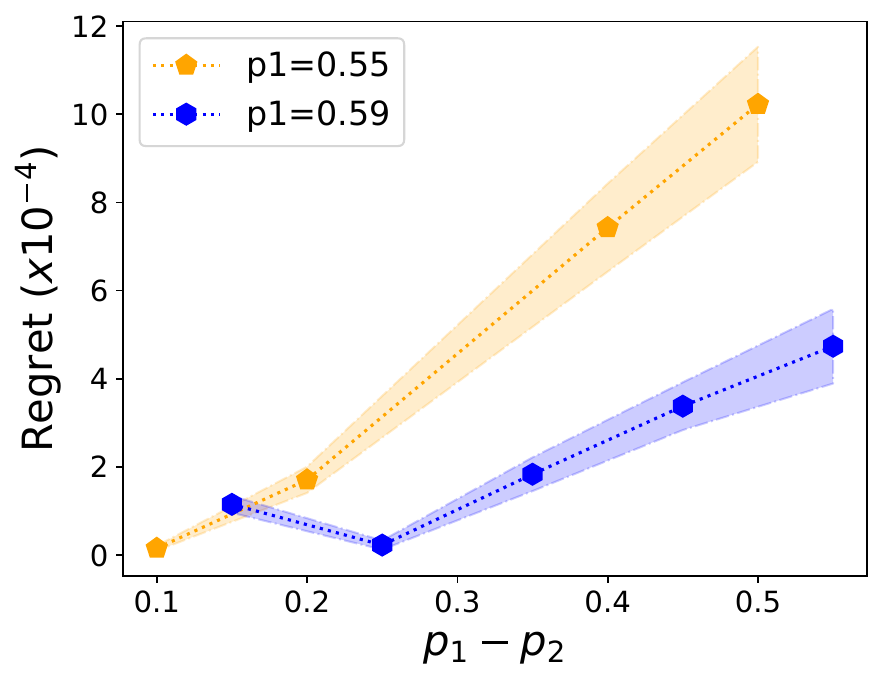}
        \caption{Regret of RS-UCB-M algorithm for different values of Bernoulli pmf differences }
        \label{fig:ucb_eta}
    \end{subfigure}
    \hspace{0.03\textwidth} 
    \begin{subfigure}[b]{0.3\textwidth} 
        \centering
        \includegraphics[width=\textwidth]{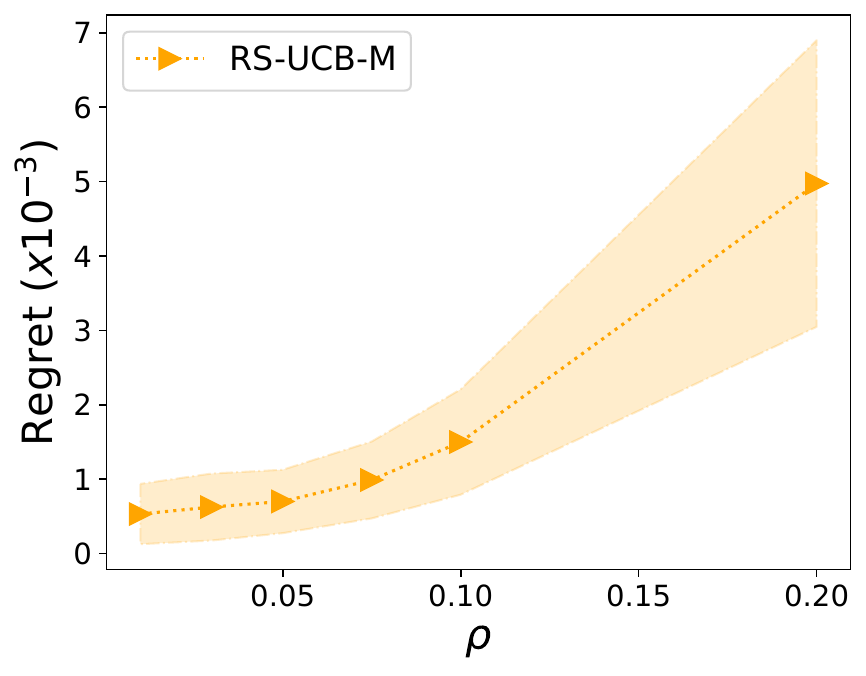}
        \caption{Regret of RS-UCB-M algorithm for different values of exploration coefficient \(\rho\)}
        \label{fig:ucb_zeta}
    \end{subfigure}

    \caption{Regrets of algorithms for different settings.}
    \label{fig:2x2images}
\end{figure}

\end{document}